\documentclass{article}

\usepackage{microtype}
\usepackage{graphicx}
\usepackage{subfigure}
\usepackage{booktabs} 

\usepackage{hyperref}
\usepackage{booktabs}

\usepackage[accepted]{icml2024}

\usepackage{amsmath}
\usepackage{amssymb}
\usepackage{mathtools}
\usepackage{amsthm}

\usepackage[capitalize,noabbrev]{cleveref}
\usepackage{wrapfig}

\theoremstyle{plain}
\newtheorem{theorem}{Theorem}[section]
\newtheorem{proposition}[theorem]{Proposition}
\newtheorem{lemma}[theorem]{Lemma}

\theoremstyle{definition}
\newtheorem{definition}[theorem]{Definition}

\theoremstyle{remark}
\newtheorem{remark}[theorem]{Remark}

\usepackage[textsize=tiny]{todonotes}

\usepackage{color,soul}

\icmltitlerunning{GFlowNet Training by Policy Gradients}

\begin{document}

\twocolumn[
\icmltitle{GFlowNet Training by Policy Gradients}




\begin{icmlauthorlist}
\icmlauthor{Puhua Niu}{yyy}
\icmlauthor{Shili Wu}{yyy}
\icmlauthor{Mingzhou Fan}{yyy}
\icmlauthor{Xiaoning Qian}{yyy,3,zzz}
\end{icmlauthorlist}

\icmlaffiliation{yyy}{Department of Electrical and Computer Engineering, Texas A\&M University, College Station, TX, USA}
\icmlaffiliation{3}{Department of Computer Science and Engineering, Texas A\&M University, College Station, TX, USA}
\icmlaffiliation{zzz}{Computational Science Intiative, Brookhaven National Laboratory, Upton, NY, USA}
\icmlcorrespondingauthor{Xiaoning Qian}{xqian@tamu.edu}

\icmlkeywords{Machine Learning, ICML}

\vskip 0.3in
]



\printAffiliationsAndNotice{}  

\begin{abstract}
Generative Flow Networks (GFlowNets) have been shown effective to generate combinatorial objects with desired properties. We here propose a new GFlowNet training framework, with policy-dependent rewards, that bridges keeping flow balance of GFlowNets to optimizing the expected accumulated reward in traditional Reinforcement-Learning~(RL). This enables the derivation of new policy-based GFlowNet training methods, in contrast to existing ones resembling value-based RL. It is known that the design of backward policies in GFlowNet training affects efficiency. We further develop a coupled training strategy that jointly solves GFlowNet forward policy training and backward policy design. Performance analysis is provided with a theoretical guarantee of our policy-based GFlowNet training. Experiments on both simulated and real-world datasets verify that our policy-based strategies provide advanced RL perspectives for robust gradient estimation to improve GFlowNet performance. Our code is available at: \href{https://github.com/niupuhua1234/GFN-PG}{github.com/niupuhua1234/GFN-PG}.
\end{abstract}

\section{Introduction}
Generative Flow Networks (GFlowNets) are a family of generative models on the space of combinatorial objects $\mathcal{X}$, e.g. graphs composed by organizing nodes and edges in a particular manner, or strings composed of characters in a particular ordering. GFlowNets aim to solve a challenging task, sampling $x\in \mathcal{X}$ with a probability proportional to some non-negative reward function $R(x)$ that defines an unnormalized distribution, where $|\mathcal{X}|$ can be enormous and the distribution modes are highly isolated by its combinatorial nature. GFlowNets~\citep{bengio2021flow,bengio2023gflownet} decompose the process of generating or sampling $x\in \mathcal{X}$ by generating incremental trajectories that start from a null state, pass through intermediate states, 
and end at $x$ as the desired terminating state. These trajectory instances are interpreted as the paths along a Directed Acyclic Graph~(DAG). Probability measures of trajectories are viewed as the amount of `water' flows along the DAG, with $R(x)$ being the total flow of trajectories that end at $x$, so that following the forward generating policy defined by the measure, sampled trajectories will end at $x$ with the probability proportional to $R(x)$. 

GFlowNets bear a similar form of reinforcement learning~(RL) in that they both operate over Markovian Decision Processes~(MDP) with a reward function $R(x)$, where nodes, edges, and node transition distributions defined by Markovian flows are considered as states, actions, and stochastic policies in MDPs. They, however, differ in the following aspects: the goal of RL problems is to learn optimal policies that maximize the expected cumulative trajectory reward by $R$. For \textbf{value-based} RL methods, the key to achieve this is by reducing the Temporal Difference~(TD) error of Bellman equations for the estimated state value function $V$ and state-action value function $Q$~\citep{sutton2018reinforcement,mnih2013playing}. GFlowNets amortize the sampling problem into finding some Markovian flow that assigns the proper probability flow to edges (actions) so that the total flow of trajectories ending at $x$ is $R(x)$. When studying these in the lens of RL, the existing GFlowNet training strategies are also value-based in that they achieve the goal by keeping the balance flow equation over states of the DAG, whose difference can be measured in trajectory-wise and edge-wise ways~\citep{bengio2021flow,bengio2023gflownet,malkin2022trajectory,madan2023learning}. 

Due to the similarity of GFlowNet training and RL, 
investigating the relationships between them can not only deepen understanding of GFlowNets but also help derive better training methods from RL. 
In this work, we propose policy-dependent rewards for GFlowNet training. This bridges GFlowNets to RL in that keeping the flow balance over DAGs can be reformulated as optimizing the expected accumulated rewards in RL problems. We then derive \textbf{policy-based} training strategies, which optimize the accumulated reward by its gradients with respect to~(w.r.t.) the forward policy directly~\citep{sutton1999policy,sutton2018reinforcement}. 

In terms of RL, we acknowledge that the existing GFlowNet training methods can be considered value-based and have the advantage of allowing off-policy training over policy-based methods~\citep{malkin2022gflownets}. Value-based methods, however, face the difficulty in designing a powerful sampler that can balance the exploration and exploitation trade-off, especially when the combinatorial space is enormous with well-isolated modes. Besides, employing typical annealing or random-mixing solutions may lead to the learned policy trapped in local optima. Finally, designing strategies for powerful samplers vary according to the structure and setting of modeling environments. Policy-based methods, especially the on-policy ones, transform the design of a powerful sampler into robust estimation of policy gradients, which can be achieved by variance reduction techniques~\citep{Schulmanetal_ICLR2016} and improvement of gradient descent directions, such as natural policy gradients~\citep{kakade2001natural} and mirror policy descent~\citep{zhan2023policy}. Conservation policy updates such as Trust-Region Policy Optimization (TRPO)~\citep{schulman2015trust,achiam2017constrained} and its first-order approximation, Proximal Policy Optimization (PPO)~\citep{schulman2017proximal}, have also been developed, for example as the backbone of ChatGPT~\citep{ouyang2022training}. Moreover, policy-based methods can be made off-policy, for example, by importance sampling~\citep{degris2012off}. Our work provides alternative ways to improve GFlowNet performance via policy-based training. Our contributions can be summarized as follows: 
\begin{itemize}\vspace{-1mm}
    \item We reformulate the GFlowNet training as RL over a special MDP where the reward is policy-dependent and the underlying Markov chain is \emph{absorbing}. We further derive policy gradients for this special MDP and propose policy-based training strategies for GFlowNets, inspired by policy gradient and TRPO methods over discounted MDPs with stationary rewards and \emph{ergodic} Markov chains.\vspace{-1mm}
    \item  We further formulate the design of backward policies in GFlowNets as an RL problem and propose a coupled training strategy. While finding a desired forward policy is the goal of GFlowNet training, well-designed backward policies, as components of the training objectives, are expected to improve training efficiency~\citep{shen2023towards}.\vspace{-1mm}
    \item  We provide performance analyzes for theoretical guaranties of our method for GFlowNet training. Our theoretical results are also accompanied by experiments in three application domains: hyper-grid modeling, biological and molecular sequence design, and Bayesian Network~(BN) structure learning. The obtained experimental results serve as empirical evidence for the validity of our work and also help empirically understand the relationship between GFlowNet training and RL.
\end{itemize}
\section{Preliminaries}
For notation compactness, we restrict DAGs of GFlowNets to be \emph{graded}\footnote{Any DAG can be equivalently converted to be graded by adding dummy non-terminating states. Please refer to Appendix A of~\citet{malkin2022gflownets} for more details.}.
In a DAG, $\mathcal{G}:=(\mathcal{S}, \mathcal{A})$, modeling a MDP of GFlowNets: $s\in \mathcal{S}$ denotes a state, $a\in \mathcal{A}$ denotes a directed edge/action $(s{\rightarrow}s')$, and $\mathcal{A} \subseteq \mathcal{S} \times \mathcal{S}$. 
Assuming that there is a topological ordering $\mathcal{S}_0,\ldots,\mathcal{S}_T$ for $T+1$ disjoint subsets of $\mathcal{S}$, then $\mathcal{S}=\bigcup_{t=0}^{T}\mathcal{S}_t$ 
and an element of $\mathcal{S}_t$ is denoted as $s_t$. We use $\{\prec,\succ,\preceq,\succeq\}$ to define the partial orders 
between states; for example, $\forall t<t': s_t \prec s_{t^\prime}$. 
Furthermore, being \emph{acyclic} means $\forall (s{\rightarrow}s')\in \mathcal{A}$: $s\prec s'$. Being \emph{graded} means $\mathcal{A}$ can be decomposed into $\bigcup_{t=0}^{T-1}\mathcal{A}_t$ where $ \mathcal{A}_t\bigcap \mathcal{A}_{t^\prime\neq t}=\emptyset$ and $a_t\in \mathcal{A}_t$ represents an edge $(s_t{\rightarrow}s_{t+1})$ connecting $\mathcal{S}_t$ and $\mathcal{S}_{t+1}$.  For any $s \in \mathcal{S}$, we denote its parent set by $Pa_{\mathcal{G}}(s) = \{ s' |( s'{\rightarrow}s) \in \mathcal{A} \}$ and its child set $Ch_{\mathcal{G}}(s) = \{ s' | (s{\rightarrow}s') \in \mathcal{A}\}$. Correspondingly, We denote the edge sets that start and end at $s$ as $\mathcal{A}(s)=\{(s{\rightarrow}s')|s'\in Ch_\mathcal{G}(s)\}$ and $\dot{\mathcal{A}}(s)=\{(s'{\rightarrow}s)|s'\in Pa_\mathcal{G}(s)\}$ respectively. The complete trajectory set {is defined as} $\mathcal{T}=\{\tau=(s_0\rightarrow \dots \rightarrow s_T)|\forall (s{\rightarrow}s') \in \tau: (s{\rightarrow}s')\in \mathcal{A} \}$. We use $\tau_{\succeq s}$ to denote the sub-trajectory that starts at $s$, and $\tau_{\geq t}$ the sub-trajectory that starts at $s_t$.
For the DAG $\mathcal{G}$ in GFlowNets, we have two special states: the initial state $s^0$ with  $Pa(s^0) = \emptyset$ and $S_0=\{s^0\}$, and the final state $s^f$ with $Ch(s^f) = \emptyset$ and $S_T=\{s^f\}$. 
Furthermore, the terminal state set, $S_{T-1}$  covers the object set $\mathcal{X}$ with a reward function $R: \mathcal{X} \rightarrow \mathbb{R}^{+}$.
\subsection{GFlowNets} 
GFlowNets aim at efficient sampling from $P^\ast(x):=\frac{R(x)}{Z^\ast}$,
where $Z^*=\sum_{x \in \mathcal{X}}R(x)$ 
 and directly computing $Z^*$ is often challenging with typically large $|\mathcal{X}|$. To achieve this, GFlowNets define a measure $F(\tau): \mathcal{T}\rightarrow \mathbb{R}^+$, termed as `flow'~\citep{bengio2023gflownet}, so that for any event $E$, $F(E)=\sum_{\tau\in E} F(\tau)$ and the total flow $Z=F(s^0)=F(s^f)$. For any event $E$ and $E'$, $P(E):=F(E)/Z$ and $P(E|E'):=\frac{F(E\cap E')}{F(E')}$. Furthermore, $F$ is restricted to be Markovian, which means $\forall \tau \in \mathcal{T}$:
\begin{equation}
\begin{split}
&P(\tau)=\prod_{t=1}^{T} P(s_{t-1}{\rightarrow}s_{t}|s_{t-1})=\prod_{t=1}^{T}\frac{F(s_{t-1}{\rightarrow}s_{t})}{F(s_{t-1})},     
\end{split}
\end{equation}where $F(s{\rightarrow}s')= \sum_{\tau \in \{\tau|(s{\rightarrow}s') \in \tau\}}F(\tau)$, $ F(s) =\sum_{\tau \in \{\tau|s \in \tau\}}F(\tau)$ and $P_F(s_{t}| s_{t-1}):=P(s_{t-1}{\rightarrow}s_{t}|s_{t-1})$. Similarly, $P_B(s_{t-1}|s_t):=P(s_{t-1}\rightarrow s_t|s_t)=\frac{F(s_{t-1}{\rightarrow}s_t)}{F(s_t)}$. A desired generative flow $F$ is set to have the terminal transition probability $P^{\mathsf{T}}(x):=P(x{\rightarrow}s_f)$ equal to $P^*(x)$. As shown 
in~\citet{bengio2023gflownet}, 
the necessary and sufficient condition is that $\forall s' \in \mathcal{S}\setminus\{s^0,s^f\}$:
\begin{equation}
\begin{split}
\sum_{s \in Pa(s')} F(s{\rightarrow}s') = \sum_{s'' \in Ch(s)} F(s'{\rightarrow}s''). 
\end{split}\label{flow-match}
\end{equation}
where $F(x{\rightarrow}s_f):=R(x)$ for any $x\in \mathcal{X}$.

\subsection{GFlowNet Training}
Directly estimating the transition flow $F(s{\rightarrow}s')$ via the flow matching objective~\citep{bengio2021flow} can suffer from the explosion of $F$ values, of which the numerical issues may lead to the failure of model training. In practice, the Trajectory Balance (TB) objective has been shown to achieve the state-of-the-art training performance~\citep{malkin2022trajectory}. With the TB objective, the desired flow is estimated by the total flow $Z$ and a pair of forward/backward policies, $P_F(s'|s)$ and $P_B(s|s')$. The TB objective $\mathcal{L}_{TB}(P_{\mathcal{D}})$ of a trajectory data sampler $P_\mathcal{D}$ is defined as:
\begin{align}
  \mathcal{L}_{TB}(P_{\mathcal{D}})&:=\mathbb{E}_{P_{\mathcal{D}}(\tau)}[L_{TB}(\tau)],\nonumber\\ 
  L_{TB}(\tau)&:=\left(\log\frac{\ P_{F}(\tau|s_{0})Z}{P_{B}(\tau|x)R(x)}\right)^2.
\end{align}
In the equation above, $P_{F}(\tau|s_0)=\prod_{t=1}^{T}{P_{F}(s_t|s_{t-1})}$ with $P_F(\tau)=P_{F}(\tau|s_0)$, $P_F^{\top}(x):=P_F(x\rightarrow s^f)$, and $P_F(\tau|x)=P_F(\tau|x\rightarrow s^f)=P_F(\tau)/P_F^{\top}(x)$. Correspondingly, $P_B(\tau|x)=P_B(\tau|x\rightarrow s^f)=\prod_{t=1}^{T-1}{P_{B}(s_{t-1}|s_{t})}$, $P_B^\top(x):=P_B(x\rightarrow s^f)$ equal to $P^\ast(x)$, $P_{B}(\tau)=P_B^\top(x)P_B(\tau|x)$, and $P_B(\tau|s_0)=P_B(\tau)$. Furthermore, we define $\mu(s_0=s^0):=Z/\widehat{Z}$ as the $1$-categorical distribution over $\mathcal{S}_0$ so that $P_{F,\mu}(\tau):=P_F(\tau|s_0)\mu(s_0)=P_F(\tau)$, where $\widehat{Z}$ is the normalizing constant whose value is clamped\footnote{For any parametrized function $f(\cdot;\theta)$ and $\hat{f}(\cdot)$ clamped to $f$, $\hat{f}$ is equal to $f$, but regarded as constant w.r.t. $\theta$ during gradient computation.} to $Z$. We 
define $P_{B,\rho}(\tau):=P_B(\tau|x)\rho(x)$ with an arbitrary distribution $\rho$ over $\mathcal{X}$.

\section{Policy Gradients for GFlowNet Training}
Following~\citet{malkin2022gflownets}, we first extend the relationship between the GFlowNet training methods based on the TB objective and KL divergence. With the extended equivalence, we then introduce our policy-based and coupled training strategies for GFlowNets. Finally, we present theoretical analyses on 
our proposed strategies. 
\subsection{Gradient Equivalence}
When choosing trajectory sampler $P_{\mathcal{D}}(\tau)$ equal to $P_{F}(\tau)$, the gradient equivalence between using the KL divergence and TB objective has been proven~\citep{malkin2022gflownets}. 
However, this forward gradient equivalence does not take the total flow estimator $Z$ into account. Moreover, the backward gradient equivalence requires computing the expectation over $P^\ast(x)$, which is not feasible. In this work, we extend the proof of the gradient equivalence to take all gradients into account and remove the dependency on $P^\ast(x)$, while keeping feasible computation. 

\begin{proposition}
\label{TB-equivalence} 
Given a forward policy $P_F(\cdot|\cdot;\theta)$, a backward policy $P_B(\cdot|\cdot;\phi)$, and a total flow  estimator $Z(\theta)$, the gradient of the TB objective\footnote{Training via TB was intrinsically done in an off-policy setting, so $\nabla \mathcal{L}_{TB}(P_\mathcal{D})=\mathbb{E}_{P_\mathcal{D}(\tau)}[\nabla L_{TB}(\tau)]$ for any choice of  $P_{\mathcal{D}}$.} can be written as: 
\begin{align}
        \frac{\nabla_\theta \mathcal{L}_{TB}(P_{F,\mu};\theta)}{2}&=\nabla_\theta D_{KL}^{\mu(\cdot;{\theta})}(P_{F}(\tau|s_0;\theta),P_{B}(\tau|s_0))\nonumber\\&\quad+\frac{1}{2}\nabla_{\theta}\left(\log Z(\theta)-\log Z^*\right)^2 \nonumber \\
        &=\nabla_\theta D_{KL}^{\mu(\cdot;{\theta})}(P_{F}(\tau|s_0;\theta),\widetilde{P}_{B}(\tau|s_0;\theta)); \nonumber  \\
         \frac{\nabla_{\phi}\mathcal{L}_{TB}(P_{B,\rho};\phi)}{2} &=\nabla_{\phi} D_{KL}^{\rho}(P_{B}(\tau|x;\phi),P_{F}(\tau|x)) \nonumber  \\&
         =\nabla_{\phi} D_{KL}^{\rho}(P_{B}(\tau|x;\phi),\widetilde{P}_{F}(\tau|x)).
\end{align}
\end{proposition} 
In the equations above, $\widetilde{P}_F(\tau|x):=P_F(\tau_{\preceq x})=\prod_{t=1}^{T-1}P_F(s_t|s_{t-1})$ and $\widetilde{P}_{B}(\tau|s_0):=P_B(\tau|x)R(x)/Z$, denoting two unnormalized distributions of $P_F(\tau|x)$ and $P_B(\tau|s_0)$. For arbitrary distributions $p$, $q$, and $u$, $D_{KL}^u(p(\cdot|s), q(\cdot|s)) := \mathbb{E}_{u(s)}[D_{KL}(p(\cdot|s), q(\cdot|s))]$. 

The proof is provided in Appendix \ref{proof-TB-equivalence}. As the TB objective is a special case of the Sub-Trajectory Balance (Sub-TB) objective~\citep{madan2023learning}, we also provide 
the proof of the gradient equivalence with respect to the Sub-TB objective in Appendix \ref{Sub-TB-equivalence}, where the initial distribution $\mu$ becomes more flexible.

\subsection{RL Formulation of GFlowNet Training}
Inspired by the equivalence relationship in Proposition~\ref{TB-equivalence}, we propose new reward functions that allow us to formulate GFlowNet training as RL problems with corresponding policy-based training strategies.
\begin{definition}[Policy-dependent Rewards]   
For any action $a=(s{\rightarrow}s')\in \mathcal{A}(s)\,(a\in \dot{\mathcal{A}}(s'))$, we define two reward functions as: 
\begin{align}
R_{F}(s,a;\theta)&:=
    \log\frac{\pi_{F}(s,a;\theta)}{\pi_B(s',a;\theta)},\nonumber\\
    R_{B}(s',a;\phi)&:=
    \log\frac{\pi_{B}(s',a;\phi)}{\pi_F(s,a)},
\end{align}
where $\pi_F(s,a;\theta):=P_F(s'|s;\theta), \pi_B(s',a;\phi):=P_B(s|s';\phi) $,  $\pi_B(s^f,a)$ is equal to $R(x)/Z$ for $a=(x{\rightarrow} s_f)$. For any $a\notin \mathcal{A}(s)$, $ R_F(s,a):=0$. For any $a\notin \dot{\mathcal{A}}(s^\prime)$,  $R_B(s^\prime,a):=0$. 
\end{definition}
Tuples $(\mathcal{S},\mathcal{A},\mathcal{G}, R_F)$ and $(\mathcal{S},\dot{\mathcal{A}},\mathcal{G}, R_B)$ specify two MDPs with policy-dependent rewards. In the MDPs, $\mathcal{G}$ specifies a deterministic transition environment such that $P(s'|s, a)=\mathbb{I}[(s{\rightarrow s^\prime})=a]$ with the indicator function $\mathbb{I}$. $(\mathcal{G}, \pi_F)$ and $(\mathcal{G}, \pi_B)$ correspond to two \emph{absorbing} Markovian chains. Accordingly, the nature of DAGs requires that each state has only one order index, allowing us to define time-invariant expected value functions of states and state-action pairs, which are defined as $V_{F}(s):=\mathbb{E}_{P_F(\tau_{> t}|s_t)}[\sum_{l=t}^{T-1}{R_{F}(s_{l},a_{l})}|s_t=s]$ and $Q_F(s,a):=\mathbb{E}_{P_{F}(\tau_{>t+1}|s_t,a_t)}[\sum_{l=t}^{T-1}{ R_{F}(s_{l},a_{l})}|s_t=s,a_t=a]$. Then we define 
$J_{F}:=\mathbb{E}_{\mu(s_0)}[V_F(s_0)],A_F(s,a):=Q_F(s,a)-V_F(s)$, 
and $ d_{F,\mu}(s):=\frac{1}{T}\sum_{t=0}^{T-1} P_F(s_t=s)$. We likewise denote the functions for the backward policy as $\{V_{B},Q_{B}, J_B,A_B,d_{B,\rho}\}$. More details are provided in Appendix \ref{RL_def}.
By definition, $V_F(s_0)=\mathbb{E}_{P(\tau|s_0)}[\sum_{t=0}^{T-1}R_F(s_t,a_t)]=D_{KL}(P_{F}(\tau|s_0),\widetilde{P}_{B}(\tau|s_0))$
, so $J_F=D_{KL}^{\mu}(P_{F}(\tau|s_0),\widetilde{P}_{B}(\tau|s_0))$. Likewise, we can obtain $ J_B=D_{KL}^{\rho}(P_{B}(\tau|x),\widetilde{P}_{F}(\tau|x))$. Thus, we can conclude that GFlowNet training can be converted into minimizing the expected value function $J_F$ and $J_B$ by Proposition \ref{TB-equivalence}. With the derived $\nabla J_F$ and $\nabla J_B$ provided in Appendix \ref{policy_gradient}, we update $\pi_F$, $\pi_B$, and $\mu$ to minimize $J_F$ and $J_B$ based on the correspondingly computed gradients of the following two objectives:
\begin{align}\mathbb{E}_{\mu(s_0;\theta)}[V_F(s_0)]+T\,&\mathbb{E}_{d_{F,\mu}(s),\pi_F(s,a;\theta)}\left[A_F(s,a) \right],\nonumber\\ (T-1)\,&\mathbb{E}_{d_{B,\rho}(s),\pi_B(s,a;\phi)}\left[A_B(s,a) \right].\label{trpo_obj}
\end{align}

Our policy-based method generalizes the TB-based training with $P_\mathcal{D}$ equal to $P_F$ as follows: TB-based training corresponds to approximating $A(s,a)$ for $(s_t=s,a_t=a)$ empirically by $\widehat{Q}_F(s_t,a_t)-C$, where $\widehat{Q}_F(s_t,a_t)=\sum_{l=t}^{T-1}R_F(s_l,a_l)$, and $C$ is a constant baseline for variance reduction. For comparison, our policy-based method can be considered approximating $A_F(s,a)$ functionally by $\widehat{A}_F^\lambda(s,a)=\sum_{l=t}^{T-1}\lambda^{l-t}(\widehat{Q}_F(s_l,a_l)-\widetilde{V}_F(s_l))$, where $\lambda \in [0,1]$ controls the \textbf{bias-variance trade-off} for gradient estimation~\citep{Schulmanetal_ICLR2016}, $\widehat{Q}_F(s_l,a_l)=R_F(s_l,a_l)+\widetilde{V}_F(s_{l+1})$, and $\widetilde{V}_F(s_l)$ is a functional approximation of exact $V_F(s_l)$, serving as a functional baseline. Specifically, our policy-based method with $\lambda = 1$ can provide unbiased gradient estimation of $\nabla J_F$ as the TB-based method. 
This supports the stability of our policy-based method with the theoretical convergence guarantee by Theorem~\ref{pg-stable} in Section~\ref{perana}. A formal discussion of their connection is provided in Appendix~\ref{relation-TB-RL} and~\ref{updating-rules}. Additionally, we discuss the relationship between our method and traditional Maximum Entropy (MaxEnt) RL in Appendix~\ref{relation-RL-softQ}.

To further exemplify that the proposed rewards bridge policy-based RL techniques to GFlowNet training, we specifically focus on the TRPO method, whose performance is usually more stable than vanilla policy-based methods due to conservative model updating rules~\citep{schulman2015trust,achiam2017constrained}. Likewise, we propose a TRPO-based objective for updating $\pi_F$:
\begin{align}
&\min_{\theta^\prime}
T\,\mathbb{E}_{d_{F,\mu}(s;\theta),\pi_F(s,a;\theta^\prime)}\left[A_F(s,a;\theta) \right]\nonumber\\
&\textrm{s.t. }D_{KL}^{d_{F,\mu}(\cdot;\theta)}\left(\pi_F(s,a;\theta),\pi_F(s,a;\theta^\prime)\right)\leq \zeta_F.
\end{align}
The objective for $\pi_B$ can be defined similarly and is omitted here. This objective is motivated as the approximation of the upper bound in Theorem~\ref{trpo}, which generalizes the original results for static rewards and \emph{ergodic} Markov chains. We defer the discussion of their relationship in Section \ref{perana}.
Although we focus on the above constrained formulation, an unconstrained surrogate objective can be constructed via importance sampling. Specifically, we can write
\begin{align}
\min_{\theta^\prime} 
&\,\mathbb{E}_{d_{F,\mu}(s;\theta),\pi_F(s,a;\theta)}
\Big[\frac{\pi_F(s,a;\theta^\prime)}
{\pi_F(s,a;\theta)} A_F(s,a;\theta)
\Big]\nonumber\\&+\zeta_F D_{KL}^{d_{F,\mu}(\cdot;\theta)}(\pi(s,a;\theta),\pi(s,a;\theta^\prime))
.
\end{align} Furthermore, applying the clipping technique introduced in Section 3 of~\citep{schulman2017proximal} yields a PPO-style objective for GFlowNet training.

Moreover, the model parameter updating rule based on $\nabla J_F$ can be written as $\theta^\prime\gets\theta-\alpha\nabla_{\theta} J_F(\theta)$ or equivalently $\theta^\prime={\mathrm{argmin}}_{\theta^\prime}(\nabla J_F)^T\theta^\prime+\frac{1}{2\alpha}\left\|\theta^\prime-\theta\right\|^2_2$. Here, $\left\|\theta-\theta^\prime\right\|_2$ can be generalized to KL divergence or Bregman divergence corresponding to natural or mirror policy gradients, which we leave for future work.

Details of the model parameter updating rules for proposed methods are provided in Appendix \ref{updating-rules}.

\subsection{RL Formulation of Guided Backward Policy Design}
\begin{figure}[t]
      \centering \includegraphics[width=0.48\textwidth]{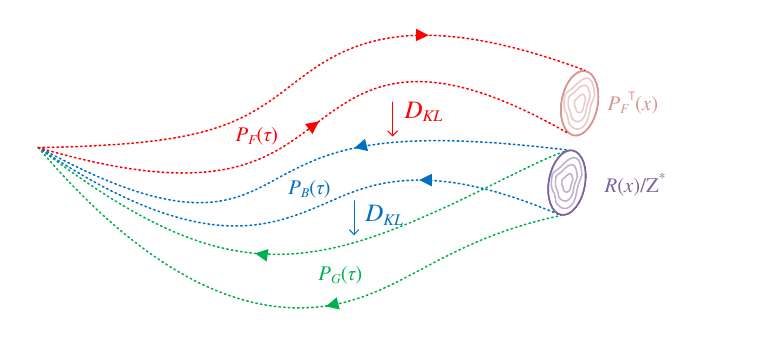}\vspace{-3mm}
      \caption{Dotted lines illustrate the spanning range of trajectories. $P_B$ and $P_G$ share the ground-truth terminating distribution $R(x)/Z^*$. When pushing $P_F$ to match $P_B$ trajectory-wise, $P_F^\top(x)$ will also be pushed to match $R(x)/Z^*$.}
  \label{fig:diagram}\vspace{-3mm}
    \end{figure}
During GFlowNet training, $(P_B, R)$ specifies the amount of desired flow that $(P_F, Z)$ is optimized to match. While $P_B(\cdot|\cdot)$ can be chosen freely in principle~\citep{bengio2023gflownet}, a well-designed $P_B$ that assigns high probabilities over sub-trajectories preceding the terminating state $x$ with a high reward value $R(x)$, will improve training efficiency. Following~\citet{shen2023towards}, we formulate the design problem as minimizing the following objective:
\begin{align}    \mathcal{L}_{TB}^G(P_B^\rho)&:=\mathbb{E}_{P_B^\rho(\tau)}[L_{TB}^G(\tau)],\nonumber\\ L_{TB}^G(\tau;\phi)&:=\left(\log \frac{P_B(\tau|x;\phi)}{P_G(\tau|x)}\right)^2, 
\end{align}
where $P_G(\tau|x)=\prod_{t=1}^{T-1}P_G(s_{t-1}|\tau_{\geq t})$ is called the conditional guided trajectory distribution, which is usually non-Markovian\footnote{By non-Markovian assumption, $P_G(\tau|x)$ can factorize in arbitrary ways conditioning on $x$. Here it is assumed to factorize in the backward direction for notation compactness.}, and $P_G(\tau)=P_G(\tau|x)P^\ast(x)$. As required by the training w.r.t. $P_F$, the objective $\mathcal{L}_{TB}^G$ aims at finding the backward policy whose Markovian flow best matches the non-Markovian flow induced by $P_G$.

\begin{proposition}\label{guided-equivalence}
Given a conditional guided trajectory distribution $P_G(\tau|x)$ and a backward policy $P_B(\cdot|\cdot;\phi)$, the gradients of $\mathcal{L}_{TB}^G$ can be written as:
\begin{equation}
\begin{split}
\frac{\nabla_{\phi}\mathcal{L}_{TB}^G(P_B^\rho;\phi)}{2}&=\nabla_{\phi} D_{KL}^{\rho}(P_{B}(\tau|x;\phi),P_G(\tau|x)).\\
\end{split}
\end{equation}
\end{proposition}
The proof can be found in Appendix~\ref{proof-guided-equivalence}. 
Based on the proposition, we propose a new reward that allows us to formulate the backward policy design problem as an RL problem. 

\begin{definition}
Given $P_G(\tau|x)$, we define a reward function for any action $a:=(s {\rightarrow} s')\in \dot{\mathcal{A}}(s')$ as:
\begin{equation}
    R_B^G(s',a;\phi):=\log \frac{\pi_B(s',a;\phi)}{\pi_G(s',a)}, 
\end{equation}
where $\pi_G(s',a):=P_G(s|\tau_{\succeq s'})$. For any $a\notin \dot{\mathcal{A}}(s^\prime)$,  $R_B^G(s^\prime,a):=0$. \end{definition}
Accordingly, we denote the associated function set as $\{V_B^G,Q_B^G,J_B^G,A_B^G, d_{B,\rho}^G\}$, which are defined in a similar way as $R_B$ but replacing $P_F$ by $P_G$. By the definition of $J_B^G$ and Proposition \ref{guided-equivalence}, we can conclude that $\nabla_{\phi}J_B^G(\phi)=\frac{1}{2}\nabla_{\phi}\mathcal{L}_{TB}^G(P_B^\rho ;\phi)$ and the design of backward policy can be solved by minimizing $J_B^G$. 
The form of $P_G$ is detailed in Appendix~\ref{experment-detail} for the corresponding experimental tasks.

In principle, following the pipeline by~\citet{shen2023towards}, we need to solve the optimization of $\mathcal{L}_{TB}^G$ to find the desired  $P_B$ at first. Then, freezing $P_B$, we can optimize $\mathcal{L}_{TB}$ to find the desired $P_F$. This gives rise to training inconvenience in practice. To avoid doing two-phase training, the authors mixed $P_B$ and $P_G$ by $\alpha P_B+(1-\alpha)P_G$ within the training objective w.r.t. $P_F$. This operation, however, lacks theoretical guarantees as the mixed distribution is still non-Markovian. By comparison, the RL formulation allows us to optimize $J_F$ and $J_B^G$ jointly with a theoretical performance guarantee, which we defer to the next section.

The workflow of our coupled training strategy is summarized in Algorithm \ref{work-flow-al} and depicted by Fig.~\ref{fig:diagram}. 

\begin{algorithm}[t]
        \caption{GFlowNet Training Workflow}
        \begin{algorithmic}\label{work-flow-al}
          \REQUIRE $P_F(\cdot|\cdot;\theta)$, $Z(\theta)$, $P_B(\cdot|\cdot;\phi)$, $P_G(\cdot|\cdot)$
\FOR{$n=\{1,\ldots,N\}$}
\STATE $\mathcal{D} \gets \{\widehat{\tau}|\widehat{\tau}\sim P_F(\tau;\theta)\}$
\STATE Update $\theta$ w.r.t. $R_F$ and $\mathcal{D}$
\IF{ $\phi \neq \emptyset$}
\STATE $\dot{\mathcal{D}}\gets \{\widehat{\tau}|\forall x\in \mathcal{D}: \widehat{\tau}|x\sim P_B(\tau|x)\}$
\IF {$P_G(\widehat{\tau}|x)\neq P_B(\widehat{\tau}|x)$ } 
  \STATE  Update $\phi$ w.r.t. $R_B^G$ and $\dot{\mathcal{D}}$
\ELSE
\STATE Update $\phi$ w.r.t. $R_B$ and $\dot{\mathcal{D}}$ 
\ENDIF 
\ENDIF
\ENDFOR
    \end{algorithmic}
      \end{algorithm}
      
\subsection{Performance Analysis}\label{perana}
In the previous sections, we formulate two RL problems with respect to $R_F$ and $R_B^G$. Now, we show below that the two problems can be solved jointly.

\begin{theorem}\label{G-F-B} 
Denoting $J_F^G$ as the corresponding function of $R_F^G$ obtained by replacing $\pi_B$ within $R_F$ with $\pi_G$ and choosing $\rho(x)=P_F^\top(x)$, then $J_F^G$, $J_F$ and $J_B^G$ satisfy the following inequality:
\begin{equation}
J_F^G\leq J_F+J_B^G
+(T-1)R_B^{G,\max}\sqrt{\frac{(J_F+\log Z^*-\log Z)}{2}},
\end{equation}
where $R_B^{G,\max}= \max_{s,a}\big|R_B^G(s,a)\big|$.
\end{theorem}
The proof is given in Appendix \ref{proof-GFB}.
As shown in Proposition~\ref{TB-equivalence}, minimization of $J_F$ will incur the decrease of $D_{KL}^\mu(P_F(\tau|s_0),P_B(\tau|s_0))=J_F+\log Z^*-\log Z$. Thus, by minimizing $J_F$ and $J_B^G$ jointly, the upper bound of $J_F^G$ decreases. 

Moreover, the TRPO-based objective introduced in the previous section is motivated by the following upper bounds.
\begin{theorem}\label{trpo} For two forward policies $(\pi_F,\pi_F^\prime)$ with $D_{KL}^{d_{F,\mu}^\prime}(\pi_F^\prime(s,\cdot),\pi_F(s,\cdot)) <\zeta_F$, and two backward policies $(\pi_B,\pi_B^\prime)$ with $D_{KL}^{d_{B,\rho}^\prime}(\pi_B^\prime(s,\cdot),\pi_B(s,\cdot)) <\zeta_B$, we have:
\begin{equation}
\begin{split}
    &\frac{J_F^\prime-J_F}{T}\leq \mathbb{E}_{d_{F,\mu}(s)\pi_F^\prime(s,a)}[A_F(s,a)]+\zeta_F+\epsilon_F\sqrt{2\zeta_F},   \\
        &\frac{J_B^\prime-J_B}{T-1}\leq \mathbb{E}_{ d_{B,\rho}(s)\pi_B^\prime(s,a)}[A_B(s,a)]+\zeta_B+\epsilon_B\sqrt{2\zeta_B}, 
\end{split}
\end{equation}
where $\epsilon_F=\max_s \big|\mathbb{E}_{\pi_F^\prime(s,a)}[A_F(s,a) ]\big|$ and $
\epsilon_B=\max_s \big|\mathbb{E}_{\pi_B^\prime(s,a)}[A_B(s,a)]\big|$. Similar results also apply to $J_B^G$ and $A_B^G$ for the backward policy $\pi_B$. 
\end{theorem}
The proof is given in Appendix \ref{proof-trpo}. The TRPO-based objective can be derived following a similar logic in~\citet{schulman2015trust} and ~\citet{achiam2017constrained}. Let's denote $M(\pi)=\mathbb{E}_{d_{F,\mu}(s),\pi(s,a)}[A_F(s,a)]+\zeta_F+\epsilon_F\sqrt{2\zeta_F}$ and set $\pi_F^\prime=\mathrm{argmax}_{\pi}M(\pi)$. In the worst case, we choose $\pi_F^\prime=\pi_F$ and $M(\pi_F^\prime)=0$; then it can be expected that there is a conservative solution. That is, $\pi_F^\prime\neq\pi_F$ and $\zeta_F$ is negligibly small, so that $M(\pi_F^\prime)<0$, thereby resulting in $J_F^\prime-J_F<0$. This implies the monotonic performance gain. TRPO method is an approximation to this update and usually provides more stable performance gain than the vanilla policy-based method.

Lastly, we provide a theoretical guarantee that policy-based methods with policy-dependent rewards can asymptotically converge to stationary points, which draws inspiration from the results for static rewards by~\citet{agarwal2019reinforcement}.   
\begin{theorem}\label{pg-stable}
Suppose that: $J_F(\theta)$ is $\beta-$smooth; $\mathbb{E}_{P(\cdot|\theta)}[\widehat{\nabla}_{\theta}J_F(\theta)]=\nabla_{\theta}J_F(\theta)$; the estimation variance, $\mathbb{E}_{P(\cdot|\theta)}\left[\big\|\widehat{\nabla}_{\theta} J_F(\theta) -\nabla_{\theta}J_F(\theta)\big\|^2_2\right]\leq \sigma_F$;  
$|\log Z(\theta)-\log Z^*|\leq \sigma_Z$; we update $\theta$ for $N$ $(>\beta)$ iterations by $\theta_{n+1}\gets\theta_n-\alpha\widehat{\nabla}J_F(\theta_n)$ with $n\in \{0,\ldots,N-1\}$, $\alpha=\sqrt{1/(\beta N)}$ and initial parameter $\theta_0$. Then we have:
\begin{align}
\min_{n\in \{0,\ldots,N-1\}}&\mathbb{E}_{P(\theta_n)}\left[\left\|\nabla_{\theta_n}J_F(\theta_n)\right\|^2_2\right]\nonumber\\ &\leq \frac{\sigma_F+\sigma_Z+\mathbb{E}_{P(\theta_0)}[J_F(\theta_0)]}{(\sqrt{(2N)/\beta}-1)}.
\end{align}
Similar results also apply to $J_B$ and $J_B^G$.
\end{theorem}
The proof is provided in Appendix~\ref{proof-pg-stable}. The assumption $\mathbb{E}_{P(\cdot|\theta)}[\widehat{\nabla}_{\theta}J_F(\theta)]=\nabla_{\theta}J_F(\theta)$ means gradient estimation is unbiased as explained in Appendix~\ref{updating-rules}.
\subsection{Related Work}
\paragraph{GFlowNet training}
GFlowNets were first proposed by~\citet{bengio2021flow} and trained by a Flow Matching (FM) objective, which aims at minimizing the mismatch of equation~\eqref{flow-match} w.r.t. a parameterized edge flow estimator $F(s\rightarrow s')$ directly. \citet{bengio2023gflownet} reformulated equation~\eqref{flow-match} and proposed a Detailed Balance (DB) objective, where edge flows $F(s\rightarrow s')$ are represented by $F(s)P_F(s'|s)$ or $F(s')P_B(s|s')$. \citet{malkin2022trajectory} claimed that the FM and DB objectives are prone to inefficient credit propagation across long trajectories and showed that the TB objective is the more efficient alternative. \citet{madan2023learning} proposed a Sub-TB objective that unified the TB and DB objectives as special cases. They can be considered as Sub-TB objectives with sub-trajectories, which are complete or of length $1$ respectively. \citet{zimmermann2022variational} proposed KL-based training objectives and~\citet{malkin2022gflownets} first established the equivalence between the KL and TB objectives.~\citet{shen2023towards} analyzed how the TB objective helps to learn the desired flow under the sequence prepend/append MDP setting, and proposed a guided TB objective. Forward-looking GFlowNets~\citep{pan2023better} improved the formulation of the DB objective by a better local credit assignment scheme, which was further generalized by learning energy decomposition GFlowNets~\citep{jang2023learning}. Finally, back-and-forth local search~\citep{kim2023local}, Thompson Sampling (TS)~\citep{rector2023thompson}, and temperature conditioning~\citep{kim2023learning} were proposed for the explicit design of $P_{\mathcal{D}}$.

\paragraph{Hierarchical variational inference}
Hierarchical Variational Inference (HVI)~\citep{vahdat2020nvae,zimmermann2021nested} generalizes amortized VI~\citep{zhang2018advances} to better explore specific statistical dependency structures between observed variables and latent variables by introducing the hierarchy of latent variables. Training HVI models typically involves minimizing the selected divergence measures between the target distribution and the variational distribution parametrized by neural networks~\citep{kingma2014autoencoding,burda2015importance}.
GFlowNets can be considered as a special HVI model, where non-terminating states are latent variables, the hierarchy corresponds to a DAG, and the task of minimizing divergences is achieved by keeping flow balance~\citep{malkin2022gflownets}. Our work provides another view of divergence minimization by interpreting the divergence as the expected accumulated reward.

\paragraph{Policy-based RL}
 Policy-based RL optimizes the expected value function $J$ directly based on policy gradients~\citep{sutton1999policy}. The most relevant policy-based methods are the Actor-Critic method~\citep{sutton2018reinforcement} and Trust Region Policy Optimization (TRPO)~\citep{schulman2015trust} along with its extension -- Constrained Policy Optimization (CPO)~\citep{achiam2017constrained}. 
 Standard formulations of these methods assume that the reward functions are fixed. Moreover, they are formulated for infinite-horizon discounted MDPs, in which the policy-induced transition matrix $P_\pi$, with entries $[P_\pi]_{j,i}=P_\pi(s^j|s^i)$ is invertible, inducing an \emph{ergodic} Markov chain.  By contrast, our method accommodates policy-dependent rewards and operates on the finite-horizon MDP defined by a DAG $\mathcal{G}$. The corresponding transition matrix $P_F$, with entries $[P_F]_{j,i}=P_F(s^j| s^i)$, is nilpotent and hence non-invertible, inducing an \emph{absorbing} Markov chain. Consequently, conventional policy-gradient methods, and particularly the standard theoretical analyzes of TRPO and CPO, do not directly apply to our setting.

 We note that \citet{weber2015reinforced} proposed a VI method based on policy gradient, despite lacking experimental support. Here, the objective can be interpreted as the KL divergence between two forward trajectory distributions. Without the help of $A$, the policy gradient is estimated in a vanilla manner, corresponding to $\widehat{A}^{1}$ and $\widehat{A}^{0}$. Besides,~\citet{rengarajan2022reinforcement} proposed a TRPO method for imitation learning, where the objective is the expected KL divergence between two forward policies, and the underlying Markovian chain is assumed to be \emph{ergodic} as the original method.


\paragraph{MaxEnt RL} \citet{bengio2021flow} has shown that directly applying MaxEnt RL with a fixed reward $R(x,a)$, defined to equal $R(x)$ for the terminal transition $((x,x\rightarrow s_f)$ and zero otherwise, is problematic as it corresponds to modeling $p(x)\propto n(x)R(x)$, where $n(x)$ is the number of trajectories that can pass through $x$. As discussed in Appendix~\ref{relation-RL-softQ}, our policy-based methods, when fixing $\log\pi_B(s',a)$ and $\log Z$ and choosing $\lambda=0$, can be related to soft-Q-learning, a typical MaxEnt RL method.

\paragraph{Imitation learning}
Imitation learning in RL is 
to learn a policy that mimics the expert demonstrations 
with limited expert data, 
by minimizing the empirical gap between the learned policy and expert policy.~\citep{rajaraman2020toward,ho2016generative}. 
For GFlowNet training in this work, 
we reduce the gap between the forward policy and the expert forward policy at the trajectory level, as the expert trajectory distribution is equal to $P_B(\tau)$, implicitly encouraging the learned policy to match the desired expert policy.  

\paragraph{Bi-level optimization}
Our proposed training strategy can also be seen as a Stochastic Bi-level Optimization method for GFlowNet training~\citep{ji2021bilevel,hong2023two,ghadimi2018approximation}. The inner problem is the RL problem w.r.t. $R_B$ or $R_B^G$ for designing backward policies. The outer problem is the RL problem w.r.t. $R_F$ for forward policies. For gradient-based solutions to Bi-level optimization in general, 
the learning rate of inner problems is carefully selected to guarantee the overall convergence, 
which is not required in our methods designed for GFlowNet training.

Additional discussion about policy-based and valued-based methods in the context of RL is provided in Appendix~\ref{policy-value-diff}. 

\section{Experiments}
To compare our policy-based training strategies for GFlowNets with the existing value-based methods, we have conducted three simulated experiments for 
hyper-grid modeling, four real-world experiments for biological and molecular sequence design, one on Bayesian Network structure learning, and ablation study of $\lambda$. We compare the performance of GFlowNets by the following training strategies: (1) DB-U, (2) DB-B, (3) TB-U, (4) TB-B, (5) TB-Sub, (6)  TB-TS, (7) RL-U, (8) RL-B; (9) RL-T and (10) RL-G, where notion `-U' means that $\pi_B$ is a fixed uniform policy;  `-B' means that $\pi_B$ is a parameterized policy; `RL' represent our policy-based method; `-T' represent our TRPO-based method with a uniform $\pi_B$ and `-G' represent our joint training strategy with guided policy;`-Sub' represent the weighted Sub-TB objective with a parameterized  $\pi_B$ in~\citet{madan2023learning}; `-TS' represent the TS objective with a parameterized  $\pi_B$ in~\citet{rector2023thompson}. By default, $P_{\mathcal{D}}$ is $\gamma$-decayed-noisy for valued-based methods.
Total variation $D_{TV}$, Jensen–Shannon divergence $D_{JSD}$, and mode accuracy $Acc$ are used to measure the gap between $P_F^\top(x)$ and $P^\ast(x)$. Detailed descriptions of experimental settings, including metric definitions, guided policy design, hyper-parameters, etc., can be found in Appendix \ref{experment-detail}. 
Our implementation is built upon the \emph{torchgfn} package~\citep{lahlou2023torchgfn}.

\begin{figure*}[t]
  \centering
\begin{minipage}[t]{0.45\linewidth}
  \centering
\includegraphics[width=1.0\textwidth]{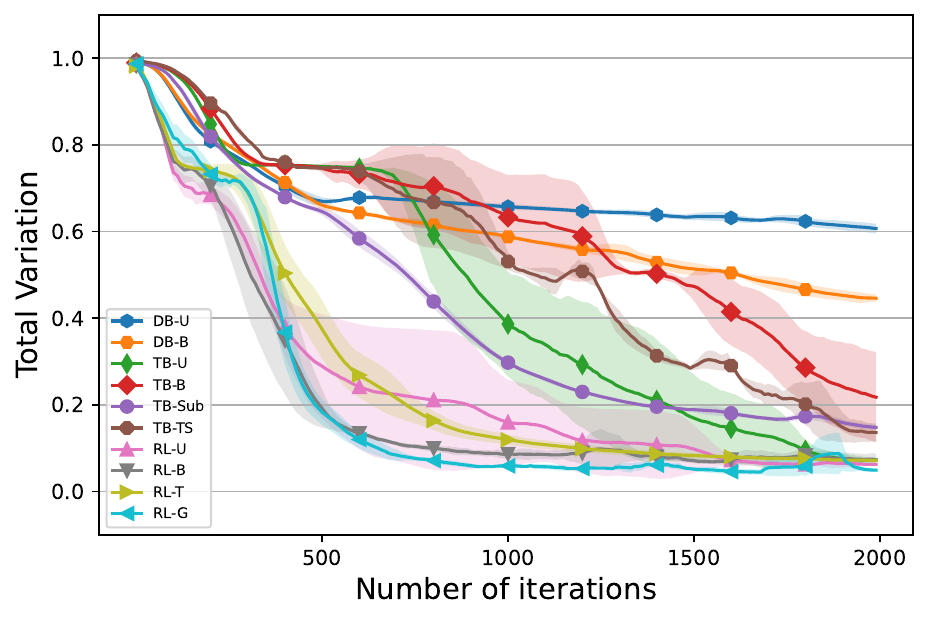}
\end{minipage}
\begin{minipage}[t]{0.45\linewidth}
  \centering
\includegraphics[width=1.0\textwidth]{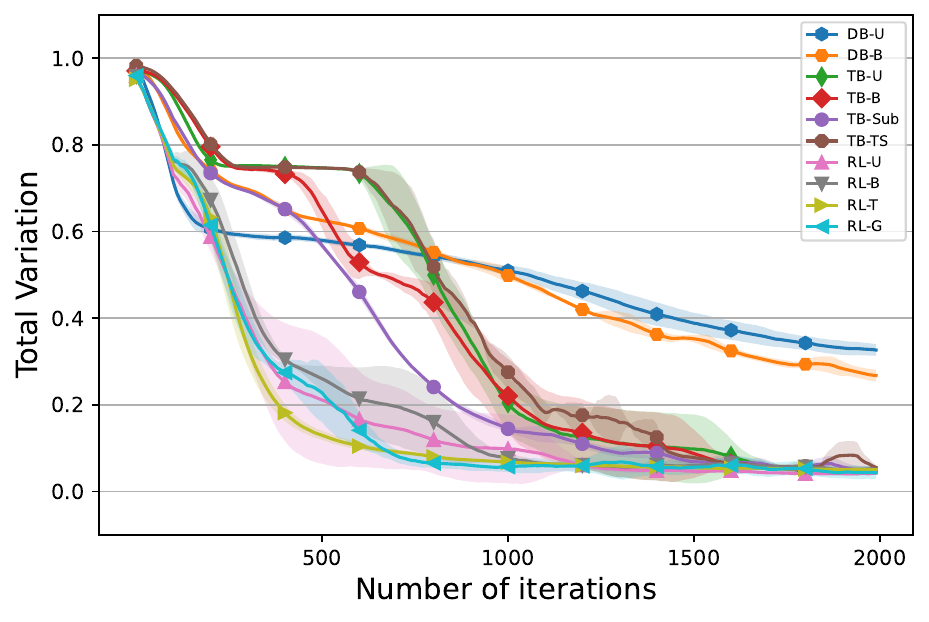}
 \end{minipage}
 \caption{Training curves by $D_{TV}$ between $P_F^\top$ and $P^\ast$ for $256\times256$~(left) and $128\times128$ hyper-grids~(right). The curves are plotted based on means and standard deviations of metric values across five runs and smoothed by a sliding window of length 10. Metric values are computed every 10 iterations. }\label{Hyper_training_256_128}
\end{figure*}

\begin{figure*}[t]
\centering
\begin{minipage}[t]{0.45\linewidth}
  \centering
\includegraphics[width=1.0\textwidth]{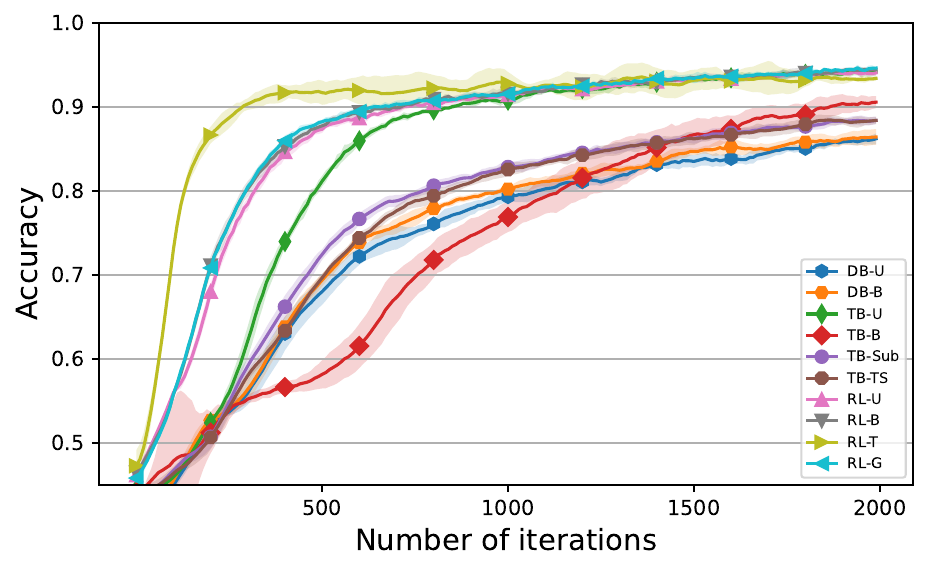}~\label{TF8-fig} 
 \end{minipage}
 \begin{minipage}[t]{0.45\linewidth}
  \centering
\includegraphics[width=1.0\textwidth]{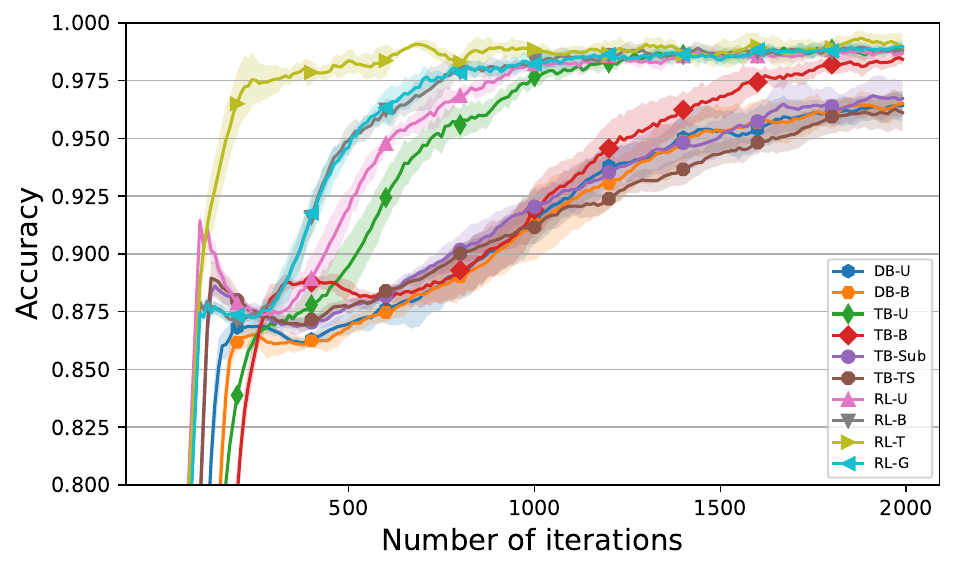}~\label{QM9-fig} 
\end{minipage}
\caption{Training curves by $Acc$ of $P_F^\top$ w.r.t. $P^\ast$ for SIX6~(left) and QM9~(right) datasets. The curves are plotted based on means and standard deviations of metric values across five runs and smoothed by a sliding window of length 10. Metric values are computed every 10 iterations.}\label{qm9-Six6-training}\vspace{-0mm}
\end{figure*}

\subsection{Hyper-grid Modeling}
In this set of experiments, we use the hyper-grid environment following~\citet{malkin2022gflownets}. In terms of GFlowNets, states are the coordinate tuples of an $D$-dimensional hyper-cubic grid with heights equal to $N$. The initial state $s^0$ is $(0,\ldots,0)$. Starting from $s^0$, actions correspond to increasing one of $D$ coordinates by $1$ for the current state or stopping the process at the current state and outputting it as the terminating state $x$. A manually designed reward function $R(\cdot)$ assigns high reward values to some grid points while assigning low values to others. We conduct experiments on $256\times 256 $, $128\times128$, $64\times64\times64$, and $32\times32\times32\times32$ grids. For performance evaluation,  $P_F^\top(x)$ is computed exactly by dynamic programming~\citep{malkin2022gflownets}. 

The training curves by $D_{TV}$ across five runs for $256\times 256$ and $128\times128$ grids are plotted in Fig.~\ref{Hyper_training_256_128}, and Table~\ref{hyper_grid_table_256_128} in Appendix~\ref{all-table} reports the mean and standard deviation of metric values at the last iteration. The graphical illustrations of $P_F^\top(x)$ are shown in Figs.~\ref{256-plots} and~\ref{128-plots} in Appendix~\ref{experiment_illustration}.
\textbf{In the first setting}, it can be observed that our policy-based methods, in terms of convergence rate or converged $D_{TV}$, perform much better than all the considered value-based training methods. This shows that our policy-based training strategies give a more robust gradient estimation. Besides,  RL-G achieves the smallest $D_{TV}$ and converges much faster than all the other competing methods. In RL-G, the guided distribution assigns small values to the probability of terminating at coordinates with low rewards. This prevents the forward policy from falling into the reward `desert' between the isolated modes. Finally, RL-T outperforms RL-U and behaves more stably than RL-U during training. This confirms that with the help of trust regions, the gradient estimator becomes less sensitive to estimation noises.  Here we use a fixed constant $\zeta_F$ for trust region control. It is expected that using a proper scheduler of $\zeta_F$ during training may further improve the performance of RL-T. \textbf{In the second setting}, the converged $D_{TV}$ of policy-based and TB-based methods are similar and significantly better than those of DB-based methods. As expected, policy-based methods converge much faster than all the value-based methods. Thus, the results further support the effectiveness of our policy-based methods. Moreover, RL-G and RL-T achieve the second-best and the best convergence, and RL-T shows better stability than RL-U. This again shows the superiority of coupled and TRPO-based strategies, confirming our theoretical analysis conclusions. 

More results and discussions for $64\times64\times64$ and $32\times32\times32\times32$ grids can be found in Appendix~\ref{Hyper-grid_add}.

\subsection{Biological and Molecular Sequence Design}
In this set of experiments, we use GFlowNets to generate nucleotide strings of length $D$  and molecular graphs composed of $N$ blocks according to given rewards. The initial state $s^0:=(-1,\ldots,-1)$ denotes an empty sequence. The generative process runs as follows: starting from $s^0$, an action is taken to pick one of the empty slots and fill it with one element until the sequence is completed. Then the sequence is returned as the terminating state $x$. We use nucleotide string datasets, SIX6 and PH04, and molecular graph datasets, QM9 and sEH, from~\citet{shen2023towards}. For metric $D_{TV}$ and $D_{JSD}$, $P_F^\top$ is computed exactly by dynamic programming. 

Following~\citet{shen2023towards}, the training curves by the mode accuracy $Acc$ and the number of modes for SIX6 and QM9 datasets are shown in Fig.~\ref{qm9-Six6-training}, and Fig.~\ref{qm9-Six6-training-mode} in Appendix~\ref{seq-add}. For evaluation consistency, we also provide the curves by $D_{TV}$ in Fig.~\ref{qm9-Six6-training-TV} in Appendix~\ref{seq-add}, as well as the metric values at the last iteration summarized in Tables~\ref{Bio-table} and~\ref{Bio-table1} in Appendix~\ref{all-table}. The graphical illustrations of $P_F^\top(x)$ are shown in Figs.~\ref{TF8-plots} and~\ref{QM9-plots} in Appendix~\ref{experiment_illustration}. In both experiments, TB-based and policy-based methods achieve better performance than DB-based methods. While the converged $Acc$ values of TB-based methods and our policy-based methods are similar, the latter converge much faster than TB-based methods with only TB-U achieving a comparable convergence rate. Besides, RL-T has the fastest convergence rates in both experiments. The performances of RL-G are similar to those of RL-B, which has a parameterized $\pi_B$, but slightly better than RL-U with a uniform $\pi_B$. In summary, experimental results for QM9 and SIX6 datasets align with those of hyper-grid tasks, confirming again the advantage offered by our policy-based methods for robust gradient estimation.

More results and discussions for PHO4 and sEH datasets can be found in Appendix~\ref{seq-add}.

\subsection{Ablation Study of $\lambda$}\label{Hyper-ab}

To investigate how the setting of $\lambda$, which controls the bias-variance-trade-off, may help robust estimation of gradients, we conduct experiments in the $256\times256$ grid environment. We compare the performance of RL-U methods with different $\lambda$ values 
and TB-U methods with different $\gamma$ values 
The obtained training curves by $D_{TV}$ across five runs are shown in Fig.~\ref{TB-RL-fig} in Appendix~\ref{Hyper-grid_add}. Among the choices of $\gamma$ for TB-U, the values 0.99 and 0.95 yield the best and the worst performances. In contrast, RL-U under all setups except $\lambda=1$, demonstrates significantly faster convergence than TB-U. It should be pointed out that when $\lambda=1$, $Q_F$ is approximated empirically as TB-based methods, but $V_F$ is approximated functionally. Additionally, the converged $D_{TV}$ in all setups of RL-U are better than those in all setups of TB-U. These results verify that 
by controlling $\lambda$, our policy-based methods can provide more robust gradient estimation than TB-based methods. 

We have also conducted performance comparisons between policy-based and value-based methods for Bayesian network structure learning. The results and discussions can be found in Appendix~\ref{DAG-add}.

\section{Conclusion, Limitations and Future Work}
This work bridges the flow-balance-based GFlowNet training to RL problems. We have developed policy-based training strategies, which provide alternative ways to improve training performance compared to the existing value-based strategies. The experimental results support our claims. 
Our policy-based methods are not limited to the cases where $\mathcal{G}$ must be a DAG as it intrinsically corresponds to minimizing the KL divergence between two distributions, which does not necessitate $\mathcal{G}$ to be a DAG. Future work will focus on extending the proposed methods to general $\mathcal{G}$ with the existence of cycles for more flexible modeling of generative processes of object $x\in \mathcal{X}$. While our policy-based training strategies do not require an explicit design of a data sampler and are shown to achieve better GFlowNet training performance, they may still get trapped into local optima due to the variance of gradient estimation when the state space is very large. Thus,   
future research will also focus on further improving policy-based methods by more robust gradient estimation techniques, under the gradient equivalence relationship.

\section*{Acknowledgements}
This work was supported in part by the U.S. National Science Foundation~(NSF) grants SHF-2215573, and by the U.S. Department of Engergy~(DOE) Office of Science, Advanced Scientific Computing Research (ASCR) under Awards B\&R\# KJ0403010/FWP\#CC132 and FWP\#CC138. Portions of this research were conducted with the advanced computing resources provided by Texas A\&M High Performance Research Computing.

\section*{Impact Statement}
The presented research aims to improve GFlowNet training methods to address the training performance challenge. The applications of our work encompass various societal realms, ranging
from medicine to materials design. 

\bibliography{icml2024}

\begin{thebibliography}{47}
\providecommand{\natexlab}[1]{#1}
\providecommand{\url}[1]{\texttt{#1}}
\expandafter\ifx\csname urlstyle\endcsname\relax
  \providecommand{\doi}[1]{doi: #1}\else
  \providecommand{\doi}{doi: \begingroup \urlstyle{rm}\Url}\fi

\bibitem[Achiam et~al.(2017)Achiam, Held, Tamar, and Abbeel]{achiam2017constrained}
Achiam, J., Held, D., Tamar, A., and Abbeel, P.
\newblock Constrained policy optimization.
\newblock In \emph{International conference on machine learning}, pp.\  22--31. PMLR, 2017.

\bibitem[Agarwal et~al.(2019)Agarwal, Jiang, Kakade, and Sun]{agarwal2019reinforcement}
Agarwal, A., Jiang, N., Kakade, S.~M., and Sun, W.
\newblock Reinforcement learning: Theory and algorithms.
\newblock \emph{CS Dept., UW Seattle, Seattle, WA, USA, Tech. Rep}, 32, 2019.

\bibitem[Beck(2017)]{beck2017first}
Beck, A.
\newblock \emph{First-order methods in optimization}.
\newblock SIAM, 2017.

\bibitem[Bengio et~al.(2021)Bengio, Jain, Korablyov, Precup, and Bengio]{bengio2021flow}
Bengio, E., Jain, M., Korablyov, M., Precup, D., and Bengio, Y.
\newblock Flow network based generative models for non-iterative diverse candidate generation.
\newblock \emph{Advances in Neural Information Processing Systems}, 34:\penalty0 27381--27394, 2021.

\bibitem[Bengio et~al.(2023)Bengio, Lahlou, Deleu, Hu, Tiwari, and Bengio]{bengio2023gflownet}
Bengio, Y., Lahlou, S., Deleu, T., Hu, E.~J., Tiwari, M., and Bengio, E.
\newblock Gflownet foundations.
\newblock \emph{Journal of Machine Learning Research}, 24\penalty0 (210):\penalty0 1--55, 2023.

\bibitem[Burda et~al.(2015)Burda, Grosse, and Salakhutdinov]{burda2015importance}
Burda, Y., Grosse, R., and Salakhutdinov, R.
\newblock Importance weighted autoencoders.
\newblock \emph{arXiv preprint arXiv:1509.00519}, 2015.

\bibitem[Degris et~al.(2012)Degris, White, and Sutton]{degris2012off}
Degris, T., White, M., and Sutton, R.~S.
\newblock Off-policy actor-critic.
\newblock In \emph{Proceedings of the 29th International Coference on International Conference on Machine Learning}, pp.\  179--186, 2012.

\bibitem[Deleu et~al.(2022)Deleu, G{\'o}is, Emezue, Rankawat, Lacoste-Julien, Bauer, and Bengio]{deleu2022bayesian}
Deleu, T., G{\'o}is, A., Emezue, C., Rankawat, M., Lacoste-Julien, S., Bauer, S., and Bengio, Y.
\newblock Bayesian structure learning with generative flow networks.
\newblock In \emph{Uncertainty in Artificial Intelligence}, pp.\  518--528. PMLR, 2022.

\bibitem[Ghadimi \& Wang(2018)Ghadimi and Wang]{ghadimi2018approximation}
Ghadimi, S. and Wang, M.
\newblock Approximation methods for bilevel programming.
\newblock \emph{arXiv preprint arXiv:1802.02246}, 2018.

\bibitem[Golpar~Raboky \& Eftekhari(2019)Golpar~Raboky and Eftekhari]{golpar2019nilpotent}
Golpar~Raboky, E. and Eftekhari, T.
\newblock On nilpotent interval matrices.
\newblock \emph{Journal of Mathematical Modeling}, 7\penalty0 (2):\penalty0 251--261, 2019.

\bibitem[Grinstead \& Snell(2006)Grinstead and Snell]{grinstead2006introduction}
Grinstead, C. and Snell, L.~J.
\newblock \emph{Introduction to probability}.
\newblock 2006.

\bibitem[Haarnoja et~al.(2018)Haarnoja, Zhou, Abbeel, and Levine]{haarnoja2018soft}
Haarnoja, T., Zhou, A., Abbeel, P., and Levine, S.
\newblock Soft actor-critic: Off-policy maximum entropy deep reinforcement learning with a stochastic actor.
\newblock In \emph{International conference on machine learning}, pp.\  1861--1870. PMLR, 2018.

\bibitem[Hestenes et~al.(1952)Hestenes, Stiefel, et~al.]{hestenes1952methods}
Hestenes, M.~R., Stiefel, E., et~al.
\newblock Methods of conjugate gradients for solving linear systems.
\newblock \emph{Journal of research of the National Bureau of Standards}, 49\penalty0 (6):\penalty0 409--436, 1952.

\bibitem[Ho \& Ermon(2016)Ho and Ermon]{ho2016generative}
Ho, J. and Ermon, S.
\newblock Generative adversarial imitation learning.
\newblock \emph{Advances in neural information processing systems}, 29, 2016.

\bibitem[Hong et~al.(2023)Hong, Wai, Wang, and Yang]{hong2023two}
Hong, M., Wai, H.-T., Wang, Z., and Yang, Z.
\newblock A two-timescale stochastic algorithm framework for bilevel optimization: Complexity analysis and application to actor-critic.
\newblock \emph{SIAM Journal on Optimization}, 33\penalty0 (1):\penalty0 147--180, 2023.

\bibitem[Jang et~al.(2023)Jang, Kim, and Ahn]{jang2023learning}
Jang, H., Kim, M., and Ahn, S.
\newblock Learning energy decompositions for partial inference of gflownets.
\newblock In \emph{The Twelfth International Conference on Learning Representations}, 2023.

\bibitem[Ji et~al.(2021)Ji, Yang, and Liang]{ji2021bilevel}
Ji, K., Yang, J., and Liang, Y.
\newblock Bilevel optimization: Convergence analysis and enhanced design.
\newblock In \emph{International conference on machine learning}, pp.\  4882--4892. PMLR, 2021.

\bibitem[Kakade(2001)]{kakade2001natural}
Kakade, S.~M.
\newblock A natural policy gradient.
\newblock \emph{Advances in neural information processing systems}, 14, 2001.

\bibitem[Kim et~al.(2023{\natexlab{a}})Kim, Ko, Zhang, Pan, Yun, Kim, Park, and Bengio]{kim2023learning}
Kim, M., Ko, J., Zhang, D., Pan, L., Yun, T., Kim, W.~C., Park, J., and Bengio, Y.
\newblock Learning to scale logits for temperature-conditional gflownets.
\newblock In \emph{NeurIPS 2023 AI for Science Workshop}, 2023{\natexlab{a}}.

\bibitem[Kim et~al.(2023{\natexlab{b}})Kim, Yun, Bengio, Zhang, Bengio, Ahn, and Park]{kim2023local}
Kim, M., Yun, T., Bengio, E., Zhang, D., Bengio, Y., Ahn, S., and Park, J.
\newblock Local search gflownets.
\newblock In \emph{The Twelfth International Conference on Learning Representations}, 2023{\natexlab{b}}.

\bibitem[Kingma \& Welling(2014)Kingma and Welling]{kingma2014autoencoding}
Kingma, D.~P. and Welling, M.
\newblock Auto-encoding variational bayes.
\newblock In Bengio, Y. and LeCun, Y. (eds.), \emph{ICLR}, 2014.
\newblock URL \url{http://dblp.uni-trier.de/db/conf/iclr/iclr2014.html#KingmaW13}.

\bibitem[Kuipers et~al.(2014)Kuipers, Moffa, and Heckerman]{kuipers2014addendum}
Kuipers, J., Moffa, G., and Heckerman, D.
\newblock Addendum on the scoring of gaussian directed acyclic graphical models.
\newblock 2014.

\bibitem[Lahlou et~al.(2023)Lahlou, Viviano, and Schmidt]{lahlou2023torchgfn}
Lahlou, S., Viviano, J.~D., and Schmidt, V.
\newblock torchgfn: A pytorch gflownet library.
\newblock \emph{arXiv preprint arXiv:2305.14594}, 2023.

\bibitem[Madan et~al.(2023)Madan, Rector-Brooks, Korablyov, Bengio, Jain, Nica, Bosc, Bengio, and Malkin]{madan2023learning}
Madan, K., Rector-Brooks, J., Korablyov, M., Bengio, E., Jain, M., Nica, A.~C., Bosc, T., Bengio, Y., and Malkin, N.
\newblock Learning {GF}low{N}ets from partial episodes for improved convergence and stability.
\newblock In \emph{International Conference on Machine Learning}, pp.\  23467--23483. PMLR, 2023.

\bibitem[Malkin et~al.(2022{\natexlab{a}})Malkin, Jain, Bengio, Sun, and Bengio]{malkin2022trajectory}
Malkin, N., Jain, M., Bengio, E., Sun, C., and Bengio, Y.
\newblock Trajectory balance: Improved credit assignment in gflownets.
\newblock \emph{Advances in Neural Information Processing Systems}, 35:\penalty0 5955--5967, 2022{\natexlab{a}}.

\bibitem[Malkin et~al.(2022{\natexlab{b}})Malkin, Lahlou, Deleu, Ji, Hu, Everett, Zhang, and Bengio]{malkin2022gflownets}
Malkin, N., Lahlou, S., Deleu, T., Ji, X., Hu, E.~J., Everett, K.~E., Zhang, D., and Bengio, Y.
\newblock Gflownets and variational inference.
\newblock In \emph{The Eleventh International Conference on Learning Representations}, 2022{\natexlab{b}}.

\bibitem[Mnih et~al.(2013)Mnih, Kavukcuoglu, Silver, Graves, Antonoglou, Wierstra, and Riedmiller]{mnih2013playing}
Mnih, V., Kavukcuoglu, K., Silver, D., Graves, A., Antonoglou, I., Wierstra, D., and Riedmiller, M.
\newblock Playing {A}tari with deep reinforcement learning.
\newblock \emph{arXiv preprint arXiv:1312.5602}, 2013.

\bibitem[Ouyang et~al.(2022)Ouyang, Wu, Jiang, Almeida, Wainwright, Mishkin, Zhang, Agarwal, Slama, Ray, et~al.]{ouyang2022training}
Ouyang, L., Wu, J., Jiang, X., Almeida, D., Wainwright, C., Mishkin, P., Zhang, C., Agarwal, S., Slama, K., Ray, A., et~al.
\newblock Training language models to follow instructions with human feedback.
\newblock \emph{Advances in Neural Information Processing Systems}, 35:\penalty0 27730--27744, 2022.

\bibitem[Pan et~al.(2023)Pan, Malkin, Zhang, and Bengio]{pan2023better}
Pan, L., Malkin, N., Zhang, D., and Bengio, Y.
\newblock Better training of gflownets with local credit and incomplete trajectories.
\newblock In \emph{International Conference on Machine Learning}, pp.\  26878--26890. PMLR, 2023.

\bibitem[Rajaraman et~al.(2020)Rajaraman, Yang, Jiao, and Ramchandran]{rajaraman2020toward}
Rajaraman, N., Yang, L., Jiao, J., and Ramchandran, K.
\newblock Toward the fundamental limits of imitation learning.
\newblock \emph{Advances in Neural Information Processing Systems}, 33:\penalty0 2914--2924, 2020.

\bibitem[Rector-Brooks et~al.(2023)Rector-Brooks, Madan, Jain, Korablyov, Liu, Chandar, Malkin, and Bengio]{rector2023thompson}
Rector-Brooks, J., Madan, K., Jain, M., Korablyov, M., Liu, C.-H., Chandar, S., Malkin, N., and Bengio, Y.
\newblock Thompson sampling for improved exploration in gflownets.
\newblock In \emph{ICML 2023 Workshop on Structured Probabilistic Inference $\{$$\backslash$\&$\}$ Generative Modeling}, 2023.

\bibitem[Rengarajan et~al.(2022)Rengarajan, Vaidya, Sarvesh, Kalathil, and Shakkottai]{rengarajan2022reinforcement}
Rengarajan, D., Vaidya, G., Sarvesh, A., Kalathil, D., and Shakkottai, S.
\newblock Reinforcement learning with sparse rewards using guidance from offline demonstration.
\newblock \emph{arXiv preprint arXiv:2202.04628}, 2022.

\bibitem[Schulman et~al.(2015)Schulman, Levine, Abbeel, Jordan, and Moritz]{schulman2015trust}
Schulman, J., Levine, S., Abbeel, P., Jordan, M., and Moritz, P.
\newblock Trust region policy optimization.
\newblock In \emph{International conference on machine learning}, pp.\  1889--1897. PMLR, 2015.

\bibitem[Schulman et~al.(2016)Schulman, Moritz, Levine, Jordan, and Abbeel]{Schulmanetal_ICLR2016}
Schulman, J., Moritz, P., Levine, S., Jordan, M., and Abbeel, P.
\newblock High-dimensional continuous control using generalized advantage estimation.
\newblock In \emph{Proceedings of the International Conference on Learning Representations (ICLR)}, 2016.

\bibitem[Schulman et~al.(2017{\natexlab{a}})Schulman, Chen, and Abbeel]{schulman2017equivalence}
Schulman, J., Chen, X., and Abbeel, P.
\newblock Equivalence between policy gradients and soft q-learning.
\newblock \emph{arXiv preprint arXiv:1704.06440}, 2017{\natexlab{a}}.

\bibitem[Schulman et~al.(2017{\natexlab{b}})Schulman, Wolski, Dhariwal, Radford, and Klimov]{schulman2017proximal}
Schulman, J., Wolski, F., Dhariwal, P., Radford, A., and Klimov, O.
\newblock Proximal policy optimization algorithms.
\newblock \emph{arXiv preprint arXiv:1707.06347}, 2017{\natexlab{b}}.

\bibitem[Shen et~al.(2023)Shen, Bengio, Hajiramezanali, Loukas, Cho, and Biancalani]{shen2023towards}
Shen, M.~W., Bengio, E., Hajiramezanali, E., Loukas, A., Cho, K., and Biancalani, T.
\newblock Towards understanding and improving gflownet training.
\newblock In \emph{International Conference on Machine Learning}, pp.\  30956--30975. PMLR, 2023.

\bibitem[Sutton \& Barto(2018)Sutton and Barto]{sutton2018reinforcement}
Sutton, R.~S. and Barto, A.~G.
\newblock \emph{Reinforcement learning: An introduction}.
\newblock MIT press, 2018.

\bibitem[Sutton et~al.(1999)Sutton, McAllester, Singh, and Mansour]{sutton1999policy}
Sutton, R.~S., McAllester, D., Singh, S., and Mansour, Y.
\newblock Policy gradient methods for reinforcement learning with function approximation.
\newblock \emph{Advances in neural information processing systems}, 12, 1999.

\bibitem[Tsitsiklis \& Van~Roy(1996)Tsitsiklis and Van~Roy]{tsitsiklis1996analysis}
Tsitsiklis, J. and Van~Roy, B.
\newblock Analysis of temporal-diffference learning with function approximation.
\newblock \emph{Advances in neural information processing systems}, 9, 1996.

\bibitem[Vahdat \& Kautz(2020)Vahdat and Kautz]{vahdat2020nvae}
Vahdat, A. and Kautz, J.
\newblock Nvae: A deep hierarchical variational autoencoder.
\newblock \emph{Advances in neural information processing systems}, 33:\penalty0 19667--19679, 2020.

\bibitem[Weber et~al.(2015)Weber, Heess, Eslami, Schulman, Wingate, and Silver]{weber2015reinforced}
Weber, T., Heess, N., Eslami, A., Schulman, J., Wingate, D., and Silver, D.
\newblock Reinforced variational inference.
\newblock In \emph{Advances in Neural Information Processing Systems (NIPS) Workshops}, 2015.

\bibitem[Williams(1992)]{williams1992simple}
Williams, R.~J.
\newblock Simple statistical gradient-following algorithms for connectionist reinforcement learning.
\newblock \emph{Machine learning}, 8\penalty0 (3-4):\penalty0 229--256, 1992.

\bibitem[Zhan et~al.(2023)Zhan, Cen, Huang, Chen, Lee, and Chi]{zhan2023policy}
Zhan, W., Cen, S., Huang, B., Chen, Y., Lee, J.~D., and Chi, Y.
\newblock Policy mirror descent for regularized reinforcement learning: A generalized framework with linear convergence.
\newblock \emph{SIAM Journal on Optimization}, 33\penalty0 (2):\penalty0 1061--1091, 2023.

\bibitem[Zhang et~al.(2018)Zhang, B{\"u}tepage, Kjellstr{\"o}m, and Mandt]{zhang2018advances}
Zhang, C., B{\"u}tepage, J., Kjellstr{\"o}m, H., and Mandt, S.
\newblock Advances in variational inference.
\newblock \emph{IEEE transactions on pattern analysis and machine intelligence}, 41\penalty0 (8):\penalty0 2008--2026, 2018.

\bibitem[Zimmermann et~al.(2021)Zimmermann, Wu, Esmaeili, and van~de Meent]{zimmermann2021nested}
Zimmermann, H., Wu, H., Esmaeili, B., and van~de Meent, J.-W.
\newblock Nested variational inference.
\newblock \emph{Advances in Neural Information Processing Systems}, 34:\penalty0 20423--20435, 2021.

\bibitem[Zimmermann et~al.(2022)Zimmermann, Lindsten, van~de Meent, and Naesseth]{zimmermann2022variational}
Zimmermann, H., Lindsten, F., van~de Meent, J.-W., and Naesseth, C.~A.
\newblock A variational perspective on generative flow networks.
\newblock \emph{Transactions on Machine Learning Research}, 2022.

\end{thebibliography}
\bibliographystyle{icml2024}
\newpage
\appendix
\onecolumn

\section{Gradient Equivalence}

\begin{lemma}(REINFORCE trick ~\citep{williams1992simple})\label{reforce-trick}
Given a random variable $u$ following a distribution $p(\cdot;\psi)$ parameterized by $\psi$ and a arbitrary function $f$, we have $\nabla_\psi \mathbb{E}_{p(u;\psi)}[f(u)]= \mathbb{E}_{p(u)}[f(u)\nabla_\psi \log p(u;\psi)]=\mathbb{E}_{p(u)}[f(u)\nabla_\psi \log \Tilde{p}(u;\psi)]$, where $p(u;\phi)=\Tilde{p}(u;\phi)/\hat{Z}_p$, and $\hat{Z}_p$ is the normalizing constant and clamped to $\sum_u \Tilde{p}(u)$. 
\end{lemma}
\subsection{Proof of Proposition \ref{TB-equivalence}}\label{proof-TB-equivalence}
\begin{proof}
First of all, we split the parameters of the total flow estimator and forward transition probability and denote them as $Z(\theta_Z)$ and $P_F(\cdot|\cdot;\theta_{F})$ respectively. We further define $c(\tau)=\left(\log\frac{P_F(\tau|s_0)}{R(x)P_B(\tau|x)}\right)$.  

For the gradients w.r.t. $\theta_F$:
\begin{align}
&\frac{1}{2}\nabla_{\theta_F}\mathbb{E}_{P_F(\tau|s_0)}[L_{TB}(\tau;\theta_F)]=\frac{1}{2}\mathbb{E}_{P_{F,\mu}(\tau)}\left[\nabla_{\theta_{F}}\left(c(\tau;\theta_F)+\log Z\right)^2\right]
\nonumber\\
&=\mathbb{E}_{P_{F,\mu}(\tau)}\left[\left(c(\tau)+\log Z\right)\nabla_{\theta_{F}}\log P_F(\tau|s_0;\theta_F)\right]\nonumber\\
&=\mathbb{E}_{P_{F,\mu}(\tau)}\left[\left(c(\tau)+\log Z\right)\nabla_{\theta_{F}}\log P_F(\tau|s_0;\theta_F)\right]+\underbrace{\mathbb{E}_{P_{F,\mu}(\tau)}\left[\nabla_{\theta_{F}}\left(c(\tau;\theta_F)+\log Z\right)\right]}_{(a)}\nonumber\\&=\mathbb{E}_{\mu(s_0)}\left[\nabla_{\theta_F}\mathbb{E}_{P_F(\tau|s_0;\theta_F)}\left[c(\tau;\theta_F)+\log Z\right]\right]\nonumber\\
&=\mathbb{E}_{\mu(s_0)}[\nabla_{\theta_{F}}D_{KL}(P_F(\tau|s_0;\theta_F),\widetilde{P}_B(\tau|s_0))]\nonumber\\
&=\nabla_{\theta_{F}}D_{KL}^{\mu}(P_F(\tau|s_0;\theta_F),\widetilde{P}_B(\tau|s_0)),\label{gradient_F}
\end{align}
where (a) is equal to zero as $\mathbb{E}_{P_F(\tau|s_0)}[\nabla_{\theta_F}c(\tau;\theta_F)]=\mathbb{E}_{P_F(\tau|s_0)}[1\cdot\nabla_{\theta_F}\log P_F(\tau|s_0;\theta_F)]=\nabla_{\theta_F} \mathbb{E}_{P_F(\tau|s_0;\theta_F)}[1]=0$ by Lemma \ref{reforce-trick}.
We also have:
\begin{align}
\nabla_{\theta_{F}}D_{KL}^{\mu}(P_F(\tau|s_0;\theta_F),\widetilde{P}_B(\tau|s_0))=& \nabla_{\theta_{F}}D_{KL}^\mu(P_F(\tau|s_0;\theta_F),\widetilde{P}_B(\tau|s_0))+ \underbrace{\nabla_{\theta_{F}}\mathbb{E}_{P_{F,\mu}(\tau;\theta_F)}\left[\log Z^*-\log Z\right]}_{=0}\nonumber\\
=&\nabla_{\theta_{F}} D_{KL}^\mu(P_F(\tau|s_0;\theta_F),P_B(\tau|s_0)).\label{gradient_FF}
\end{align}
It should be emphasized that $P_B(\tau)$ is the ground-truth distribution with $P_B(x):=R(x)/Z^\ast$, while $\widetilde{P}_B(\tau)$ is the approximated one with $\widetilde{P}_B(x):=R(x)/Z$.
The gradients w.r.t. $\theta_Z$ can be written as:
\begin{align}
   \frac{1}{2}\nabla_{\theta_Z}\mathbb{E}_{P_F(\tau|s_0)}[L_{TB}(\tau;\theta_Z)]&=\frac{1}{2}\mathbb{E}_{P_F(\tau|s_0)}\left[\nabla_{\theta_Z}\left(c(\tau)+\log Z(\theta_Z)\right)^2\right]\nonumber\\
&=\mathbb{E}_{P_F(\tau|s_0)}\left[\left(c(\tau)+\log Z\right)\nabla_{\theta_Z}\log Z(\theta_Z)\right]\nonumber\\
&=[D_{KL}(P_F(\tau|s_0),\widetilde{P}_B(\tau|s_0))]\left[\nabla_{\theta_Z}\log Z(\theta_Z)\right]
\nonumber\\
&=\nabla_{\theta_Z}\frac{Z(\theta_Z)}{\widehat{Z}} D_{KL}(P_F(\tau|s_0),\widetilde{P}_B(\tau|s_0))\nonumber\\
&=\nabla_{\theta_Z} D_{KL}^{\mu(\cdot;\theta_Z)}(P_F(\tau|s_0),\widetilde{P}_B(\tau|s_0)).\label{gradient_Z}
\end{align}
Besides, we have:
\begin{align}  
\nabla_{\theta_Z} D_{KL}^{\mu(\cdot;\theta_Z)}(P_F(\tau|s_0),\widetilde{P}_B(\tau|s_0))=&[\nabla_{\theta_Z}\log Z(\theta_Z)]\left[D_{KL}(P_F(\tau|s_0),\widetilde{P}_B(\tau|s_0))+\log\frac{Z^*}{Z^*}\right]\nonumber\\
    =&\left[\nabla_{\theta_Z}\log Z(\theta_Z)\right]\left[D_{KL}(P_F(\tau|s_0),P_B(\tau|s_0))+\log \frac{Z}{Z^*}\right]\nonumber\\
    =&\nabla_{\theta_Z} \frac{Z(\theta_Z)}{\widehat{Z}}\left[D_{KL}(P_F(\tau|s_0),P_B(\tau|s_0))\right]+\left[\nabla_{\theta_Z}\log\frac{ Z(\theta_Z)}{Z^*}\right]\left[\log\frac{Z}{Z^*}\right]\nonumber\\
    =&\nabla_{\theta_Z} D_{KL}^{\mu(\cdot;\theta_Z)}(P_F(\tau|s_0),P_B(\tau|s_0))+\frac{1}{2}\nabla_{\theta_Z}\left(\log Z(\theta_Z)-\log Z^\ast\right)^2.\label{gradient_ZZ}
\end{align}
Combining equations \eqref{gradient_F} and \eqref{gradient_Z}, we obtain:
\begin{equation}
    \frac{1}{2}\nabla_{\theta}\mathbb{E}_{P_{F,\mu}(\tau)}[ L_{TB}(\tau;\theta)]=\nabla_{\theta} D_{KL}^{\mu(\cdot;\theta)}(P_F(\tau|s_0;\theta),\widetilde{P}_B(\tau|s_0)).
\end{equation}
Combining equations \eqref{gradient_FF} and \eqref{gradient_ZZ}, we obtain:
\begin{equation}
        \frac{1}{2}\nabla_{\theta}\mathbb{E}_{P_{F,\mu}(\tau)}[L_{TB}(\tau;\theta)]=\nabla_{\theta} \left\{D_{KL}^{\mu(\cdot;\theta)}(P_F(\tau|s_0;\theta),P_B(\tau|s_0))+\frac{1}{2}\left(\log Z(\theta)-\log Z^*\right)^2 \right\}.
\end{equation}
Now let's consider the backward gradients and denote $c(\tau)=\left(\log\frac{P_B(\tau|x)}{P_F(\tau_{\preceq x})}\right)$. Then,
\begin{align}
       &\frac{1}{2}\nabla_{\phi} \mathbb{E}_{P_{B,\rho}(\tau)}[L_{TB}(\tau;\phi )]=\frac{1}{2} \mathbb{E}_{P_{B,\rho}(\tau)}\left[\nabla_{\phi} \left(c(\tau;\phi)+\log R(x)-\log Z-\log P_F(s^f|x)\right)^2\right]\nonumber\\&= \mathbb{E}_{P_{B,\rho}(\tau)}\left[\left(c(\tau)+\log R(x)-\log Z-\log P_F(s^f|x)\right)\nabla_{\phi}\log P_B(\tau|x;\phi)\right]\nonumber\\&= \mathbb{E}_{P_{B,\rho}(\tau)}\left[c(\tau)\nabla_{\phi}\log P_B(\tau|x;\phi)  \right]+\mathbb{E}_{\rho(x)}\big[\left(\log R(x)-\log Z-\log P_F(s^f|x)\right)\underbrace{\mathbb{E}_{P_{B}(\tau|x)}[\nabla_{\phi}\log P_B(\tau|x;\phi)]}_{=0 \text{ by Lemma \ref{reforce-trick}}}\big]\nonumber
       \\
       &=\mathbb{E}_{P_{B,\rho}(\tau)}[c(\tau)\nabla_{\phi}\log P_B(\tau|x;\phi)]+\underbrace{\mathbb{E}_{P_{B,\rho}(\tau)}[\nabla_{\phi}c(\tau;\phi)]}_{=0 \text{ by Lemma \ref{reforce-trick}}}\nonumber\\&=\mathbb{E}_{\rho(x)}\big[\nabla_{\phi} D_{KL}(P_{B}(\tau|x;\phi),\widetilde{P}_{F}(\tau|x))\big]\nonumber\\
       &=\nabla_{\phi} D_{KL}^{\rho}(P_{B}(\tau|x;\phi),\widetilde{P}_{F}(\tau|x))\label{gradient_B}.
\end{align}
Besides, we have
\begin{align}
       \nabla_{\phi} D_{KL}^{\rho}(P_{B}(\tau|x;\phi),\widetilde{P}_{F}(\tau|x))&=\nabla_{\phi} D_{KL}^\rho(P_{B,\rho}(\tau|x;\phi),\widetilde{P}_F(\tau|x))+\mathbb{E}_{\rho(x)}\big[\underbrace{\nabla_{\phi}\mathbb{E}_{P_B(\tau|x;\phi)}\left[\log (P_F^{\top}(x)/P_F(s^f|x))\right]}_{=\log (P_F^\top(x)/P_F(s^f|x))\nabla_\phi 1=0}\big]\nonumber\\
       &=\nabla_{\phi} D_{KL}^{\rho}(P_{B}(\tau|x;\phi),\widetilde{P}_F(\tau|x))+\nabla_\phi\mathbb{E}_{P_{B,\rho}(\tau;\phi)}\left[\log (P_F^{\top}(x)/P_F(s^f|x))\right]\nonumber\\&=\nabla_{\phi} D_{KL}^{\rho}(P_{B}(\tau|x;\phi),P_F(\tau|x)).\label{gradient_BB}
\end{align}
Equations~\eqref{gradient_B} and~\eqref{gradient_BB} are the expected results.
\end{proof}

\subsection{Proof of Proposition \ref{guided-equivalence}}\label{proof-guided-equivalence}
\begin{proof}
The proof can be done by a procedure similar to that of backward gradients in Proposition \ref{TB-equivalence} by replacing $\widetilde{P}_F(\tau|x)$ 
with $P_G(\tau|x)$. 
\end{proof}

\subsection{Sub-trajectory Equivalence}\label{Sub-TB-equivalence}
Proposition 2 in the paper by~\citet{malkin2022gflownets} only considered the gradients of the Sub-TB objective~\citep{madan2023learning} w.r.t. $P_F(\cdot|\cdot)$ and $P_B(\cdot|\cdot)$. We provide an extended proposition below that also takes the gradients w.r.t. state flow estimator $F(\cdot)$ into consideration. For any $m<n$ and $n,m \in \{1,T-1\}$, we denote the set of sub-trajectories that start at some state in $\mathcal{S}_m$ and end in some state in $\mathcal{S}_n$ as $\Bar{\mathcal{T}}=\{\bar{\tau}=(s_m{\rightarrow}\ldots {\rightarrow} s_n)|\forall t\in \{m,\ldots,n-1\}: (s_t{\rightarrow} s_{t+1}) \in \mathcal{A}_t\}$. The sub-trajectory objective $\mathcal{L}_{Sub-TB}(P_\mathcal{D})$ is defined by:
\begin{equation}
\begin{split}
    \mathcal{L}_{Sub-TB}(P_\mathcal{D})=\mathbb{E}_{P_\mathcal{D}(\bar{\tau})}[L_{Sub-TB}(\bar{\tau})],\quad L_{Sub-TB}(\bar{\tau})=\log\left(\frac{\ P_{F}(\bar{\tau}|s_m)F(s_m)}{P_{B}(\bar{\tau}|s_n)F(s_n)}\right)^2. 
\end{split}
\end{equation}
In the equations above, $P_F(\bar{\tau}|s_m)=\prod_{t=m}^{n-1}{P_F(s_{t+1}|s_t)}$, $P_B(\bar{\tau}|s_n)=\prod_{t=m}^{n-1}{P_B(s_t|s_{t+1})}$ and $F(s_n=x):=R(x)$. Besides, we define $\mu(s_m):=F(s_m)/\widehat{Z}_m$ and  $\rho(s_n):=F(s_n)/\widehat{Z}_n$ where  $\widehat{Z}_m$ and $\widehat{Z}_n$ are the two normalizing constants whose values are clamped to $\sum_{s_m}F(s_m)$ and $\sum_{s_n}F(s_n)$. 

Furthermore, 
$P_{F,\mu}(\bar{\tau}):=\mu(s_m)P_F(\bar{\tau}|s_m)$ and $P_{B,\rho}(\bar{\tau}):=\rho(s_n)P_B(\bar{\tau}|s_{n})$ so that $P_{F,\mu}(\bar{\tau}|s_n)=P_{F,\mu}(\bar{\tau})/\rho^*(s_n)$ and  $P_{B,\rho}(\bar{\tau}|s_m)=P_{B,\rho}(\bar{\tau})/\mu^*(s_m)$. Here, $\rho^*(s_n):=F^*(s_n)/\widehat{Z}^*_n$, $\widehat{Z}^*_n$ is clamped to $\sum_{s_n}F^*(s_n)$, and $F^*(s_n):=\sum_{\bar{\tau}:s_n \in \bar{\tau}}F(s_m)P_F(\bar{\tau}|s_m)$ is the ground-truth state flow over $\mathcal{S}_n$ implied by $P_F$; $\mu^*(s_m):=F^*(s_m)/\widehat{Z}^*_m$,  $\widehat{Z}^*_m$ is clamped to $\sum_{s_m}F^*(s_m)$, and $F^*(s_m):=\sum_{\bar{\tau}:s_m \in \bar{\tau}}F(s_n)P_B(\bar{\tau}|s_n)$ is the ground-truth state flow over $\mathcal{S}_m$ implied by $P_B$.

\begin{proposition}
    For a forward policy  $P_F(\cdot|\cdot;\theta)$, a backward policy $P_B(\cdot|\cdot;\phi)$, a state flow estimator $F(\cdot;\theta)$ for $\mathcal{S}_m$, and a state flow estimator $F(\cdot;\phi)$ for $\mathcal{S}_n$\footnote{Here, $F_{\theta}$ and $F_{\phi}$ actually share the same parameters and represent the same flow estimator $F$. Model parameters are duplicated just for the clarity of gradient equivalences. Therefore the true gradient of the state flow estimator $F$ is $\nabla_{\theta}F_{\theta}+\nabla_{\phi}F_{\phi}$.},  the gradients of Sub-TB can be written as:
\begin{align}
 \frac{1}{2}\nabla_{\theta} \mathcal{L}_{Sub-TB}(P_{F,\mu};\theta)&=\nabla_\theta D_{KL}^{\mu(\cdot;\theta)}(P_F(\bar{\tau}|s_m;\theta),P_{B,\rho}(\bar{\tau}|s_m))+\nabla_\theta D_{KL}(\mu(s_m;\theta),\mu^*(s_m))
\nonumber\\&=\nabla_\theta D_{KL}^{\mu(\cdot;\theta)}(P_F(\bar{\tau}|s_m;\theta),\widetilde{P}_{B,\rho}(\bar{\tau}|s_m)),\nonumber\\
\frac{1}{2}\nabla_{\phi} \mathcal{L}_{Sub-TB}(P_{B,\rho};\phi)&=\nabla_{\phi}D_{KL}^{\rho(\cdot;\phi)}(P_B(\bar{\tau}|s_{n};\phi),P_{F,\mu}(\bar{\tau}|s_n))+\nabla_\phi D_{KL}(\rho(s_n;\phi),\rho^*(s_n))\nonumber\\
 &=\nabla_{\phi} D_{KL}^{\rho(\cdot;\phi)}(P_B(\bar{\tau}|s_{n};\phi),\widetilde{P}_{F,\mu}(\bar{\tau}|s_n)), 
\end{align}
where $\widetilde{P}_{F,\mu}(\bar{\tau}|s_n):=P_{F,\mu}(\bar{\tau})/\rho(s_n)$ and $\widetilde{P}_{B,\rho}(\bar{\tau}|s_{m}):=P_{B,\rho}(\bar{\tau})/\mu(s_m)$ are approximation to $P_{F,\mu}(\bar{\tau}|s_n)$ and  $P_{B,\rho}(\bar{\tau}|s_m)$.
\end{proposition}
\begin{proof}
First of all, we split the parameters of the state flow estimator and forward transition probability and denote them as $F(\cdot;\theta_M)$ and $P_F(\cdot|\cdot;\theta_F)$ respectively. We further define $c(\bar{\tau})=\left(\log\frac{P_F(\bar{\tau}|s_m)}{F(s_n)P_B(\bar{\tau}|s_{n})}\right)$.   

For the gradients w.r.t. $\theta_F$:
\begin{align}
&\frac{1}{2}\nabla_{\theta_F}\mathbb{E}_{P_{F,\mu}(\bar{\tau})}[L_{Sub-TB}(\bar{\tau};\theta_F)]\nonumber=\frac{1}{2}\mathbb{E}_{P_{F,\mu}(\bar{\tau})}\left[\nabla_{\theta_{F}}\left(c(\bar{\tau};\theta_F)+\log F(s_m)\right)^2\right]
\nonumber\\
&=\mathbb{E}_{P_{F,\mu}(\bar{\tau})}\left[(c(\bar{\tau})+\log F(s_m))\nabla_{\theta_{F}}\log P_F(\bar{\tau}|s_m;\theta_F)\right]+\underbrace{\mathbb{E}_{P_{F,\mu}(\bar{\tau})}\left[\nabla_{\theta_F}(c(\bar{\tau};\theta_F)+\log F(s_m))\right]}_{=0 \text { by Lemma \ref{reforce-trick}}}\nonumber\\&=\nabla_{\theta_{F}}\mathbb{E}_{P_{F,\mu}(\bar{\tau};\theta_F)}\left[(c(\bar{\tau};\theta_F)+\log F(s_m))\right]\nonumber+\underbrace{\nabla_{\theta_{F}}\mathbb{E}_{P_{F,\mu}(\bar{\tau};\theta_F)}[\log \widehat{Z}_n-\log\widehat{Z}_m]}_{=0}\\
&=\nabla_{\theta_{F}}D_{KL}^{\mu}(P_F(\bar{\tau}|s_m;\theta_F),\widetilde{P}_{B,\rho}(\bar{\tau}|s_m)).\label{sub-gradient_F}
\end{align}
Besides,
\begin{align}
&\frac{1}{2}\nabla_{\theta_F}\mathbb{E}_{P_{F,\mu}(\bar{\tau})}[L_{Sub-TB}(\bar{\tau};\theta_F)]\nonumber=\frac{1}{2}\mathbb{E}_{P_{F,\mu}(\bar{\tau})}\left[\nabla_{\theta_{F}}\left(c(\bar{\tau};\theta_F)+\log F(s_m)\right)^2\right]
\nonumber\\
&=\mathbb{E}_{P_{F,\mu}(\bar{\tau})}\left[c(\bar{\tau})\nabla_{\theta_{F}}\log P_F(\bar{\tau}|s_m;\theta_F)\right]+\mathbb{E}_{\mu(s_m)}\big[\log F(s_m)\underbrace{\mathbb{E}_{P_F(\bar{\tau}|s_m)}\left[\nabla_{\theta_{F}}\log P_F(\bar{\tau}|s_m;\theta_F)\right]}_{=0 \text { by Lemma \ref{reforce-trick}}}\big]\nonumber\\
&=\mathbb{E}_{P_{F,\mu}(\bar{\tau})}\left[c(\bar{\tau})\nabla_{\theta_{F}}\log P_F(\bar{\tau}|s_m;\theta_F)\right]+\underbrace{\mathbb{E}_{P_{F,\mu}(\bar{\tau})}\left[\nabla_{\theta_{F}} c(\bar{\tau};\theta_F)\right]}_{=0}\nonumber\\&=\nabla_{\theta_F}\mathbb{E}_{P_{F,\mu}(\bar{\tau};\theta_F)}\left[c(\bar{\tau};\theta_F)\right]+\mathbb{E}_{\mu(s_m)}\big[\underbrace{\nabla_{\theta_F}\mathbb{E}_{P_F(\bar{\tau}|s_m;\theta_F)}[\log \mu^*(s_m)]}_{=\log \mu^*(s_m)\nabla_{\theta_F}1=0}\big]\nonumber+\underbrace{\nabla_{\theta_{F}}\mathbb{E}_{P_{F,\mu}(\bar{\tau};\theta_F)}[\log \widehat{Z}_n]}_{=0}\\&=\nabla_{\theta_F}\mathbb{E}_{P_{F,\mu}(\bar{\tau};\theta_F)}\left[c(\bar{\tau};\theta_F)\right]+\nabla_{\theta_F}\mathbb{E}_{P_{F,\mu}(\bar{\tau};\theta_F)}[\log \mu^*(s_m)+\log\widehat{Z}_n]\nonumber\\
&=\nabla_{\theta_{F}}D_{KL}^{\mu}(P_F(\bar{\tau}|s_m;\theta_F),P_{B,\rho}(\bar{\tau}|s_m))].\label{sub-gradient_FF}
\end{align}
For the gradients w.r.t. $\theta_M$, we have:
\begin{align}
&\frac{1}{2}\nabla_{\theta_M}\mathbb{E}_{P_{F,\mu}(\bar{\tau})}[L_{Sub-TB}(\bar{\tau};\theta_M)]=\frac{1}{2}\mathbb{E}_{P_{F,\mu}(\bar{\tau})}\left[\nabla_{\theta_{M}}\left(c(\bar{\tau})+\log F(s_m;\theta_M)\right)^2\right]
\nonumber\\
&=\mathbb{E}_{P_{F,\mu}(\bar{\tau})}\left[(c(\bar{\tau})+\log F(s_m))\nabla_{\theta_{M}}\log F(s_m;\theta_M)\right]+\underbrace{\mathbb{E}_{P_{F,\mu}(\bar{\tau})}[(\log \widehat{Z}_n-\log \widehat{Z}_m)\nabla_{\theta_{M}}\log F(s_m;\theta_M)]}_{=0 \text { by Lemma \ref{reforce-trick}}}\nonumber\\
&=\mathbb{E}_{\mu(s_m)}\left[D_{KL}(P_F(\bar{\tau}|s_m),\widetilde{P}_B(\bar{\tau}|s_m))\nabla_{\theta_{M}}\log \frac{F(s_m;\theta_M)}{\widehat{Z}_m}\right]\nonumber\\
&=\nabla_{\theta_{M}}D_{KL}^{\mu(\cdot;\theta_M)}(P_F(\bar{\tau}|s_m),\widetilde{P}_B(\bar{\tau}|s_m)).\label{sub-gradient_Z}
\end{align}
Besides,
\begin{align}
&\frac{1}{2}\nabla_{\theta_M}\mathbb{E}_{P_{F,\mu}(\bar{\tau})}[L_{Sub-TB}(\bar{\tau};\theta_M)]=\frac{1}{2}\mathbb{E}_{P_{F,\mu}(\bar{\tau})}\left[\nabla_{\theta_{M}}\left(c(\bar{\tau})+\log F(s_m;\theta_M)\right)^2\right]
\nonumber\\
&=\mathbb{E}_{P_{F,\mu}(\bar{\tau})}\left[c(\bar{\tau})\nabla_{\theta_{F}}\log F(s_m;\theta_M)\right]+\mathbb{E}_{\mu(s_m)}\left[\log F(s_m)\nabla_{\theta_{M}} \log F(s_m;\theta_M)\right]\nonumber\\
&=\mathbb{E}_{P_{F,\mu}(\bar{\tau})}\left[(c(\bar{\tau})+\log \mu^*(s_m))\nabla_{\theta_{M}}\log F(s_m;\theta_M)\right]+\mathbb{E}_{\mu(s_m)}\left[(\log F(s_m)-\log \mu^*(s_m))\nabla_{\theta_{M}} \log F(s_m;\theta_M)\right]\nonumber\\&\quad+\underbrace{\mathbb{E}_{P_{F,\mu}(\bar{\tau})}[(\log \widehat{Z}_n-\log \widehat{Z}_m)\nabla_{\theta_{M}}\log F(s_m;\theta_M)]}_{=0 \text { by Lemma \ref{reforce-trick}}}
\nonumber\\
&=\mathbb{E}_{\mu(s_m)}\left[D_{KL}(P_F(\bar{\tau}|s_m),P_{B,\rho}(\bar{\tau}|s_m))\nabla_{\theta_{M}}\log\frac{F(s_m;\theta_M)}{\widehat{Z}_m}\right]\nonumber\\&\quad+\mathbb{E}_{\mu(s_m)}\left[\left(\log \mu(s_m)-\log \mu^*(s_m)\right)\nabla_{\theta_{M}} \log \frac{F(s_m;\theta_M)}{ \widehat{Z}_m}\right]+\underbrace{\mathbb{E}_{\mu(s_m)}\left[\nabla_{\theta_{M}}\left(\log\mu(s_m;\theta_M)-\log\mu^\ast(s_m)\right)\right]}_{=0\text{ By Lemma~\ref{reforce-trick}}}\nonumber\\&=\nabla_{\theta_{M}}D_{KL}^{\mu(\cdot;\theta_M)}(P_F(\bar{\tau}|s_m),P_{B,\rho}(\bar{\tau}|s_m))+\nabla_{\theta_M}D_{KL}(\mu(s_m;\theta_M),\mu^*(s_m)).\label{sub-gradient_ZZ}
\end{align}
Combining equations \eqref{sub-gradient_F} and \eqref{sub-gradient_Z}, we obtain
\begin{equation}
    \frac{1}{2}\nabla_{\theta}\mathbb{E}_{P_{F,\mu}(\bar{\tau})}[L_{Sub-TB}(\bar{\tau};\theta)]=\nabla_{\theta}D_{KL}^{\mu(\cdot;\theta)}(P_F(\bar{\tau}|s_m;\theta),\widetilde{P}_B(\bar{\tau}|s_m)).
\end{equation}
Combining equations \eqref{sub-gradient_FF} and \eqref{sub-gradient_ZZ}, we obtain
\begin{equation}
    \frac{1}{2}\nabla_{\theta}\mathbb{E}_{P_{F,\mu}(\bar{\tau})}[L_{Sub-TB}(\bar{\tau};\theta)]=\nabla_{\theta}\left\{D_{KL}^{\mu(\cdot;\theta)}(P_F(\bar{\tau}|s_m;\theta),P_B(\bar{\tau}|s_m))+D_{KL}(\mu(s_m;\theta),\mu^*(s_m))\right\}.
\end{equation}
Splitting $\phi$ into $\phi_B$ and $\phi_M$ and denoting $c(\bar{\tau})=\log \frac{P_B(\bar{\tau}|s_n)}{F(s_m)P_F(\bar{\tau}|s_m)}$, the gradient derivation of $\phi$ follows the similar way as $\theta$, and is omitted here.
\end{proof}
\section{RL Framework}
\subsection{Derivation of RL Functions} \label{RL_def}
Let's first consider the case of forward policies. For any  $s\in \mathcal{S}_t$ and $a=(s{\rightarrow} s')\in \mathcal{A}(s)$ with $t\in \{0,\ldots,T-1\}$, we define the $V_{F,t}$ and $Q_{F,t}$ as:
\begin{align}
V_{F,t}(s)&:=\mathbb{E}_{P_F(\tau_{> t}|s_t)}\left[\sum_{l=t}^{T-1}{R_{F}(s_{l},a_{l})}\bigg|s_t=s\right]\nonumber\\&=R_F(s)+\mathbb{E}_{P_F(s_{t+1}|s_t)}\left[\mathbb{E}_{P_F(\tau_{> t+1}|s_{t+1})}\left[\sum_{l=t+1}^{T}{R_{F}(s_{l},a_{l})}\bigg|s_{t+1}=s'\right]\bigg|s_t=s\right]\nonumber\\&=R_F(s)+ \mathbb{E}_{\pi_F(s,a)}[V_{F,t+1}(s')],
\nonumber\\
Q_{F,t}(s,a)&:=\mathbb{E}_{P_{F}(\tau_{>t+1}|s_t,a_t)}\left[\sum_{l=t}^{T-1}{ R_{F}(s_{l},a_{l})}\bigg|s_t=s,a_t=a\right]\nonumber\\&=R_F(s,a)+\mathbb{E}_{P_F(\tau_{>t+1}|s_{t+1})}\left[\sum_{l=t+1}^{T}R_{F}(s_{l},a_{l})\bigg|s_{t+1}=s'\right]
\nonumber\\&=R_F(s,a)+ V_{F,t+1}(s'),
\end{align}
where $R_F(s):=\mathbb{E}_{\pi_F(s,a)}[R_F(s,a)]$, $V_{F,T}(\cdot):=0$, and $Q_{F,T}(\cdot,\cdot):=0$. 
Since $S_t \cap S_{t^\prime}=\emptyset$ for any $t\neq t^\prime$, we can read off the time indices (topological orders) from state values. Plus the fact that $R_F(s,a):=0$ for any $a\notin A(s)$, we are allowed to define two universal functions $V_F:\mathcal{S}\rightarrow \mathbb{R}$  and $Q_F:\mathcal{S}\times\mathcal{A} \rightarrow \mathbb{R}$ such that $V_F(s_t=s):=V_{F,t}(s)$ and $Q_F(s_t=s,a):=Q_{F,t}(s,a)$.
\begin{remark}
While the transition environment $\mathcal{G}$ is exactly known, the state space $\mathcal{S}$ can be exponentially large, making the exact values of $V$ and $Q$ intractable. This fact, in spirit, corresponds to a regular RL problem where the exact values of $V$ and $Q$ are infeasible due to the unknown and uncertain transition environment $P(s'|s, a)$.
\end{remark}
For backward policies, rewards are accumulated from $T-1$ to $1$. Similarly, for $s'\in \mathcal{S}_t$ and $a=(s {\rightarrow} s')\in \dot{\mathcal{A}}(s')$,
\begin{align}
V_{B,t}(s')&:=\mathbb{E}_{P_B(\tau_{<t}|s_t)}\left[\sum_{l=1}^{t}{R_{B}(s_{l},a_{l})}\bigg|s_{t}=s'\right]=R_B(s')+ \mathbb{E}_{\pi_B(s',a)}[V_{B,t-1}(s)],\nonumber\\
Q_{B,t}(s',a)&:=\mathbb{E}_{P_{B}(\tau_{<t-1}|s_t,a_t)}\left[\sum_{l=1}^{t}{ R_{B}(s_{l},a_{l})}\bigg|s_t=s',a_t=a\right]=R_B(s',a)+ V_{B,t-1}(s),
\end{align}
where $R_B(s')=\mathbb{E}_{\pi_B(s',a)}[R_B(s',a)]$, $V_{B,0}(\cdot):=0$, and $Q_{B,0}(\cdot,\cdot):=0$. For the same reason as forward policies, we can define universal functions 
$V_B:\mathcal{S}\rightarrow \mathbb{R}$  and $Q_B:\mathcal{S}\times\dot{\mathcal{A}} \rightarrow \mathbb{R}$ such that $V_B(s_t=s'):=V_{B,t}(s')$ and $Q_B(s_t=s',a):=Q_{B,t}(s',a)$.

Based on the definitions above, the expected value functions are defined as:
\begin{align}
    J_F:=\mathbb{E}_{\mu(s_0)}[V_F(s_0)],\quad J_B:=\mathbb{E}_{\rho(x)}[V_B(x)].    
\end{align}
By definitions, $V_F(s_0)=\mathbb{E}_{P_F(\tau|s_0)}[\sum_{t=0}^{T-1}{R_{F}(s_{l},a_{l})}|s_0]=D_{KL}(P_{F}(\tau|s_0),\widetilde{P}_{B}(\tau|s_0))$
, so $J_F=D_{KL}^{\mu}(P_{F}(\tau|s_0),\widetilde{P}_{B}(\tau|s_0))$. Likewise, we can obtain $ J_B=D_{KL}^{\rho}(P_{B}(\tau|x),\widetilde{P}_{F}(\tau|x))$. 
The advantages functions are defined as: 
\begin{align}
A_F(s,a):=Q_F(s,a)-V_F(s),\quad A_B(s^\prime,a):=Q_B(s^\prime,a)-V_B(s^\prime).    
\end{align}
We define the forward accumulated state distribution as $d_{F,\mu}(s):=\frac{1}{T}\sum_{t=0}^{T-1}{P_{F,\mu}(s_t=s)}$ such that for arbitrary function $f:\mathcal{S}\times\mathcal{A}\rightarrow \mathbb{R}$, 
\begin{align}
        &\mathbb{E}_{P_{F,\mu}(\tau)}\left[\sum_{t=0}^{T-1} f(s_t,a_t)\right]=   \sum_{t=0}^{T-1}\mathbb{E}_{P_{F,\mu}(s_t\rightarrow s_{t+1})}\left[ f(s_t,a_t)\right] =\sum_{t=0}^{T-1} \mathbb{E}_{P_{F,\mu}(s_t),\pi_F(s_t,a_t)}\left[f(s_t,a_t)\right]\nonumber \\& =\sum_{t=0}^{T-1}\sum_s^\mathcal{S} P_{F,\mu}(s_t=s)\sum_a^\mathcal{A} \pi_F(s,a) f(s,a)\label{acc-dist}= \sum_s^\mathcal{S} \sum_a^\mathcal{A} 
\left(\sum_{t=0}^{T-1} P_{F,\mu}(s_t=s)\right) \pi_F(s,a) f(s,a)\nonumber\\&=T\,\sum_s^\mathcal{S} \sum_a^\mathcal{A} 
  d_{F,\mu}(s) \pi_F(s,a) f(s,a)=T\, \mathbb{E}_{d_{F,\mu}(s),\pi_F(s,a)}[f(s,a)],
\end{align}
where equation~\eqref{acc-dist}  holds in that  $\forall s\notin \mathcal{S}_t: P(s_t=s)=0$ and $\forall a\notin \mathcal{A}(s): \pi_F(s,a)=0$.
By the fact that $S_t \cap S_{t^\prime}=\emptyset$ for any $t\neq t^\prime$ and any trajectory $\tau\in \mathcal{T}$ must pass some $s_t\in \mathcal{S}_t$ for $t\in \{0,\ldots,T-1\}$, $P_{F,\mu}(s_t)$ is a valid distribution over $\mathcal{S}_t$ and $\sum_{s_t}P_{F,\mu}(s_t)=1$. Accordingly, $d_{F,\mu}(s)$ is a valid distribution over $\mathcal{S}$ and $T\,d_{F,\mu}(s_t)=P_{F,\mu}(s_t)$. Analogously, we can define 
$d_{B,\rho}(s^\prime):=\frac{1}{T-1}\sum_{t=1}^{T-1}{P_{B,\rho}(s_t=s')}$ such that 
for arbitrary function $f:\mathcal{S}\times\dot{\mathcal{A}}\rightarrow \mathbb{R}$, 
\begin{align}
& \mathbb{E}_{P_{B,\rho}(\tau)}\left[\sum_{t=1}^{T-1} f(s_t,a_t)\right]=(T-1)\,\mathbb{E}_{d_{B,\rho}(s'),\pi_B(s',a)}[f(s',a)].
\end{align}

\subsection{DAGs as Transition Environments}\label{environment}
\begin{theorem}~\citep{golpar2019nilpotent}\label{DAG} Let  $P\in \mathbb{R}^{N\times N}$  be a non-negative matrix. The following statements are equivalent:
\vspace{-3mm}
\begin{enumerate}
\item$P$ is nilpotent;
	\item$P^N$ =0;
	\item The directed graph $\mathcal{G}(\mathcal{S} ,\mathcal{A})$ associated with $P$ is a DAG;
	\item There exists a permutation matrix $U$  such that $U^TPU$  is a strictly triangular matrix.
\end{enumerate}\vspace{-3mm}
where  $\mathcal{S}=\{s^0,\ldots,s^{N-1}\}$ and $\mathcal{A}=\{(s^i{\rightarrow} s^j)|P_{i,j}\neq0\}$ are node and edge sets.
\end{theorem}
\begin{lemma}\label{np-order} For any DAG graph $\mathcal{G}(\mathcal{S},\mathcal{A})$ associated with $P\in R^{N\times N}$  with $T+1(\le N)$ different topological node orders indexed by integers $[0,T]$,
\begin{equation}
    \forall t > T,\quad P^t=\mathbf{0}.
\end{equation}
\end{lemma}

\begin{proof}
We prove the result by contradiction. Assuming $P^t( t> T)$ is not zero, then $\exists\,i\neq j$:  
\begin{equation}
[P^t]_{i,j}=\sum_{k_{1:{t-1}}}P_{i,k_1}P_{k_1,k_2}\ldots P_{k_{t-1},j}>0.
\end{equation}
By the nature of DAGs, $\forall (s'{\rightarrow} s)\in \mathcal{A}: s^\prime\prec s$. Then the above expression is equal to:
\begin{align}
    [P^t]_{i,j}&=\sum_{k_1: s^i\prec s^{k_1}} P_{i,k_1}\left(\sum_{k_2:s^{k_1}\prec s^{k_2}}P_{k_1,k_2}\ldots\left(\sum_{k_{t-1}:s^{k_{t-2}}\prec s^{k_{t-1}}}P_{k_{t-2},k_{t-1}}P_{k_{t-1},j}\right)\right)\nonumber\\
&=\sum_{k_{1:t-1}: (s^i\prec s^{k_1}\prec\ldots \prec s^{k_{t-1}}\prec s^j)}{P_{i,k_1}P_{k_1,k_2}\ldots P_{k_{t-1},j}}\nonumber\\&>0.   
\end{align}
This means that there at least exists a trajectory $(s^i\prec s^{k_1}\prec \ldots \prec s^{k_{t-1}}\prec s^j)$ with non-zero probability. However, there are $t+1$ distinct node orders in the path, which contradicts the assumption that there are $T+1$ different node orders. 
\end{proof}

Let's return to the graded DAG, $\mathcal{G}(\mathcal{S}, \mathcal{A})$ in GFlowNets. For the easiness of analysis, we restrict forward and backward policies and initial distribution to be tabular forms, $P_F\in \mathbb{R}^{|\mathcal{S}|\times |\mathcal{S}|}$, $\mu\in \mathbb{R}^{|\mathcal{S}|}$, $P_B\in \mathbb{R}^{|\mathcal{S}|\times |\mathcal{S}|}$, and $\rho \in \mathbb{R}^{|\mathcal{S}|}$ such that $P_F(s^j|s^i)=[P_F]_{j,i}$ and $P_B(s^i|s^j)=[P_B]_{i,j}$. 
Besides, we split initial distribution vectors by $\mu=(\bar{\mu};0) \in \mathbb{R}^{|\mathcal{S}|}$ and $\rho=(0;\bar{\rho})\in \mathbb{R}^{|\mathcal{S}|}$, where $\bar{\mu}$ and $\bar{\rho}$ denote the probabilities of states except $s^f$ and $s^0$ respectively. We denote the graph equipped with a self-loop over $s^f$ as $\mathcal{G}_F(\mathcal{S},\mathcal{A}\cup\{{(s^f{\rightarrow} s^f)}\})$, and the reverse graph equipped with a self-loop over $s^0$ as $\mathcal{G}_B(\mathcal{S},\dot{\mathcal{A}}\cup\{{(s^0{\rightarrow} s^0)}\})$. Accordingly, we enhance $P_F$ and $P_B$ by defining $P_F(s^f|s^f):=1$ and $P_B(s^0|s^0):=1$. $(\mathcal{G}_F, P_F)$ specifies an absorbing Markov Chains: $s^f$ is the only absorbing state as only the self-loop is allowed once entering $s^f$; the sub-graph over $S\setminus\{s^f\}$, denoted as $\Bar{\mathcal{G}}_F$ is still a DAG, so any state $s\in S\setminus\{s^f\}$ is transient as it can be visited at most one time. Similarly, $(\mathcal{G}_B, P_B)$ specifies another absorbing Markov Chain with absorbing state $s^0$ and a DAG over $S\setminus\{s^0\}$, denoted as $\Bar{\mathcal{G}}_B$. For graph $\mathcal{G}_F$ and $\mathcal{G}_B$, their transition matrices $P_F$ and $P_B$ can be decomposed into:
\begin{equation}
P_F=\left(\begin{matrix}\bar{P}_F&\mathbf{0}\\r_F&1\\\end{matrix}\right), P_B=\left(\begin{matrix}1&r_B\\\mathbf{0}&\bar{P}_B\\\end{matrix}\right).
\end{equation}
In the equations above, $r_F\in \mathbb{R}^{1\times(|\mathcal{S}|-1)}$ and $ r_B\in \mathbb{R}^{1\times(|\mathcal{S}|-1)}$ denote the  probabilities of $(s{\rightarrow} s^f)$ for any $ s\in \mathcal{S}\setminus\{s^f\}$ and $(s^0 \leftarrow s)$ for any $s\in \mathcal{S}\setminus\{s^0\}$ respectively;  $\bar{P}_F\in \mathbb{R}^{(|\mathcal{S}|-1)\times (|\mathcal{S}|-1)}$ and $\bar{P}_B\in \mathbb{R}^{(|\mathcal{S}|-1)\times (|\mathcal{S}|-1)}$ denote probability of $(s {\rightarrow} s')$ for any $s,s' \in \mathcal{S}\setminus\{s^f\}$ and $(s \leftarrow s')$ for any $s,s' \in \mathcal{S}\setminus\{s^0\}$, that is, the transition matrices over $\Bar{\mathcal{G}}_F$ and $\Bar{\mathcal{G}}_B$ respectively.
\begin{lemma}For $(\mathcal{G}_, P_F,\mu)$ and $(\mathcal{G}_B, P_B,\rho)$, $d_{F,\mu}\in \mathbb{R}^{|S|}$\ and $d_{B,\rho} \in \mathbb{R}^{|S|}$, can be written in the following forms:
\begin{align} 
d_{F,\mu}&=\left(\begin{matrix}\bar{d}_{F,\mu}\\0\end{matrix}\right),\quad\bar{d}_{F,\mu}=\frac{1}{T}(I-\bar{P}_F)^{-1}\bar{\mu},\nonumber\\
    d_{B,\rho}&=\left(\begin{matrix}
        0\\ \bar{d}_{B,\rho}
    \end{matrix}\right),\quad \bar{d}_{B,\rho}=\frac{1}{T-1}(I-\bar{P}_B)^{-1}\bar{\rho}.
\end{align}
\end{lemma}
\begin{proof}
We first prove the result for the forward case. By the nature of Markov Chains, $P_{F,\mu}(s_t=s^i)=[(P_F)^t\mu]_i$, and $d_{F,\mu}=\frac{1}{T}\sum_{t=0}^{T-1}(P_F)^t\mu$. Then, it can be easily verified~\citep{grinstead2006introduction} that:
\begin{equation}
(P_F)^t=\left(\begin{matrix}(\bar{P}_F)^t&\mathbf{0}\\\ast&1\\\end{matrix}\right),
\end{equation}
where the explicit expression of the upper right corner is omitted. By Theorem \ref{DAG}, $\bar{P}_F$ is a nilpotent matrix and by Lemma \ref{np-order}, $\sum_{t=0}^{T-1}(\bar{P}_F)^t =\sum_{t=0}^{\infty}(\bar{P}_F)^t=(I-\bar{P}_F)^{-1}$, where the first equality follows from the fact that $\Bar{\mathcal{G}}_F$ has $T$ topological orders, and the second equality is by the fact that $(I-\bar{P}_F)\sum_{t=0}^{\infty}(\bar{P}_F)^t=\sum_{t=0}^{\infty}(\bar{P}_F)^t-\sum_{t=1}^{\infty}(\bar{P}_F)^t=I$. Therefore,
\begin{align}
d_{F,\mu}=\frac{1}{T}\left(\begin{matrix}\sum_{t=0}^{T-1} {\bar{P}_F}^t&\mathbf{0}\\\ast&1\\\end{matrix}\right)\mu=\frac{1}{T}\left(\begin{matrix}(I-\bar{P}_F)^{-1}\bar{\mu}\\  \ast\bar{\mu} \end{matrix}\right).
\end{align}
By Theorem 11.4 in~\citet{grinstead2006introduction}, $[(I-\bar{P}_F)^{-1}]_{j,i}$ is the expected number of times the chain is in state $s^j$, starting from $s^i$, before being absorbed in $s^f$. And $[(I-\bar{P}_F)^{-1}\bar{\mu}]_j$ is the expected number of times the chain is in state $s^j$ before being absorbed. Since  $\forall s\notin \mathcal{S}_0:\, \mu(s)=0$ and  $\bar{\mathcal{G}}_F$ is graded,  any forward trajectory over sub-graph $\bar{\mathcal{G}}_F$ must start from $s\in\mathcal{S}_0$ and end in  $s\in \mathcal{S}_{T-1}$, meaning $\sum_j [(I-\bar{P}_F)^{-1}\bar{\mu}]_j=T$.   Thus,  $\frac{1}{T}[(I-\bar{P}_F)^{-1}\bar{\mu}]_j$ denotes the fraction of staying in transient state  $s^j$  before being absorbed, that is, the probability observing state $s^j$ within $T$ time steps. By the same reasoning, we can conclude that $ \ast\bar{\mu}=0$ as $s^f$ can not be reached within $T$ time steps.

For backward case, any backward trajectory over sub-graph $\bar{\mathcal{G}}_B$ must start from $s\in\mathcal{S}_{T-1}$ and end in  $s\in \mathcal{S}_{1}$ as $\forall s\notin \mathcal{S}_{T-1}: \rho(s)=0$ and  $\bar{\mathcal{G}}_B$ is graded. Then, a proof procedure for the desired result can be derived similarly, so it is omitted. 
\end{proof}

\begin{lemma}\label{d_dist} For two forward policy, $\pi_F$ and $\pi_F^\prime$, and two backward policy, $\pi_B$ and $\pi_B^\prime$, we have:
\begin{align}
        D_{TV}(d_{F,\mu}^\prime(\cdot),d_{F,\mu}(\cdot))\le  D_{TV}^{d_{F,\mu}^\prime}(\pi_F^\prime(s,\cdot),\pi_F(s,\cdot)), \nonumber\\     D_{TV}(d_{B,\rho}^\prime(\cdot),d_{B,\rho}(\cdot))\le  D_{TV}^{d_{B,\rho}^\prime}(\pi_B^\prime(s,\cdot),\pi_B(s,\cdot)),
\end{align}
where for three arbitrary distributions $p$,$q$ and $u$, $D_{TV}(p(\cdot),q(\cdot)):=\frac{1}{2}\left\|p(\cdot)-q(\cdot)\right\|_1$ and $D_{TV}^u(p(\cdot|s),q(\cdot|s)):=\frac{1}{2}\mathbb{E}_{u(s)}\left[\left\|p(\cdot|s)-q(\cdot|s)\right\|_1\right]$.
\end{lemma}
\begin{proof}
The proof procedure follows that of Lemma 3 in~\citet{achiam2017constrained}.
For two forward policy $\pi_F$ and $\pi_F^\prime$, let $\bar{N}_F:=(I-\bar{P}_F)^{-1}$ and $\bar{N}_F^{\prime}:=(I-\bar{P}_F^\prime)^{-1}$. Then, 
\begin{equation}
     \Delta:=\bar{P}_F-\bar{P}_F^\prime=(\bar{N}_F^{\prime})^{-1}-{\bar{N}_F}^{-1},
\end{equation}
and
\begin{equation}
\bar{N}_F-\bar{N}_F^\prime=\bar{N}_F\Delta\bar{N}_F^\prime.
\end{equation}
Then,
\begin{align}
 \left\|d_{F,\mu}-d_{F,\mu}^\prime\right\|_1&= \left\|\bar{d}_{F,\mu}-\bar{d}_{F,\mu}^\prime\right\|_1\nonumber\\&=\frac{1}{T}\left\|( \bar{N}_F-\bar{N}_F^\prime)\bar{\mu}\right\|_1=\frac{1}{T}\left\|\bar{N}_F\Delta \bar{d}_{F,\mu}^\prime\right\|_1 \nonumber\\
&\leq \frac{1}{T}\left\|\bar{N}_F\right\|_1\left\|\Delta \bar{d}_{F,\mu}^\prime\right\|_1\leq\frac{1}{T}\left(\sum_{t=0}^{T-1}\left\|P_F^t\right\|_1\right)\left\|\Delta \bar{d}_{F,\mu}^\prime\right\|_1\nonumber\\
&\le\left\|\Delta \bar{d}_{F,\mu}^\prime\right\|_1=\left\|(\bar{P}_F-\bar{P}_F^\prime)\bar{d}_{F,\mu}^\prime\right\|_1.
\end{align}
Therefore, we have
\begin{align}
\left\|d_{F,\mu}^\prime-d_{F,\mu}\right\|_1&\leq\left\|(\bar{P}_F^\prime-\bar{P}_F)\bar{d}_{F,\mu}^\prime\right\|_1\nonumber\\&\leq\left\|(\bar{P}_F^\prime-\bar{P}_F)\bar{d}_{F,\mu}^\prime\right\|_1+\left|(r_F^\prime-r_F)\bar{d}_{F,\mu}^\prime\right|=\left\|(P_F^\prime-P_F)d_{F,\mu}^\prime\right\|_1\nonumber\\&=\sum_s\left|\sum_{s^\prime}\left(P_F^\prime(s'|s)-P_F(s'|s)\right)d_{F,\mu}^\prime(s)\right|\nonumber\\ &\le\sum_{s,s'}\left|P_F^\prime(s'|s)-P_F(s'|s)\right|d_{F,\mu}^\prime(s)\nonumber\\&=\sum_{s,a}\left|\pi_F^\prime(s,a)-\pi_F(s,a)\right|d_{F,\mu}^\prime(s)
=\mathbb{E}_{d_{F,\mu}^\prime(s)}\left[\left\|\pi_F^\prime(s,\cdot)-\pi_F(s,\cdot)\right\|_1\right].
\end{align}

The result for backward policies can be derived analogously and is omitted here.
\end{proof}

\subsection{Derivation of Gradients}\label{policy_gradient}
\begin{proposition}
 The gradients of  $J_F(\theta)$ and $J_B(\phi)$ w.r.t. $\theta$ and $\phi$ can be written as:
\begin{align}
\nabla_{\theta}J_F(\theta)&=T\,\mathbb{E}_{d_{F,\mu}(s)\pi_F(s,a)}\left[Q_F(s,a)\nabla_{\theta} \log \pi_F(s,a;\theta)\right]+\mathbb{E}_{\mu\left(s_0\right)}[V_F(s_0)\nabla_{\theta}\log  \mu(s_0;\theta)]\nonumber\\&=T\,\mathbb{E}_{d_{F,\mu}(s)\pi_F(s,a)}\left[A_F(s,a)\nabla_{\theta} \log \pi_F(s,a;\theta)\right]+\mathbb{E}_{\mu\left(s_0\right)}[V_F(s_0)\nabla_{\theta}\log  \mu(s_0;\theta)],\nonumber\\
\nabla_{\phi}J_B(\phi)&=(T-1)\, \mathbb{E}_{d_{B,\rho}(s)\pi_B(s,a)}\left[Q_B(s,a)\nabla_{\phi} \log \pi_B(s,a;\phi)\right]\nonumber\\&=(T-1)\, \mathbb{E}_{d_{B,\rho}(s)\pi_B(s,a)}\left[A_B(s,a)\nabla_{\phi} \log \pi_B(s,a;\phi)\right].
\end{align}
\end{proposition}
\begin{remark}
This result implies that an estimated value function, which may differ from the exact one, does not lead to biased gradient estimation.\label{unbias-V}
\end{remark}
\begin{proof}
\begin{align}
{\nabla}_{\theta}J_{F}(\theta)&=
\mathbb{E}_{\mu\left(s_0\right)}[V_F(s_0)\nabla_{\theta}\log  \mu(s_0;\theta)]+\underbrace{\mathbb{E}_{\mu(s_0)}[{\nabla}_{\theta}V_F(s_0;\theta)]}_{(1)}\nonumber\\
&\overset{(1)}{=}\mathbb{E}_{P_{F,\mu}(s_0)}\left[{\nabla}_{\theta}\mathbb{E}_{\pi_F(s_0,a_0;\theta)}[Q_F(s_0,a_0;\theta)]\right]\nonumber\\
&=\mathbb{E}_{P_{F,\mu}(s_0)}\left[\mathbb{E}_{\pi_F(s_0,a_0)}[Q_F(s_0,a_0)\nabla_{\theta}\log\pi_F(s_0,a_0;\theta)+\nabla_{\theta}Q_F(s_0,a_0;\theta)]\right]\nonumber\\
&=\mathbb{E}_{P_{F,\mu}(s_0\rightarrow s_1)}\left[Q_F(s_0,a_0)\nabla_{\theta}\log\pi_F(s_0,a_0;\theta)\right]+\underbrace{\mathbb{E}_{P_{F,\mu}(s_0\rightarrow s_1)}\left[\nabla_{\theta}R_F(s_0,a_0;\theta)+\nabla_{\theta} V_{F}(s_1;\theta)\right]}_{(2)}\nonumber\\
&\overset{(2)}{=}\underbrace{\mathbb{E}_{P_{F,\mu}(s_0\rightarrow s_1)}\left[\nabla_\theta \log\frac{\pi_F(s_0,a_0;\theta)}{\pi_{B}(s_1,a_0)} \right]}_{(3)}+ \mathbb{E}_{P_{F,\mu}(s_1)}\left[\nabla_\theta V_F(s_1;\theta) \right]\nonumber\\
&\overset{(3)}{=}\mathbb{E}_{P_{F,\mu}(s_0)}\big[\underbrace{\mathbb{E}_{\pi_F(s_0,a_0)}[1\cdot \nabla_\theta \log\pi_F(s_0,a_0;\theta) ]}_{=0\text{ by Lemma~\ref{reforce-trick}}}\big].
\end{align}
Therefore,
\begin{align}
\mathbb{E}_{P_{F,\mu}(s_0)}[{\nabla}_{\theta}V_F(s_0;\theta)]\overset{(1)}{=}\mathbb{E}_{P_{F,\mu}(s_0\rightarrow s_1)}\left[Q_F(s_0,a_0)\nabla_{\theta}\log\pi_F(s_0,a_0;\theta)\right]+\mathbb{E}_{P_{F,\mu}(s_1)}\left[\nabla_\theta V_F(s_1;\theta) \right].
\end{align}
Keep doing the process,  we have 
\begin{equation}
\mathbb{E}_{P_{F,\mu}(s_t)}[{\nabla}_{\theta}V_F(s_t;\theta)]=\mathbb{E}_{P_{F,\mu}(s_t\rightarrow s_{t+1})}\left[Q_F(s_t,a_t)\nabla_{\theta}\log\pi_F(s_t,a_t;\theta)\right]+\mathbb{E}_{P_{F,\mu}(s_{t+1})}\big[\nabla_\theta \underbrace{V_F(s_{t+1};\theta)}_{V_F(s_T)=0} \big].
\end{equation}
Then,
\begin{align}
\overset{(1)}{=}\mathbb{E}_{P_{F,\mu}(\tau)}\left[\sum_{t=0}^{T-1}Q_F(s_t,a_t)\nabla_{\theta}\log\pi_F(s_t,a_t;\theta)\right]=\mathbb{E}_{d_{F,\mu}(s)\pi_F(s,a)}\left[Q_F(s,a)\nabla_{\theta}\log\pi_F(s,a;\theta)\right].
\end{align}
Besides,
\begin{align}
&\overset{(1)}{=}\mathbb{E}_{d_{F,\mu}(s)\pi_F(s,a)}\left[Q_F(s,a)\nabla_{\theta}\log\pi_F(s,a;\theta)\right]-\mathbb{E}_{d_{F,\mu}(s)}\big[V_F(s)\underbrace{\mathbb{E}_{\pi_F(s,a)}[\nabla_{\theta}\log\pi_F(s,a;\theta)}_{=0}]\big]\nonumber\\
    &=\mathbb{E}_{d_{F,\mu}(s)\pi_F(s,a)}\left[A_F(s,a)\nabla_{\theta}\log\pi_F(s,a;\theta)\right].
\end{align}
The derivation of $\nabla_{\phi}J_B(\phi)$ follows the similar way to  $\nabla_{\theta}J_F(\theta)$, and is omitted here.
\end{proof}

\subsection{Connection of Policy-based Training to TB-based Training}\label{relation-TB-RL}
The gradient of the TB objective w.r.t. $\theta_F$ can be written as:
\begin{align}
&\frac{1}{2}\nabla_{\theta_F}\mathcal{L}_{TB}(P_{F,\mu};\theta_F)=\sum_{t=1}^{T}\mathbb{E}_{P_{F,\mu}(\tau)}\left[\nabla_{\theta_{F}}\log P_F(s_{t}|s_{t-1};\theta_F)\left(\sum_{t=1}^{T}\log\frac{P_F(s_t|s_{t-1})}{\widetilde{P}_B(s_{t-1}|s_t)}\right)\right].\label{TB-sum}\end{align}
In equation~\eqref{TB-sum}, each term for $t>1$ can be expanded as:

\begin{align}
&\mathbb{E}_{P_{F,\mu}(\tau_{\geq t-1})}\left[\nabla_{\theta_{F}}\log P_F(s_{t}|s_{t-1};\theta_F)\left(\sum_{l=t}^{T}\log\frac{P_F(s_{l}|s_{l-1})}{\widetilde{P}_B(s_{l-1}|s_{l})}\right)\right]\nonumber\\&+\mathbb{E}_{P_{F,\mu}(\tau_{\leq t})}\left[\nabla_{\theta_{F}}\log P_F(s_{t}|s_{t-1};\theta_F)\left(\sum_{l=1}^{t-1}\log\frac{P_F(s_{l}|s_{l-1})}{\widetilde{P}_B(s_{l-1}|s_{l})}\right)\right]\nonumber\\=&\mathbb{E}_{P_{F,\mu}(\tau_{\geq t-1})}\left[\nabla_{\theta_{F}}\log P_F(s_{t}|s_{t-1};\theta_F)\left(\sum_{l=t}^{T}\log\frac{P_F(s_{l}|s_{l-1})}{\widetilde{P}_B(s_{l-1}|s_{l})}\right)\right]\nonumber\\&+\mathbb{E}_{P_{F,\mu}(\tau_{\leq t-1})}\left[\left(\sum_{l=1}^{t-1}\log\frac{P_F(s_{l}|s_{l-1})}{\widetilde{P}_B(s_{l-1}|s_l)}\right)\underbrace{\mathbb{E}_{P_{F,\mu}(s_t|s_{t-1})}[\nabla_{\theta_{F}}\log P_F(s_{t}|s_{t-1};\theta_F)]}_{=0 \text{ by Lemma \ref{reforce-trick}}}\right].
\end{align}
Thus, 
\begin{align}
    \frac{1}{2}\nabla_{\theta_F}\mathcal{L}_{TB}(P_{F,\mu};\theta_F)&=\sum_{t=1}^{T}\mathbb{E}_{P_{F,\mu}(\tau)}\left[\nabla_{\theta_{F}}\log P_F(s_{t}|s_{t-1};\theta_F)\left(\sum_{l=t}^{T}\log\frac{P_F(s_{l}|s_{l-1})}{\widetilde{P}_B(s_{l-1}|s_{l})}\right)\right]\nonumber\\&-C\sum_{t=0}^{T-1}\underbrace{\mathbb{E}_{P_{F,\mu}(\tau)}\left[\nabla_{\theta_F}\log P_F(s_t,a_t;\theta_F)\right]}_{=0 \text{ by Lemma~\ref{reforce-trick}}}\nonumber\\&=\mathbb{E}_{P_{F,\mu}(\tau)}\left[\sum_{t=1}^{T}\nabla_{\theta_{F}}\log P_F(s_{t}|s_{t-1};\theta_F)\left(\sum_{l=t}^{T}\log\frac{P_F(s_{l}|s_{l-1})}{\widetilde{P}_B(s_{l-1}|s_{l})}-C\right)\right], 
\end{align}
where $C$ is an added baseline and constant w.r.t. $\theta_F$ for variance reduction during gradient estimation. As shown in Appendix \ref{policy_gradient}, the gradient of $J_F$ w.r.t. $\theta_F$ can be written as:
\begin{align}    \nabla_{\theta_F}J_F(\theta_F)&=T\, \mathbb{E}_{d_{F,\mu}(s)\pi_F(s,a)}\left[A_F(s,a)\nabla_{\theta_F} \log\pi_F(s,a;\theta_F)\right]\nonumber\\
    &=\mathbb{E}_{P_{F,\mu}(\tau)}\left[\sum_{t=0}^{T-1}\nabla_{\theta_F}\log\pi_F(s_t,a_t;\theta_F)Q_F(s_t,a_t)\right]\nonumber\\&=\mathbb{E}_{P_{F,\mu}(\tau)}\left[\sum_{t=0}^{T-1}\nabla_{\theta_F}\log\pi_F(s_t,a_t;\theta_F)\mathbb{E}_{P_{F,\mu}(\tau_{>t+1}|s_t,a_t)}\left[\sum_{l=t}^{T-1}R_F(s_{l},a_{l})\bigg|s_t,a_t\right]\right]\nonumber\\&=\mathbb{E}_{P_{F,\mu}(\tau)}\left[\sum_{t=0}^{T-1}\nabla_{\theta_F}\log\pi_F(s_t,a_t;\theta_F)\left(\sum_{l=t}^{T-1}R_F(s_{l},a_{l})\right)\right]\nonumber\nonumber\\&=\mathbb{E}_{P_{F,\mu}(\tau)}\left[\sum_{t=0}^{T-1}\nabla_{\theta_F}\log\pi_F(s_t,a_t;\theta_F)\left(\sum_{l=t}^{T-1}R_F(s_{l},a_{l})-C\right)\right].
\end{align} 
This result implies that: when we update the forward policy by the estimation of $\nabla_{\theta_F}\mathcal{L}_{TB}(\theta_F)$ based on a batch of sampled trajectories,
we approximate $Q_F(s_t,a_t)$ empirically by 
$\widehat{Q}_F(s_t,a_t)=\sum_{l=t}^{T-1}R_F(s_{l}, a_{l})$ for each sample, and can further reduce the estimation variance by some unbiased constant baseline $C$. By comparison, the RL formulation generalizes the constant to an unbiased functional baseline $\widetilde{V}_F(s;\eta)$, which is the approximation of exact $V_F(s)$. This enables to approximate $Q_F(s_t,a_t)$ and $A_F(s_t,a_t)$ functionally by $\widehat{Q}_F(s_t,a_t)=R(s_t,a_t)+\widetilde{V}_F(s_{t+1})$ and $\widehat{A}_F(s_t,a_t)=\widehat{Q}_F(s_t,a_t)-\widetilde{V}_F(s_t)$. Here, $\widehat{A}_F(s_t,a_t)$ can further be generalized to $\sum_{l=t}^{T-1}\lambda^{l-t}\left(\widehat{Q}_F(s_l,a_l)-\widetilde{V}_F(s_l)\right)$, allowing flexible bias-variance trade-off for gradient estimation~(Appendix~\ref{updating-rules}). 

\subsection{Connection between Policy-based Training and Soft-Q-learning}\label{relation-RL-softQ}
In the following text, we discuss the relationship between our policy-based method and soft-Q-learning~\citep{haarnoja2018soft}, one of the most representative Maximum-Entropy~(MaxEnt) RL methods.

Firstly, we introduce their connection when the total flow estimator $\log Z$ is \textbf{fixed}. We can expand $-J_F$ as:
\begin{align}
    -J_F&=T\,\mathbb{E}_{d_{F,\mu}(s),\pi_F(s,a)}\left[\log\pi_B(s',a)-\log\pi_F(s,a)\right]\nonumber\\&=T\,\mathbb{E}_{d_{F,\mu}(s),\pi_F(s,a)}\left[\log \pi_B(s',a))+\mathbb{E}_{\pi_F(s,a)}[-\log\pi_F(s,a)]\right]\nonumber\\&=T\,\mathbb{E}_{d_{F,\mu}(s),\pi_F(s,a)}\left[\log \pi_B(s',a))+\mathcal{H}(\pi_F(s,\cdot))\right],
\end{align}where $a=(s\rightarrow s')$, and $\mathcal{H}$ 
 denotes the entropy of a distribution. The equation above implies that fixing the total flow estimator $\log Z$, maximizing $-J_F$ w.r.t. $\pi_F$ can be interpreted as a MaxEnt RL problem, where $\log\pi_B(s',a)$ is the \textbf{static} reward w.r.t. $\pi_F(s,a)$. We define $Q_F^S(s,a):=\mathbb{E}_{P_F(\tau_{>t+1}|s_t,a_t)}\Big[\pi_B(s_{t+1},a_t)+\sum_{l=t+1}^{T-1}\pi_B(s_{l+1},a_l)+\mathcal{H}(\pi_F(s_l,\cdot))|s_t=s,a_t=a\Big]=-Q_F(s,a)+\log\pi_F(s,a)$ and $V_F^S(s,a):=\mathbb{E}_{P_F(\tau_{>t}|s_t)}\Big[\sum_{l=t}^{T-1}\pi_B(s_{l+1},a_l)+\mathcal{H}(\pi_F(s_l,\cdot))|s_t=s\Big]=-V_F(s,a)$, which implies that 
$Q_F^S(s,a)=\log\pi_B(s^\prime,a)+V_F^S(s^\prime)$.
 We use a parameterized function $\widetilde{Q}_F^S$ as the estimator of $Q_F^S$, and define $\widetilde{V}_F^S(s):=\log\sum_a\exp\{\widetilde{Q}_F^S(s,a)\}$ as the estimator of $V_F^S$. Then, we define $\pi_F(s,a;\theta):=\frac{\exp\widetilde{Q}_F^S(s,a;\theta)}{\exp\widetilde{V}_F^S(s;\theta)}$, which implies that:
 \begin{align} \widetilde{Q}_F^S(s,a)=\widetilde{V}_F^S(s)+\log\pi_F(s,a).\label{soft-Q-eq2}
 \end{align}  
 In soft-Q-learning, we define the target function as:
 \begin{align}
\widehat{Q}_F^S(s,a):=\log\pi_B(s^\prime,a)+\widetilde{V}_F^S(s^\prime).
 \end{align}
Then, $\pi_F$ is updated by the gradient of the following objective:
\begin{align}
    \frac{T}{2}\mathbb{E}_{d_{\mathcal{D}}(s),\pi_{\mathcal{D}}(s,a)}\left[\left(\widetilde{Q}_F^S(s,a;\theta)-\widehat{Q}_F^S(s,a)\right)^2\right].\label{soft-Q-obj}
\end{align}
It has been shown that the policy gradients for static rewards plus the gradients of the policy entropy (or KL divergence from some reference policy) are equivalent to the gradients of the corresponding soft $Q$ estimator~\citep{schulman2017equivalence}. Likewise, we demonstrate our policy-based method with $\lambda=0$ is equivalent to soft-Q-learning with $\pi_\mathcal{D}$ equal to $\pi_F$, without any adaption. Noting soft-Q-learning is off-policy, the gradients of equation~\eqref{soft-Q-obj} w.r.t. $\theta$ can be written as:
\begin{align}
&\frac{T}{2}\nabla_{\theta}\mathbb{E}_{d_{F,\mu}(s)\pi_{F}(s,a)}\left[\left(\widetilde{Q}_F^S(s,a;\theta)-(\log P_B(s',a)+\widetilde{V}_F^S(s'))\right)^2\right]\nonumber\\&=T\mathbb{E}_{d_{F,\mu}(s)\pi_{F}(s,a)}\left[\nabla_{\theta}\widetilde{Q}_F^S(s,a;\theta)\left(\widetilde{Q}_F^S(s,a;\theta)-\log P_B(s',a)-\widetilde{V}_F^S(s')\right)\right]\nonumber\\&=T\mathbb{E}_{d_{F,\mu}(s)\pi_{F}(s,a)}\bigg[\nabla_{\theta}\left(\log\pi_F(s,a)+\widetilde{V}_F^S(s)\right)\left(\widetilde{V}_F^S(s)+R_F(s,a)-\widetilde{V}_F^S(s')\right)\bigg]\nonumber\\&=T\mathbb{E}_{d_{F,\mu}(s)\pi_{F}(s,a)}\bigg[\nabla_{\theta}\left(\log\pi_F(s,a)+\widetilde{V}_F^S(s)\right)\hat{\delta}_F(s,a)\bigg]\nonumber\\&=\mathbb{E}_{P_{F,\mu}(\tau)}\left[\sum_{t=0}^{T-1}\nabla_{\theta}\log\pi_F(s_t,a_t;\theta)\hat{\delta}_F(s_t,a_t)\right]+\mathbb{E}_{P_{F,\mu}(\tau)}\left[\sum_{t=0}^{T-1}\nabla_{\theta}\widetilde{V}_F^S(s_t;\theta)\hat{\delta}_F(s_t,a_t)\right], 
\end{align}
where $\hat{\delta}_F(s,a):=-\widetilde{V}_F(s)+R_F(s,a)+\widetilde{V}_F(s')$, and $\widetilde{V}_F:=-\widetilde{V}_F^S$. Compared to formula~\eqref{pg-update-obj} with $\lambda=0$, a clear equivalence can be established.

We further connect the soft-Q-learning objective~\eqref{soft-Q-obj} to the Flow Matching (FM) objective~\citep{bengio2021flow} and explain the role that $\log Z$ plays during training. 
When $\widetilde{Q}_F^S$ achieves the optimal point, we have $\widetilde{Q}_F^S=\widehat{Q}_F^S$. Consequently, for the desired flow $F$, $\widetilde{V}_F^S(s^f)=0$ by definition, $\widetilde{Q}_F^S(x,a_{T-1})=\log\pi_B(x,a_{T-1})+0=\log\frac{F(x\rightarrow s^f)}{Z}$, and  $\widetilde{V}_F^S(x)=\log\sum_a\exp Q_F^S(x,a)=\log\frac{F(x)}{Z}~(=\log \frac{R(x)}{Z})$. Accordingly, $\widetilde{Q}_F^S(s_{T-2},a_{T-2})=\widetilde{V}_F^S(x)+\log \pi_B(x,a_{T-2})=\log\frac{F(x)}{Z}+\log\frac{F(s_{T-2}\rightarrow x)}{F(x)}=\log\frac{F(s_{T-2}\rightarrow x)}{Z}$, and $\widetilde{V}_F^S(s_{T-2})=\log\sum_a\exp Q_F^S(s_{T-2},a)=\log \frac{F(s_{T-2})}{Z}$. Continuing this process, it can be verified  that  $\widetilde{Q}_F^S(s,a)=\log \frac{F(s\rightarrow s')}{Z}$ and $\widetilde{V}_F^S(s)=\log \frac{F(s)}{Z}$ when $\widetilde{Q}^S$ achieve the optimum. Based on the above optimum condition of $\widetilde{Q}^S(s,a)$ and the fact that $\widetilde{Q}^S(s,a)\in \mathbb{R}$ is a parametrized function with no assumption over its output form during training, we can safely substitute it by $F^{\log}(s\rightarrow s^\prime):=\widetilde{Q}^S(s,a)+\log Z\in\mathbb{R}$, where $F^{\log}(s\rightarrow s^\prime)$ is the estimator for the logarithm of the desired edge flow, $\log F(s\rightarrow s^\prime)$, and no assumption over its output form during training is made as well. Then, objective~\eqref{soft-Q-obj} can be equivalently rewritten as:
\begin{align}
\mathbb{E}_{P_{\mathcal{D}}(\tau)}\left[\sum_{t=1}^{T-1}\left(F^{\log}(s_{t-1}\rightarrow s_t)-\log \left(P_B(s_{t-1}|s_t)\sum_{s_{t+1}}\exp F^{\log}(s_t\rightarrow s_{t+1})\right)\right)^2\right], \label{soft-Q-flow-obj}
\end{align}
This objective is similar to the FM objective, which can be written as:
\begin{align}
    \mathbb{E}_{P_{\mathcal{D}}(\tau)}\left[\sum_{t=1}^{T-1}\left(\log\left(\sum_{s_{t-1}} \exp F^{\log}(s_{t-1}\rightarrow s_t)\right)-\log\left(\sum_{s_t} \exp F^{\log}(s_t\rightarrow s_{t+1})\right)\right)^2\right].~\label{flow-match-obj}
\end{align} The reason is that the optimal solution of the objective~\eqref{soft-Q-flow-obj} satisfies $\exp F^{\log}(s_{t-1}\rightarrow s_t)=P_B(s_{t-1}|s_t)\sum_{s_{t+1}}\exp F^{\log}(s_t\rightarrow s_{t+1})$. Taking summation over $s_{t-1}$ of this equation, we have $\sum_{s_{t-1}}\exp F^{\log}(s_{t-1}\rightarrow s_t)=\sum_{s_{t+1}}\exp F^{\log}(s_t\rightarrow s_{t+1})$, the optimal solution of the objective~\eqref{flow-match-obj}. We can see that $\log Z$ serves as a baseline for modeling $\log F$. Without $\log Z$, we need to approximate $\log F$ by $F^{\log}$ directly. This often leads to numerical issues. For example, let's suppose a small perturbation of $\log F$, denoted as $\epsilon$. Then the flow difference is $\exp(\log F+\epsilon)-F=(\exp \epsilon-1)F$. As values of $F$ can be exponentially large, especially for nodes near the root, the flow difference can be large even if  $\epsilon$ is small, making approximation to $ F$ by $\exp F^{\log}$ very difficult.
By contrast, TB-based methods and our policy-based method allow the updating of $\log Z$ to dynamically scale down the value of $\widetilde{Q}$ during training. This also complements the claims by~\citet{malkin2022trajectory}, who show that the TB-based methods is more efficient than flow-matching and DB-based methods.

\subsection{Model Parameter Updating Rules}\label{updating-rules}

In the following context, we will explain the updating rules for $P_F$ and $\mu$ within the vanilla policy-based method, also called the Actor-Critic method, and the TRPO method. The updating rules for $P_B$ follow those of $P_F$ analogously. 
\paragraph{Actor-Critic} First of all, we split the parameter $\theta$ into $\theta_F$ and $\theta_Z$ corresponding to $\pi_F$ and $Z$. Since computing the exact $V_F$ is usually intractable, we use $\widetilde{V}_F$ parametrized by $\eta$ as the functional approximation.
Given a batch of trajectories samples, we compute the sampling averaging approximation of the following gradient estimators to update $\theta_F,\theta_Z$ and $\eta$ as proposed by~\citet{Schulmanetal_ICLR2016} and ~\citet{tsitsiklis1996analysis}:
\begin{align}
&\mathbb{E}_{P_{F,\mu}(\tau)}\left[\sum_{t=0}^{T-1}\widehat{A}_F^{\lambda}(s_t,a_t)\nabla_{\theta_F}\log \pi_F(s_t,a_t;\theta_F) \right]+\mathbb{E}_{\mu(s_0)}\left[\widehat{V}_F^\lambda(s_0)\nabla_{\theta_Z}\log\mu(s_0;\theta_Z)\right],\nonumber\\
&\mathbb{E}_{P_{F,\mu}(s_t)}\left[\sum_{t=0}^{T-1}\nabla_{\eta}(\widehat{V}_F^\lambda(s_t)-\widetilde{V}_F(s_t;\eta))^2\right],\label{pg-update-obj}
\end{align}
where $\lambda\in [0,1]$, 
\begin{align}
&\widehat{A}_F^\lambda(s_t,a_t):=\sum_{l=t}^{T-1}\lambda^{l-t}\hat{\delta}_F(s_{l},a_{l}),
\quad\widehat{V}_F^\lambda(s_t):=\sum_{l=t}^{T-1}\lambda^{l-t}\hat{\delta}_F(s_{l},a_{l})+\widetilde{V}_F(s_t),\nonumber\\
&\hat{\delta}_F(s_t,a_t):=R_F(s_t,a_t)+\widetilde{V}_F(s_{t+1})-\widetilde{V}_F(s_t),
\end{align}$
\widehat{A}_F$ is called \textbf{critic} and $\pi_F$ is called \textbf{actor}. It can be verified that $\widehat{A}_F^1(s_t,a_t)=\sum_{l=t}^{T-1}R_F(s_{l},a_{l})-\widetilde{V}_F(s_t;\eta)$ renders an unbiased estimator of $\nabla_{\theta_F}J(\theta)$ as the first term is an unbiased estimation of $Q_F$ and  $\widetilde{V}_F(\cdot;\eta)$ does not introduce estimation bias (Remark~\ref{unbias-V}); 
$\widehat{A}_F^0(s_t,a_t)=R(s_t,a_t)+\widetilde{V}_F(s_{t+1})-\widetilde{V}_F(s_t)$ provided an direct functional approximation of $A_F(s_t,a_t)$, which usually render biased estimation with lower variance as $\widetilde{V}_F$ may not equal to $V_F$ exactly. Thus, $\lambda$ enables the \textbf{variance-bias trade-off} for robust gradient estimation. Likewise, $\widehat{V}_F^1(s_t)=\sum_{l=t}^{T-1}R(s_l,a_l)$ and $\widehat{V}_F^0(s_t)=R(s_t,a_t)+\widetilde{V}_F(s_{t+1})$ for each $\tau$, corresponding to unbiased and biased estimation of $V_F(s_t)$. 
Denoting the estimated gradient w.r.t. $(\theta_F$, $\theta_Z$, $\eta)$ as $(\hat{g}_F$, $\hat{g}_Z$, $\hat{g}_V)$, these parameters are updated by $(\theta_F^\prime, \theta_Z^\prime,\eta^\prime )\gets (\theta_F-\alpha_F\hat{g}_F, \theta_Z-\alpha_Z \hat{g}_Z, \eta-\alpha_V \hat{g}_V$).
\paragraph{TRPO} Parameters $\theta_Z$ and $\eta$ are updated in the same way as the actor-critic method. Parameter $\theta_F$ is updated by the linear approximation of objective~(\ref{trpo_obj}): 
\begin{align}
\min_{\theta^\prime_F}
\quad& T\,g_F^\top\big(\theta^\prime_F-\theta_F\big)\nonumber\\
\textrm{s.t.} \quad &\frac{1}{2}(\theta^\prime_F-\theta_F)^\top H_F(\theta^\prime_F-\theta_F)\leq \zeta_F,
\end{align}
with
\begin{align}
g_F=\nabla_{\theta^\prime_F}\mathbb{E}_{d_{F,\mu}(s;\theta),\pi_F(s,a;\theta^\prime_F)}\left[\widehat{A}_F^\lambda(s,a;\theta_F) \right],\quad H_F=\nabla^2_{\theta_F^\prime}D_{KL}^{d_{F,\mu}(\cdot;\theta)}\left(\pi_F(s,a;\theta_F),\pi_F(s,a;\theta^\prime_F)\right).
\end{align}
Let's denote the Lagrangian formulation of the above problem as $L(\delta,\kappa):=Tg_F^\top \delta-\kappa (\delta^\top H_F \delta-\zeta_F)$ with Lagrangian constant $\kappa$ and $\delta:=\theta_F^\prime-\theta_F$. By the optimal conditions of $L(\delta,\kappa)$, $\nabla_\kappa L(\delta,\kappa)=0$ and $\nabla_{\delta}L(\delta,\kappa)=0$, we have $\delta=\frac{1}{\kappa}H_F^{-1}g_F$ and $\kappa=\left(\frac{g_F^\top H_F^{-1}g_F}{2\zeta_F}\right)^{0.5}$. Thus, 
the maximal updating step of model parameters is: 
$\theta^\prime_F \gets \theta_F-\left(\frac{2\zeta_F}{\hat{g}_F^\top \widehat{H}_F^{-1}\hat{g}_F}\right)^{0.5}\widehat{H}_F^{-1}\hat{g}_F$. When the dimension of $\theta_F^\prime$ is high, computing  $\widehat{H}_F^{-1}$ is time-demanding. Thus, we adopt the conjugate gradient method to estimate $\widehat{H}_F^{-1}\hat{g}_F$ based on $\widehat{H}_F\hat{g}_F$~\citep{hestenes1952methods}. Besides, following~\citet{schulman2015trust}, we perform a line search of updating step size to improve performance, instead of taking the maximal step.


\section{Performance Analysis}
\begin{lemma}\label{descent}(Descent lemma~\citep{beck2017first}) Supposing $f(\cdot)$ is a $\beta$-smooth function, then for any $\theta$ and $\theta^\prime$:
\begin{align}
    f(\theta^\prime)\leq f(\theta)+\langle\nabla_{\theta}f(\theta), \theta^\prime-\theta\rangle+\frac{\beta}{2}\left\|\theta-\theta^\prime\right\|_2^2.
\end{align}
\end{lemma}
\begin{lemma}\label{pdl} Given two forward policies $(\pi_F^{\prime}, \pi_F)$ or two backward policies $(\pi_B^{\prime}, \pi_B)$, we have 
\begin{equation}
\begin{split}
\frac{1}{T}(J_F^\prime-J_F)&\leq \mathbb{E}_{d_{F,\mu}^\prime(s), \pi_F^\prime(s,a)}[A_F(s,a)]+D_{KL}^{d_{F,\mu}^{\prime}}(\pi_F^{\prime}(s,\cdot),\pi_F(s,\cdot)),\\
\frac{1}{T-1}(J_B^\prime-J_B)&\leq \mathbb{E}_{d_{B,\rho}^\prime(s), \pi_B^\prime(s,a)}[A_B(s,a)]+D_{KL}^{d_{B,\rho}^{\prime}}(\pi_B^{\prime}(s,\cdot),\pi_B(s,\cdot)).
\end{split}
\end{equation}
\end{lemma}
\begin{proof}
The proof procedure is analogous to that of~\citet{schulman2015trust} and~\citet{rengarajan2022reinforcement}. By the definition of $A_F$, 
\begin{align}
\mathbb{E}_{P_F^\prime(\tau|s_0)}\left[\sum_{t=0}^{T-1}{A_F(s_t,a_t)}\right]&=\mathbb{E}_{P_F^\prime(\tau|s_0)}\left[\sum_{t=0}^{T-1}\big(R_F(s_t,a_t)+V_F(s_{t+1})-V_F(s_t)\big)\right]\nonumber\\&=\mathbb{E}_{P_F^\prime(\tau|s_0)}\left[\sum_{t=0}^{T-1}{R_F(s_t,a_t)}\right]+\mathbb{E}_{P_F^\prime(\tau|s_0)}[\underbrace{V_F(s_T)}_{=0}] -V_F(s_0)\nonumber\\&=\mathbb{E}_{P_F^\prime(\tau|s_0)}\left[\sum_{t=0}^{T-1}{(R_F^\prime(s_t,a_t)+R_F(s_t,a_t)-R_F^\prime(s_t,a_t))}\right] -V_F(s_0)\nonumber\\&=V_F^\prime(s_0)-V_F(s_0)+\mathbb{E}_{P_F^\prime(\tau|s_0)}\left[\sum_{t=0}^{T-1}(R_F(s_t,a_t)-R_F^\prime(s_t,a_t))\right].
\end{align}
Thus,\begin{align}
J_F^\prime-J_F&=\mathbb{E}_{P_{F,\mu}^\prime(\tau)}\left[\sum_{t=0}^{T-1}{A_F(s_t,a_t)}\right]+\mathbb{E}_{P_{F,\mu}^\prime(\tau)}\left[\sum_{t=0}^{T-1}D_{KL}(\pi_F^\prime(s_t,\cdot),\pi_F(s_t,\cdot))\right]\nonumber\\&=T\,\mathbb{E}_{d_{F,\mu}^\prime(s),\pi_F^\prime(s,a)}[A_F(s,a)]+T\,D_{KL}^{d_{F,\mu}^\prime}(\pi_F^\prime(s,\cdot),\pi_F(s,\cdot))
\end{align}
Using the fact that $V_B^\prime(s_0)=0$ and backward rewards are accumulated from $T-1$ back to $1$, the results for the backward case can be derived by a similar procedure as in the forward case, so it is omitted here.
\end{proof}

\subsection{Proof of Theorem \ref{G-F-B}}\label{proof-GFB}
\begin{proof}
Firstly,
\begin{align}
    J_F^G&=J_F+(J_F^G-J_F)\nonumber\\&=J_F+\mathbb{E}_{P_{F,\mu}(\tau)}\left[\log\frac{P_F(\tau|s_0)Z}{P_G(\tau|x)R(x)}-\log\frac{P_F(\tau|s_0)Z}{P_B(\tau|x)R(x)}\right]\nonumber\\&=J_F+\mathbb{E}_{P_{F,\mu}(\tau)}\left[\log\frac{P_B(\tau|x)}{P_G(\tau|x)}\right]+\mathbb{E}_{P_{B,\rho}(\tau)}\left[\log\frac{P_B(\tau|x)}{P_G(\tau|x)}\right]-\mathbb{E}_{P_{B,\rho}(\tau)}\left[\log\frac{P_B(\tau|x)}{P_G(\tau|x)}\right]\nonumber\\&=
    J_F+J_B^G
    +\sum_{\tau}\left(P_{F,\mu}(\tau)-P_{B,\rho}(\tau)\right)R_B^G(\tau),  
\end{align}
where $R_B^G(\tau):=\log\frac{P_B(\tau|x)}{P_G(\tau|x)}=\sum_{t=1}^{T-1}R_B^G(s_t,a_t)$. Then,
\begin{align}
    J_F^G&=J_F+J_B^G+\left\langle P_{F,\mu}(\cdot)-P_{B,\rho}(\cdot),R_B^G(\cdot)  \right\rangle\nonumber\\
    &\leq J_F+J_B^G+\left\| P_{F,\mu}(\cdot)-P_{B,\rho}(\cdot)\right\|_1\left\|R_B^G(\cdot)\right\|_{\infty} \nonumber\\    &\leq J_F+J_B^G+(T-1)\,\left\| P_{F,\mu}(\cdot)-P_{B,\rho}(\cdot)\right\|_1 R_B^{G,\max},  
\end{align}
where the first inequality holds by Hölder's inequality, and the second inequality holds by $ R_B^{G,\max}:=\max_{s,a}\big|R_B^G(s,a)\big|\geq \frac{1}{T-1}\max_{\tau}\big|R_B^G(\tau)\big|$.
By Pinsker's inequality:
\begin{equation}
\left\|P_{F,\mu}(\cdot)-P_{B,\rho}(\cdot)\right\|_1\le \sqrt{\frac{1}{2}D_{KL}(P_{F,\mu}(\tau),P_{B,\rho}(\tau))}.
\end{equation}
Besides, 
\begin{align}
        D_{KL}(P_{F,\mu}(\tau),P_{B,\rho}(\tau)) &=\mathbb{E}_{P_{F,\mu}(\tau)}\left[\log\frac{P_{F,\mu}(\tau|x)P_{F,\mu}^\top(x)}{P_B(\tau|x)P_{F,\mu}^\top(x)}\right]   \nonumber \\&\leq\mathbb{E}_{P_{F,\mu}(\tau)}\left[\log\frac{P_{F,\mu}(\tau|x)}{P_B(\tau|x)}\right]+\underbrace{\mathbb{E}_{P_{F,\mu}^\top(x)}\left[\log\frac{P_{F,\mu}^\top(x)}{R(x)/Z^*}\right]}_{\geq 0}\nonumber\\&=D_{KL}(P_{F,\mu}(\tau),P_B(\tau))\nonumber\\&= D_{KL}^\mu(P_F(\tau|s_0),P_B(\tau|s_0))+\underbrace{D_{KL}(\mu(s_0),P_B(s_0))}_{=0}\nonumber\\&=D_{KL}^{\mu}(P_F(\tau|s_0),\widetilde{P}_B(\tau|s_0))-\log Z+\log Z^*\nonumber\\&= J_F+\log Z^* -\log Z.
\end{align}
Then, we have:
\begin{equation}
    J_F^G\le J_F+J_B^G+(T-1)\, R_B^{G,\max}\sqrt{\frac{1}{2}(J_F+\log Z^\ast-\log Z)}.
\end{equation}
\end{proof}

\subsection{Proof of Theorem \ref{trpo}}\label{proof-trpo}
\begin{proof}
By Lemma \ref{pdl} and the definition of $\zeta_F$:
\begin{equation}
      \frac{1}{T} (J_F^\prime-J_F)\leq\mathbb{E}_{ d_{F,\mu}^\prime(s),\pi_F^\prime(s,a)}[A_F(s,a)]+\zeta_F.\label{proof-trpo-part1}
\end{equation}
Let $\bar{A}_F\in R^{|\mathcal{S}|}$ denote the vector components of $\mathbb{E}_{\pi_F^\prime(s,a)}[A_F(s,a)]$. Then, we have:
\begin{align}
\mathbb{E}_{d_{F,\mu}^\prime(s)\pi_F^\prime(s,a)}[A_F(s,a)]&=\left\langle d_{F,\mu}^\prime,\bar{A}_F\right\rangle\nonumber\\&=\left\langle d_{F,\mu},\bar{A}_F\right\rangle+\left\langle d_{F,\mu}^\prime-d_{F,\mu},\bar{A}_F\right\rangle\nonumber\\
&\leq  \mathbb{E}_{d_{F,\mu}(s)\pi_F^\prime(s,a)}[A_F(s,a)]+ \left\|d_{F,\mu}^\prime-d_{F,\mu}\right\|_1\left\|\bar{A}_F\right\|_{\infty},
\end{align}
where the last inequality holds by Hölder's inequality.
By Lemma \ref{d_dist} and the definition of $\epsilon_F$:
\begin{align}
     \mathbb{E}_{d_{F,\mu}^\prime(s)\pi_F^\prime(s,a)}[A_F(s,a)]&\leq   \mathbb{E}_{d_{F,\mu}(s)\pi_F^\prime(s,a)}[A_F(s,a)]+2\mathbb{E}_{d_{F,\mu}^\prime(s)}\big[D_{TV}(\pi_F^\prime(s,\cdot),\pi_F(s,\cdot))\big]\epsilon_F.
\end{align}
By Pinsker's inequality, 
\begin{equation}
    D_{TV}(\pi_F^\prime(s,\cdot),\pi_F(s,\cdot))\leq \left(\frac{1}{2}D_{KL}(\pi_F^\prime(s,\cdot),\pi_F(s,\cdot)\right)^{0.5}.\nonumber
\end{equation} 
By Jensen's inequality and the definition of $\zeta_F$, 
\begin{equation}
    \mathbb{E}_{d_{F,\mu}^\prime(s)}\left[\left(\frac{1}{2}D_{KL}(\pi_F^\prime(s,\cdot),\pi_F(s,\cdot)\right)^{0.5}\right]\leq  \left(\frac{1}{2}\mathbb{E}_{d_{F,\mu}^\prime(s)}\left[D_{KL}(\pi_F^\prime(s,\cdot),\pi_F(s,\cdot)\right]\right)^{0.5}\leq \left(\frac{\zeta_F}{2}\right)^{0.5}.\nonumber
\end{equation}
Thus, we have:
\begin{equation}
    \mathbb{E}_{d_{F,\mu}^\prime(s)\pi_F^\prime(s,a)}[A_F(s,a)]\leq   \mathbb{E}_{d_{F,\mu}(s)\pi_F^\prime(s,a)}[A_F(s,a)]+(2\zeta_F)^{0.5}\epsilon_F.\label{proof-trpo-part2}
\end{equation}
Combing inequalities \eqref{proof-trpo-part2} and \eqref{proof-trpo-part1}, we have:
\begin{equation}
     \frac{1}{T}(J_F^\prime-J_F) \le E_{d_{F,\mu}(s)\pi_F^\prime(s,a)}[A_F(s,a)]+\zeta_F+(2\zeta_F)^{0.5}\epsilon_F.
\end{equation}
\end{proof}
\subsection{Proof of Theorem \ref{pg-stable}}\label{proof-pg-stable}
\begin{proof}
By Lemma~\ref{descent}, 
\begin{align}
     J_F(\theta_{n+1})\leq J_F(\theta_n)+ \left\langle\nabla_{\theta_n}J_F(\theta_n),\theta_{n+1}-\theta_n\right\rangle+ \frac{\beta}{2}\left\|\theta_{n+1}-\theta_n\right\|^2_2. 
\end{align}
Thus, 
\begin{align}
     -\left\langle\nabla_{\theta_n}J_F(\theta_n),\theta_{n+1}-\theta_n\right\rangle&\leq J_F(\theta_n)-J_F(\theta_{n+1})+\frac{\beta}{2}\left\|\theta_{n+1}-\theta_n\right\|^2_2,\nonumber\\\alpha\big\langle\nabla_{\theta_n}J_F(\theta_n),\widehat{\nabla}_{\theta_n}J_F(\theta_n)\big\rangle&\leq J_F(\theta_n)-J_F(\theta_{n+1})+\frac{\beta\alpha^2}{2}\big\|\widehat{\nabla}_{\theta_n}J_F(\theta_n)\big\|^2_2. \nonumber
\end{align}
Conditioning on $\theta_n$, taking expectations over both sides and noting that $\mathbb{E}_{P(\cdot|\theta_n)}\left[\big\langle\nabla_{\theta_n}J_F(\theta_n),\widehat{\nabla}_{\theta_n}J_F(\theta_n)\big\rangle\right]=\big\langle\nabla_{\theta_n}J_F(\theta_n),\mathbb{E}_{P(\cdot|\theta_n)}[\widehat{\nabla}_{\theta_n}J_F(\theta_n)]\big\rangle=\left\|\nabla_{\theta_n}J_F(\theta_n)\right\|^2_2$, we have:
\begin{align} 
\alpha\left\|\nabla_{\theta_n}J_F(\theta_n)\right\|^2_2 &\leq J_F(\theta_n)-\mathbb{E}_{P(\theta_{n+1}|\theta_n)}\left[J_F(\theta_{n+1})\right]+\frac{\beta\alpha^2}{2}\mathbb{E}_{P(\cdot|\theta_n)}\left[\big\|\widehat{\nabla}_{\theta_n}J_F(\theta_n)\big\|^2_2\right]. 
\end{align}
By the assumption that $\mathbb{E}_{P(\cdot|\theta)}\left[\big\|\widehat{\nabla}_{\theta} J_F(\theta) -\nabla_{\theta}J_F(\theta)\big\|_2^2\right]=\mathbb{E}_{P(\cdot|\theta)}\left[ \big\|\widehat{\nabla}_{\theta} J_F(\theta)\big\|_2^2\right]-\left\|\nabla_{\theta}J_F(\theta)\right\|_2^2\leq \sigma_F$, we have:
\begin{align} 
\alpha\left\|\nabla_{\theta_n}J_F(\theta_n)\right\|^2_2&\leq J_F(\theta_n)-\mathbb{E}_{P(\theta_{n+1}|\theta_n)}\left[J_F(\theta_{n+1})\right]+\frac{\beta\alpha^2}{2}\left\|\nabla_{\theta_n}J_F(\theta_n)\right\|_2^2+\frac{\beta\alpha^2\sigma_F}{2}.
\end{align}
Consequently, we have:
\begin{align}
\left(\alpha-\frac{\beta\alpha^2}{2}\right)\mathbb{E}_{P(\theta_{0:N-1})} \left[\sum_{n=0}^{N-1}\left\|\nabla_{\theta_n}J_F(\theta_n)\right\|_2^2\right]&\leq \frac{N\beta\alpha^2\sigma_F}{2}+\mathbb{E}_{P(\theta_{0:N})} \left[\sum_{n=0}^{N-1} J_F(\theta_n)-J_F(\theta_{n+1})\right],\nonumber\\
\left(\alpha-\frac{\beta\alpha^2}{2}\right)\sum_{n=0}^{N-1}\mathbb{E}_{P(\theta_{n})} \left[\left\|\nabla_{\theta_n}J_F(\theta_n)\right\|_2^2\right]&\leq \frac{N\beta\alpha^2\sigma_F}{2}+\mathbb{E}_{P(\theta_{0:N})} \left[J_F(\theta_0)-J_F(\theta_N)\right],\nonumber\\
  \left(\alpha-\frac{\beta\alpha^2}{2}\right)N\min_{n\in\{0,\ldots,N-1\}}\mathbb{E}_{P(\theta_{n})} \left[\left\|\nabla_{\theta_n}J_F(\theta_n)\right\|_2^2\right] &\leq \frac{N\beta\alpha^2\sigma_F}{2}+ \mathbb{E}_{P(\theta_0)}\left[ J_F(\theta_0)\right]-\mathbb{E}_{P(\theta_N)}\left[J_F(\theta_N)\right].
\end{align}
Setting $\alpha=\sqrt{2/(\beta N)}$, we have:
\begin{align}
\left(\sqrt{(2N)/\beta}-1\right)\min_{n\in\{0,\ldots,N-1\}}\mathbb{E}_{P(\theta_{n})} \left[\left\|\nabla_{\theta_n}J_F(\theta_n)\right\|^2_2\right] &\leq \sigma_F+\mathbb{E}_{P(\theta_0)}\left[ J_F(\theta_0)\right]-\mathbb{E}_{P(\theta_N)}\left[J_F(\theta_N)\right].
\end{align}
Since $J_F(\theta)+\log Z^*-\log Z(\theta)=D_{KL}^{\mu(\cdot;\theta)}(P_F(\tau|s_0;\theta),P_B(\tau|s_0))$, and $J_F(\theta^*)=0$ with $\log Z^*=\log Z(\theta^*)$ for optimal parameter $\theta^*$, then $J_F(\theta_{N})+\log Z^*-\log Z(\theta_{N})\geq J_F(\theta^*)$ and we have:
\begin{align}
\min_{n\in \{0,\ldots,N-1\}}\mathbb{E}_{P(\theta_{n})} \left[\left\|\nabla_{\theta_n}J_F(\theta_n)\right\|_2^2 \right]&\leq \frac{\sigma_F+\mathbb{E}_{P(\theta_0)}\left[ J_F(\theta_0)\right]-\mathbb{E}_{P(\theta_N)}\left[J_F(\theta_N)+\log Z^*-\log Z(\theta_{N})\right]}{\left(\sqrt{(2N)/\beta}-1\right)}\nonumber\\&\quad+\frac{\mathbb{E}_{P(\theta_N)}\left[\log Z^*-\log Z(\theta_{N})\right]}{\left(\sqrt{(2N)/\beta}-1\right)}\nonumber\\&\leq \frac{\sigma_F+\mathbb{E}_{P(\theta_0)}\left[ J_F(\theta_0)\right]+\mathbb{E}_{P(\theta_N)}\left[\left|\log Z^*-\log Z(\theta_{N})\right|\right]}{\left(\sqrt{(2N)/\beta}-1\right)}.
\end{align}
By the assumption that $\left|\log Z-\log Z^*\right|\leq \sigma_Z$, we have:
\begin{align}
\min_{n\in \{0,\ldots,N-1\}}\mathbb{E}_{P(\theta_{n})} \left[\left\|\nabla_{\theta_n}J_F(\theta_n)\right\|_2^2\right] &\leq \frac{\sigma_F+\sigma_Z+\mathbb{E}_{P(\theta_0)}[J_F(\theta_0)]}{\left(\sqrt{(2N)/\beta}-1\right)}.
\end{align}
\end{proof}

\section{Additional Discussion about Policy-based and Value-based Methods}\label{policy-value-diff}

The goal of traditional RL is to learn a policy $\pi$ that achieves the optimality in the expected accumulated reward $J_{\pi}$ (for GFlowNet training, corresponding to the distance between 
$P_F(\tau)$ and $P_B(\tau)$, $\mathrm{Dist}(P_F(\tau),P_B(\tau))$) addressing the challenge of the exploration-exploitation (Exp-Exp) dilemma. While valued-based methods are usually off-policy allowing to explicitly balance the \textbf{Exp-Exp trade-off} by designing $P_{\mathcal{D}}$, the objectives of the valued-based methods are optimized to encourage the improvement of $J_{\pi}$ but they do not directly solve the optimization formulation with $J_{\pi}$. Policy-based methods directly optimize $J_{\pi}$ w.r.t. $\pi$, enabling optimization techniques that tackle the \textbf{Exp-Exp trade-off} implicitly but efficiently. Our joint framework manages to inherit both advantages of the value-based and the policy-based methods by keeping the optimization formulation of $J_{\pi}$ and allowing explicit design of $P_G$ as $P_{\mathcal{D}}$. We provide more detailed explanations of our arguments as follows:
\begin{itemize}
    \item The \textbf{Exp-Exp dilemma} is the main challenge in decision-making including different reinforcement learning (RL) formulations. RL is guided by reward functions. To learn the desired policies, a reinforcement learning agent must prefer actions that it has tried in the past and found to be effective in producing rewards (exploitation). But to discover such actions, it has to try actions that it has not been selected before (exploration) at the expense of an exploitation opportunity~\citep{sutton2018reinforcement}. Therefore, both policy-based and value-based methods face the fundamental challenge and try to overcome them in different ways. 
    
    \item In RL, the goal is to learn a policy $\pi$ that achieves the optimality in the expected accumulated reward $J_{\pi}$. The value-based methods, represented by Q-learning and soft-Q-learning, do not optimize $J_{\pi}$ w.r.t. $\pi$ directly. They leverage the fact that the optimal policy should satisfy the Bellman equation. By minimizing the mismatch of the Bellman equation, which typically takes an off-policy form $\mathbb{E}_{(s,a)\sim P_{\mathcal{D}}}[(Q_\pi(s,a)-\widehat{Q}_{\pi}(s,a))^2]$~\citep{haarnoja2018soft}, the corresponding $J_{\pi}$ is encouraged to be improved. The improvement, however, is not guaranteed, since they do not directly solve the optimization formulation with $J_\pi$. So the core of the value-based methods turns to explicitly design a sampler (agent), $P_{\mathcal{D}}$, that can effectively identify the state-action pair that gives rise to the mismatch the most (exploration) while allowing revisiting the state-action pair that has already been found to be effective (exploitation), and how to represent the target function $\widehat{Q}_{\pi}$ that approximates the optimal function $Q^\ast$ while balancing the trade-off properly as well. To exemplify that both exploration and exploitation are important, let's take the hyper-grid experiment as an example, where modes are highly separated but do not exactly lie on the margin of the grids. In the extreme case, a purely explorative sampler will always favor taking actions that lead to visiting the marginal coordinates, which, however, yield low rewards.
    \item Likewise, original GFlowNet methods do not optimize the distance between $P_F(\tau)$ and $P_B(\tau)$ directly (so that $P_F^\top(x)=P_B^\top(x)$), but optimizes the flow mismatch associated with, $P_F$, which implicitly encourages the minimization of the distance and the core of training efficiency is to design a sampler $P_{\mathcal{D}}$ that effectively balance the \textbf{Exp-Exp trade-off}. This motivates back-and-forth local search~\citep{kim2023local}, Thompson sampling~\citep{rector2023thompson} and  temperature conditioning~\citep{kim2023learning}. Besides, Detailed Balance (DB) objective can be understood as favoring exploitation as the `target' edge flow is $P_B(s|s^\prime)F(s^\prime)$, where $F(s^\prime)$ is learned from the data collected so far and represents our partial knowledge about the environment. The Trajectory Balance (TB) objective can be understood as favoring exploration as the target trajectory flow is $P_B(\tau|x)R(x)$ which can be fixed w.r.t. $P_F$ and regarded as pure environment feedback. So, finding better `target' flow representations that properly balance the \textbf{Exp-Exp trade-off} is one of the common motivations for sub-trajectory balance~\citep{madan2023learning}, forward-looking~\citep{pan2023better}, and energy decomposition~\citep{kim2023learning} approaches.
    \item  Policy-based methods reformulate the problem of balancing the Bellman equation into optimizing $J_{\pi}$ w.r.t. $\pi$ directly. On one hand, the nature of the policy-based methods requires to be on-policy (i.e. $P_{\mathcal{D}}=P_\pi$), so we can not explicitly design $P_{\mathcal{D}}$ to overcome the exploration-exploitation dilemma. On the other hand, it saves us from the difficulty in sampler design. The policy-based methods compute the gradients of $J_{\pi}$ w.r.t. $\pi$ and learn $\pi$ by gradient-based strategies. So the problem of designing a sampler is converted to gradient-based optimization with robust gradient estimation. Related techniques include variance reduction techniques, improvement of gradient descent directions like natural policy gradients and mirror policy descent, and conservative policy updates such as TRPO and PPO. These methods implicitly address \textbf{Exp-Exp trade-off}. Because a policy $\pi$ that always favors either exploration or exploitation will not render the maximum or minimum $J_{\pi}$ unless it is equal to the optimal policy. Besides, the intuition of conservative policy updates like TRPO is that we keep the policy unchanged to prevent it from getting trapped into local optima (exploration) unless we find a better point in the trust region (exploitation).
    \item Our joint training framework manages to inherit the advantages of both value-based methods and policy-based methods. It keeps optimizing $J_\pi$ directly and makes the associated gradient-based optimization techniques applicable. In the meanwhile, we can explicitly design guided policy as the design of $P_{\mathcal{D}}$ in the off-policy case, to integrate expert knowledge about the environment and balance the \textbf{Exp-Exp trade-off} explicitly.
\end{itemize}

\section{Additional Experimental Settings and Results}\label{experment-detail}
In all experiments, we follow a regular way of designing off-policy sampler,$P_\mathcal{D}$ for value-based methods: $P_\mathcal{D}$ is a mixture of the learned forward policy and a uniform policy where the mix-in factor of the uniform policy starts at $1.0$ and decays exponentially at a rate of $\gamma$ after each training iteration, where $\gamma$ is set to $0.99$ based on the results of the ablation study. In TB-Sub, the objective is a convex combination of the sub-trajectory balance losses following~\citet{madan2023learning}, where the hyperparameter that controls the weights assigned to sub-trajectories of different
lengths, is set to $0.9$ selected from $\{0.80,0.85,0.90,0.95,0.99\}$. For our policy-based methods, we set the value of hyper-parameter $\lambda$ to $0.99$ based on the results of the ablation study. The gradients of the total flow estimator $Z$ and the backward value estimator $\widetilde{V}_B$ that approximates $V_B$ are estimated unbiasedly, which corresponds to setting $\lambda$ specifically to 1. Trust region hyper-parameter $\zeta_F$ is set to $0.01$ selected from $\{0.01,0.02,0.03,0.04,0.05\}$.

We use the Adam optimizer for model optimization. The learning rates of forward and backward policy are equal to $1\times10^{-3}$, which is selected from $\{5\times 10^{-3},1\times10^{-3},5\times 10^{-4},1\times 10^{-4}\}$ by TB-U. The learning rates of value functions are set to $5\times10^{-3}$, which is selected from $\{1\times 10^{-2},5\times 10^{-3},1\times10^{-3}\}$ by RL-U.  The learning rates of total flow estimator is $1\times10^{-1}$, which is selected from $\{1\times 10^{-1},5\times 10^{-2},1\times 10^{-2},5\times10^{-3}\}$ by TB-U. The sample batch size is set to 128 for each optimization iteration.  For all experiments, we report the performance with five different random seeds.

\subsection{Evaluation Metrics}\label{metric-detail}
The total variation $D_{TV}$ between $P_F^\top(x)$ and $P^\ast(x)$ is defined as:
\begin{align}
    D_{TV}(P_F^\top, P^\ast)=\frac{1}{2}\sum_{x\in \mathcal{X}} |P^\top_F(x)-P^\ast(x)|. 
\end{align}
The total variation is similar to the average $l_1$-distance used in prior works, which can be computed by $\frac{1}{|\mathcal{X}|}\sum_x |P^\ast(x)-P^\top_F(x)|$. However, the average $l_1$-distance may be inappropriate as  $|\mathcal{X}|$ is usually large ($>10^4$) and $\sum_x |P^\ast(x)-P^\top_F(x)|\leq 2$,  resulting in the average $l_1$-distance being heavily scaled down by $|\mathcal{X}|$. 

The Jensen–Shannon divergence $D_{JSD}$ between $P_F^\top(x)$ and $P^\ast(x)$ is defined as:
\begin{align}
    D_{JSD}(P_F^\top,P^\ast)=\frac{1}{2}D_{KL}(P_F^\top,P^M)+ \frac{1}{2}D_{KL}(P^\ast,P^M),\quad P^M=P_F^\top+P^\ast.
\end{align}

Following~\citet{shen2023towards}, the mode accuracy $Acc$ of $P_F^\top(x)$ w.r.t. $P^\ast(x)$ is defined as:
\begin{align}
Acc(P_F^\top,P^\ast)=\min\left(\frac{\mathbb{E}_{P_F^\top(x)}[R(x)]}{\mathbb{E}_{P^\ast(x)}[R(x)]}, 1\right). 
\end{align}

For biological and molecular sequence experiments, we also count the number of modes, that is, the number of modes discovered during training. At every 10 training iterations, we sample  $|\mathcal{X}^{\mathrm{mode}}|$ terminating states by the current learned $P_F(\cdot|\cdot)$ and store the states which are modes and have not been discovered before; then we count the total number of discovered unique modes for evaluation.~\citep{shen2023towards,kim2023local}. The mode set $\mathcal{X}^{\mathrm{mode}}$ is defined as the set of terminating states whose rewards are in the top 0.5\%, 0.5\%, 0.5\%, and 0.1\%  of all rewards for the SIX6, QM9, PHO4 and sEH datasets respectively.

\subsection{Hyper-grid Modeling}\label{Hyper-grid_add}

\paragraph{Environment} In this environment, $\mathcal{S}\setminus\{s^f\}$ is equal to 
$\left\{s=([s]_0,\dots,[s]_d,\ldots,[s]_D)|[s]_d\in \{0,\ldots, N-1\}\right\}\,(=\{1,\ldots, N-1\}^D)$, where the initial state $s^0=(0,\ldots,0)$, and the final state $s^f$ can be represented by any invalid coordinate tuple of the hyper-grid, denoted as $(-1,\ldots,-1)$ in our implementation. For state $s\in \mathcal{S}\setminus\{s^f\}$, we have $D+1$ possible actions in $\mathcal{A}(s)$: (1) increment the coordinate $[s]_d$ by one, arriving at $s'=([s]_0,\ldots,[s]_d+1,\ldots)$; (2) choose stopping actions $(s{\rightarrow}s^f)$, terminating the process and returning $s$ as the terminating coordinate tuple $x$. In this environment, $\mathcal{G}$ is not a graded DAG, and $\mathcal{S}\setminus\{s^f\}=\mathcal{X}$ as all coordinate tuples can be returned as the terminating states. The reward $R(x)$ is defined as:
\begin{align}
R(x)=R_0+R_1\prod_{d=1}^D\mathbb{I}\left[\left|\frac{[s]_d}{N-1}-0.5\right|\in (0.25,0.5]\right]+R_2\prod_{d=1}^D\mathbb{I}\left[\left|\frac{[s]_d}{N-1}-0.5\right|\in (0.3,0.4]\right],
\end{align}
where $R_0=10^{-2}$, $R_1=0.5$ and $R_2=2$ in our experiment. Conditioning on $x$, we use an unnormalized conditional guided trajectory distribution $\widetilde{P}_G(\tau|x)$ for backward policy design, which is defined as:
\begin{align}
    &\widetilde{P}_G(\tau|x\rightarrow s^f):=P_f(\tau_{\preceq x})= \prod_{t=1}^{T-1} P_f(s_t|s_{t-1}),  \nonumber \\\forall s_t\neq s^f:\, &P_f(s_t|s_{t-1}):= \left\{\begin{matrix}
        \frac{P_F(s_t|s_{t-1})}{\sum_{s: s\neq s^f} P_F(s|s_{t-1})+\epsilon^f}  &\text{  if  } R(s_{t-1})\leq R_0\\P_F(s_t|s_{t-1}) &\text{  otherwise}
    \end{matrix}\right.\text{,}  \nonumber\\   &P_f(s^f|s_{t-1}):= \left\{\begin{matrix}
        \frac{\epsilon^f}{\sum_{s: s\neq s^f} P_F(s|s_{t-1})+\epsilon^f}  &\text{  if  } R(s_{t-1})\leq R_0\\P_F(s^f|s_{t-1}) &\text{  otherwise}
    \end{matrix}\right.\text{,}
\end{align}
where $\epsilon^f =10^{-5}$, the corresponding normalized distribution can be understood as $P_G(\tau|x\rightarrow s^f)=P_f(\tau|x\rightarrow s^f)\propto P_f(\tau_{\preceq x})$, and $P_f(\tau)=\prod_{t=1}^{T} P_f(s_t|s_{t-1})$. Similar to the proof of Proposition~\ref{TB-equivalence}, it can be verified that $\nabla_{\phi} D_{KL}^{\rho}(P_{B}(\tau|x\rightarrow s^f;\phi), P_G(\tau|x\rightarrow s^f))$$=\nabla_{\phi} D_{KL}^{\rho}(P_{B}(\tau|x\rightarrow s^f;\phi),\widetilde{P}_G(\tau|x\rightarrow s^f))$. As all the coordinate tuples can be terminating states (i.e., $s^f$ is the child of all the other states) the expression above means that $P_f$ assigns a low probability to the event of the terminating state being a state with a low reward. In this way, we discourage the generative process from stopping early at low reward coordinate tuples.
\begin{figure*}[t]
  \centering
\begin{minipage}[t]{0.45\linewidth}
  \centering
\includegraphics[width=1.0\textwidth]{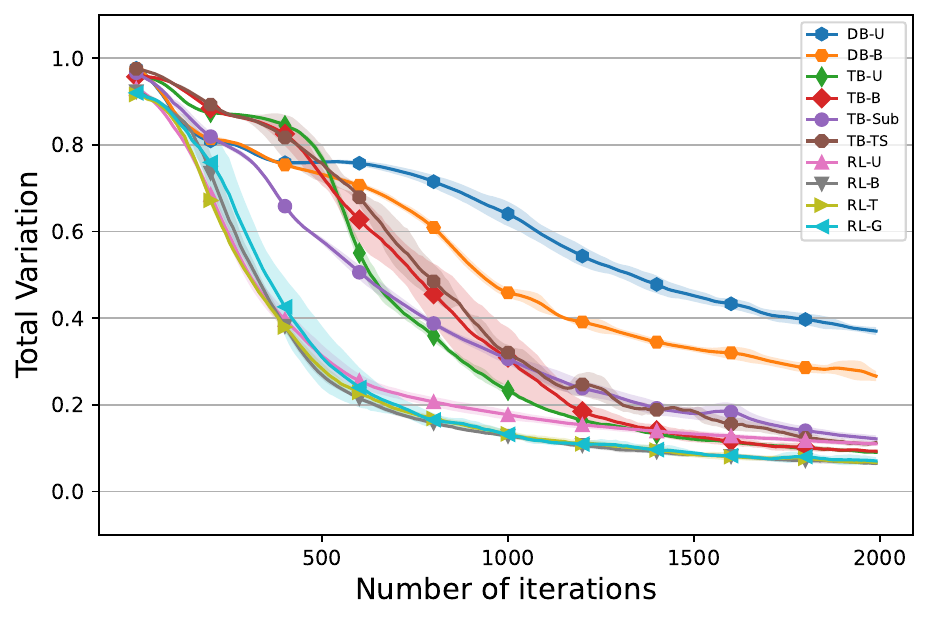}
\end{minipage}
\begin{minipage}[t]{0.45\linewidth}
  \centering
\includegraphics[width=1.0\textwidth]{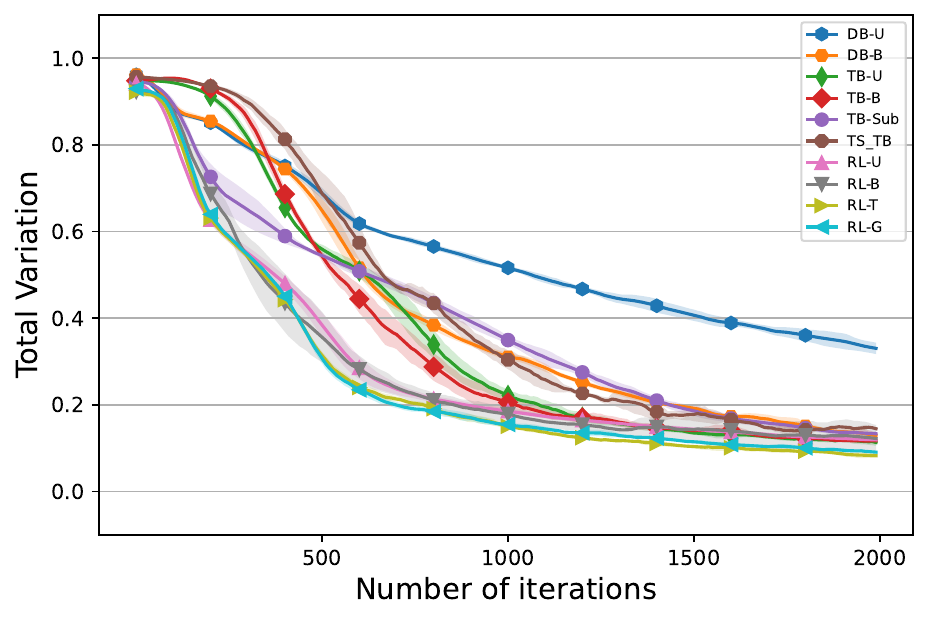}
 \end{minipage}
 \caption{Training curves by $D_{TV}$ between $P_F^\top$ and $P^\ast$ for $64\times64\times64$~(left) and $32\times32\times32\times32$ hyper-grids~(right). The curves are plotted
based on means and standard deviations of metric values across five runs and smoothed by a sliding window of length 10. Metric values
are computed every 10 iterations.}\label{Hyper_training_64_32}
\end{figure*}


\begin{figure*}[t]\hspace{-1.5mm}
  \centering
\begin{minipage}[t]{0.45\linewidth}
  \centering
\includegraphics[width=1.0\textwidth]{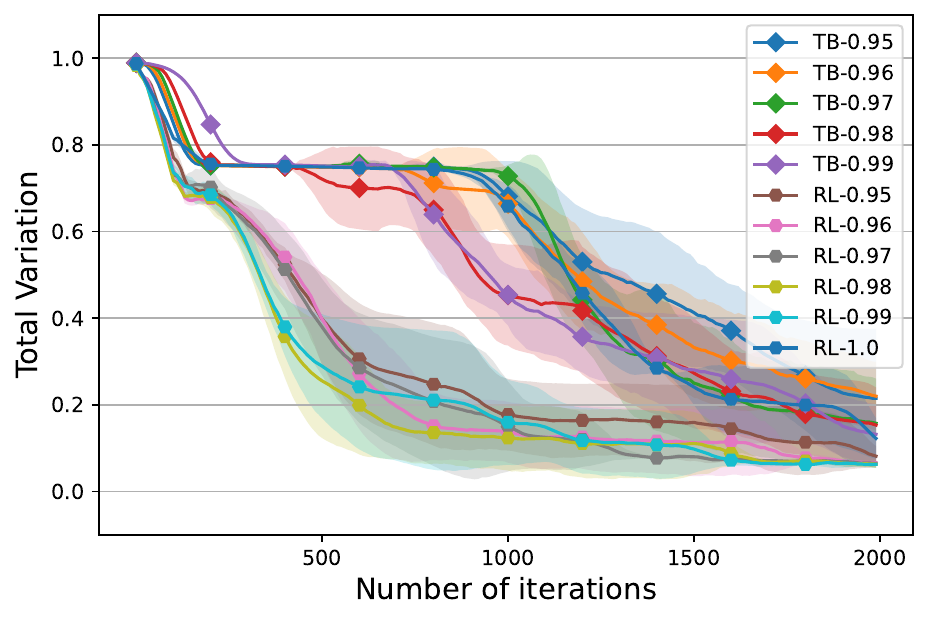}
\end{minipage}\hspace{3mm}
\begin{minipage}[t]{0.45\linewidth}
  \centering
\includegraphics[width=1.0\textwidth]{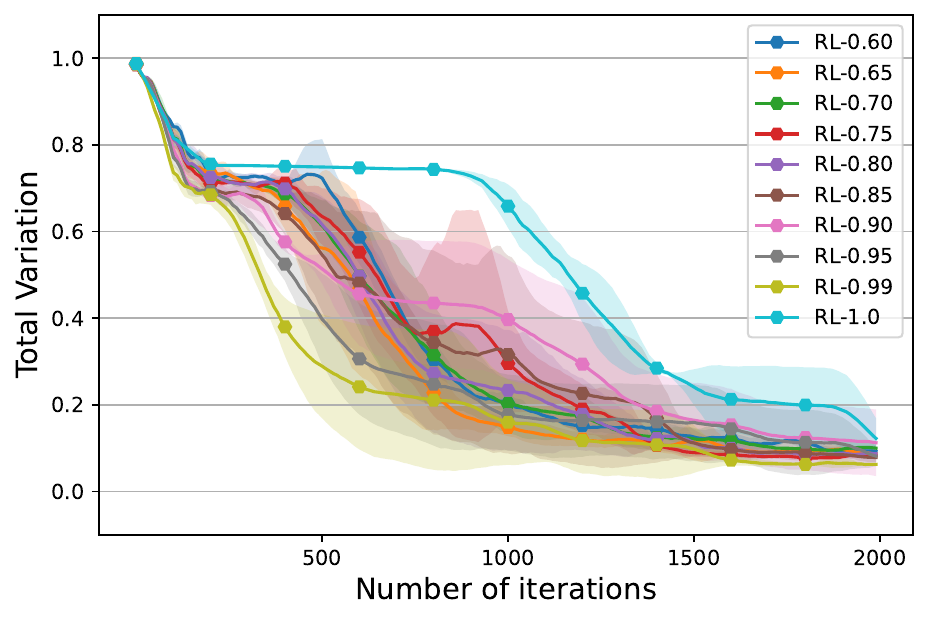}
 \end{minipage}
 \caption{Performance comparison between RL-U with different $\lambda$ values and TB-U with different $\gamma$ values. The curves are plotted based
on their mean and standard deviation values across five runs and smoothed by a sliding window of length 10. Metric values
are computed every 10 iterations.}\label{TB-RL-fig}
\end{figure*}
\paragraph{Model architecture} Forward policy $P_F(\cdot|\cdot)$ is parametrized by a neural network with 4 hidden layers and the hidden dimension is $256$. Backward policy $P_B(\cdot|\cdot)$ is fixed to be uniform over valid actions or parameterized in the same way as $P_F$. Coordinate tuples are transformed by K-hot encoding before being fed into neural networks. 

\paragraph{Additional experiment results}
The obtained results from five runs on $64\times 64\times64$ and $32\times 32\times32\times32$ grids are summarized in Fig.~\ref{Hyper_training_64_32}  and Table \ref{hyper_grid_table_64_32}. The graphical illustrations of $P_F^\top(x)$ are shown in Figs.~\ref{64-plots} and~\ref{32-plots}. We can observe similar performance trends to those in $256\times 256$ and $128\times128$ grids: both TB-based methods and our policy-based method are better than DB-based method, and our policy-based methods achieve much faster convergence than TB-based methods. While these trends are less obvious than in $256\times 256$ and $128\times128$ grids, this phenomenon can be ascribed to the fact that the environment height $N$ has more influence on the modeling difficulty than the environment dimension $D$. The reason is that hyper-grids are homogeneous w.r.t. each dimension, and the minimum distance between modes only depends on $N$.

\subsection{Biological and Molecular Sequence Design}\label{seq-add}
\begin{figure*}[t]
  \centering
\begin{minipage}[t]{0.45\linewidth}
  \centering
\includegraphics[width=1.0\textwidth]{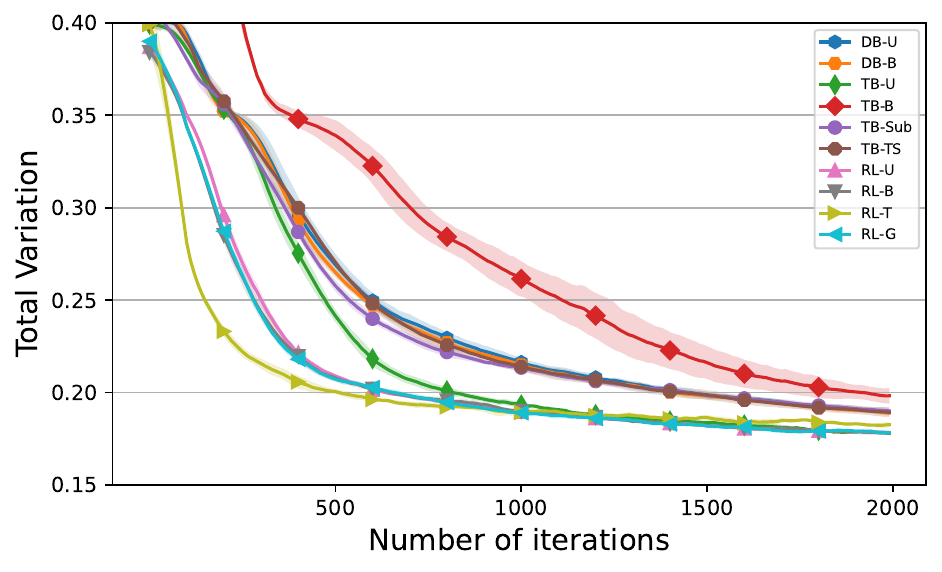}
\end{minipage}\hspace{3mm}
\begin{minipage}[t]{0.45\linewidth}
  \centering
\includegraphics[width=1.0\textwidth]{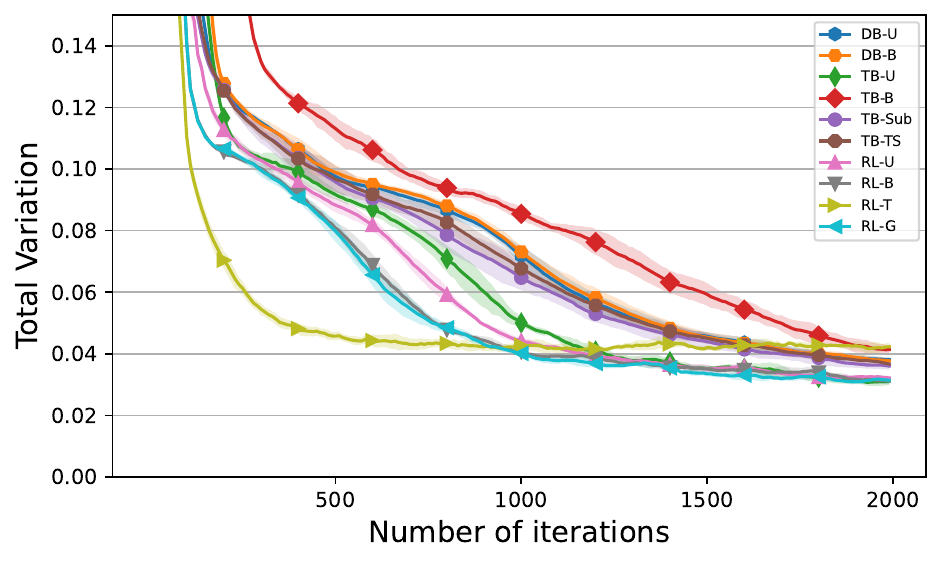}
 \end{minipage}
 \caption{Training curves by $D_{TV}$ between $P_F^\top$ and $P^\ast$ for SIX6~(left) and QM9~(right). The curves are plotted
based on means and standard deviations of metric values across five runs and smoothed by a sliding window of length 10. Metric values
are computed every 10 iterations.}\label{qm9-Six6-training-TV}
\end{figure*}

\begin{figure*}[t]
  \centering
\begin{minipage}[t]{0.45\linewidth}
  \centering
\includegraphics[width=1.0\textwidth]{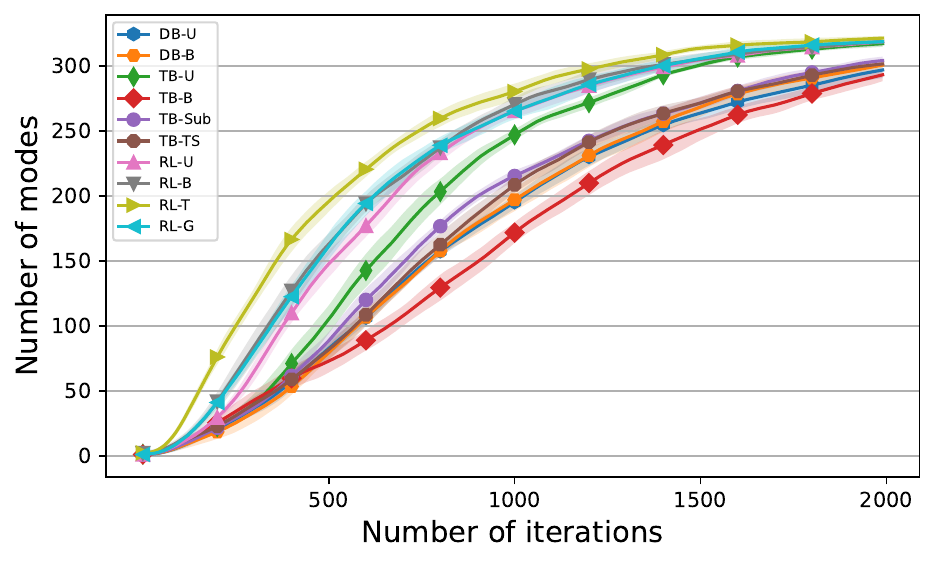}
\end{minipage}\hspace{3mm}
\begin{minipage}[t]{0.45\linewidth}
  \centering
\includegraphics[width=1.0\textwidth]{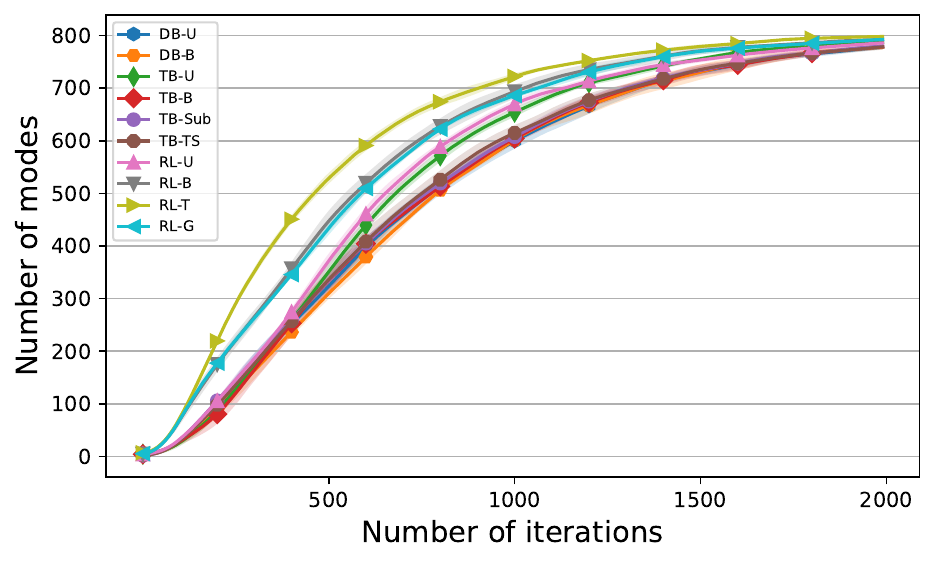}
 \end{minipage}
 \caption{Training curves by the number of modes discovered during training for SIX6~(left) and QM9~(right). The curves are plotted
based on means and standard deviations of metric values across five runs and smoothed by a sliding window of length 10. Metric values
are computed every 10 iterations.}\label{qm9-Six6-training-mode}
\end{figure*}
\paragraph{Environment} In this environment, $\mathcal{S}\setminus{s^f}=\{-1,0,\ldots,N-1\}^D$ with element $s$ corresponds to a sequence composed of integers ranging from $-1$ to $N-1$. The set $\{0,\ldots,N-1\}$ denotes the $N$ nucleotide types or molecular building blocks, and the integer $-1$ represents that the corresponding position within $s$ is unfilled. The initial state $s^0=(-1,\ldots,-1)$ represents an empty sequence and the final state $s^f=(N,\ldots,N)$. For $s_t\in \mathcal{S}_t$, there are $t$ elements in $\{0,\ldots,N-1\}$ and the rest equal to $-1$. There are $N\cdot(D-t)$ actions in $\mathcal{A}(s)$ that correspond to fill in one of the empty slots by one integer in $\{0,\ldots, N-1\}$. The generative process will not stop until sequences are fulfilled. By definition, $\mathcal{G}$ is a graded DAG and $\mathcal{S}_D=\mathcal{X}=\{0,\ldots,N-1\}^D$. We use the reward values provided in the dataset directly. Following~\citet{shen2023towards}, we compute reward exponents $R^{\beta}(x)$ with hyper-parameter $\beta$ set to 3, 5, 3, 6 and normalize the reward exponents to $[1\times10^{-3},10]$, $[1\times10^{-3},10]$, $[0,10]$ and $[1\times10^{-3},10]$ for the SIX6, QM9, PHO4 and sEH datasets respectively. The guided distribution design also follows~\citet{shen2023towards}. For content completeness, we provide the definitions as:
\begin{align}
    P_G(\tau|x)=\prod_{t=1}^T P_G(s_t|s_{t-1}, x),\quad P_G(s_t|s_{t-1}, x)=\frac{\mathrm{score}(s_t|x)}{\sum_{s^\prime\in Ch(s_{t-1})}\mathrm{score}(s'|x)},\nonumber\\  \mathrm{score}(s|x):= \left\{\begin{matrix}
      \mathrm{mean}(\{R(x')|s\in x', x'\in \mathcal{X}^{\mathrm{replay}}\})  &\text{  if  } s \in x\\0 &\text{  otherwise}
    \end{matrix}\right.
\end{align}
where $\mathcal{X}^{\mathrm{replay}}$ corresponds to a replay buffer that stores the sampled terminating states during training.

\begin{figure*}[t]
  \centering
  \begin{minipage}[t]{0.45\linewidth}
  \centering
\includegraphics[width=1.0\textwidth]{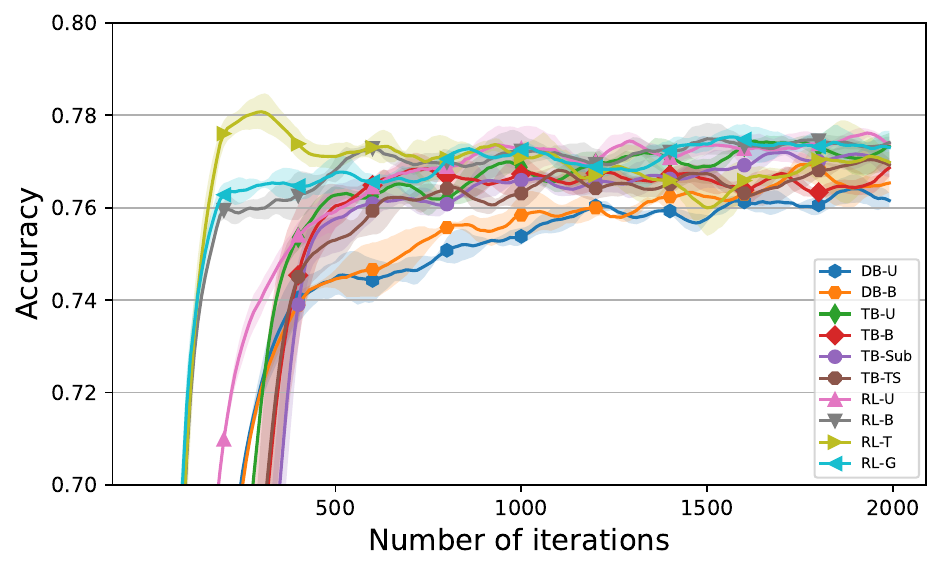}
 \end{minipage}\hspace{3mm}
   \begin{minipage}[t]{0.45\linewidth}
  \centering
\includegraphics[width=1.0\textwidth]{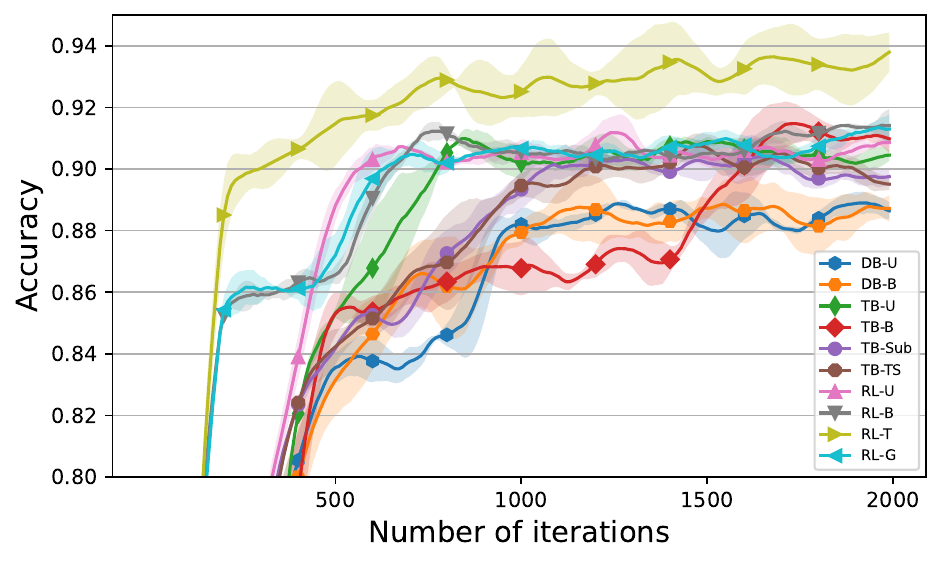}
 \end{minipage}\caption{Training curves by $Acc$ of $P_F^\top$ w.r.t. $P^\ast$ for PHO4 and sEH. The curves are plotted based on their mean and standard deviation values. The curves are plotted
based on means and standard deviations of metric values across five runs and smoothed by a sliding window of length 10. Metric values
are computed every 10 iterations.}\label{TF10-sEH-training}
\end{figure*}

\begin{figure*}[t]
  \centering
  \begin{minipage}[t]{0.45\linewidth}
  \centering
\includegraphics[width=1.0\textwidth]{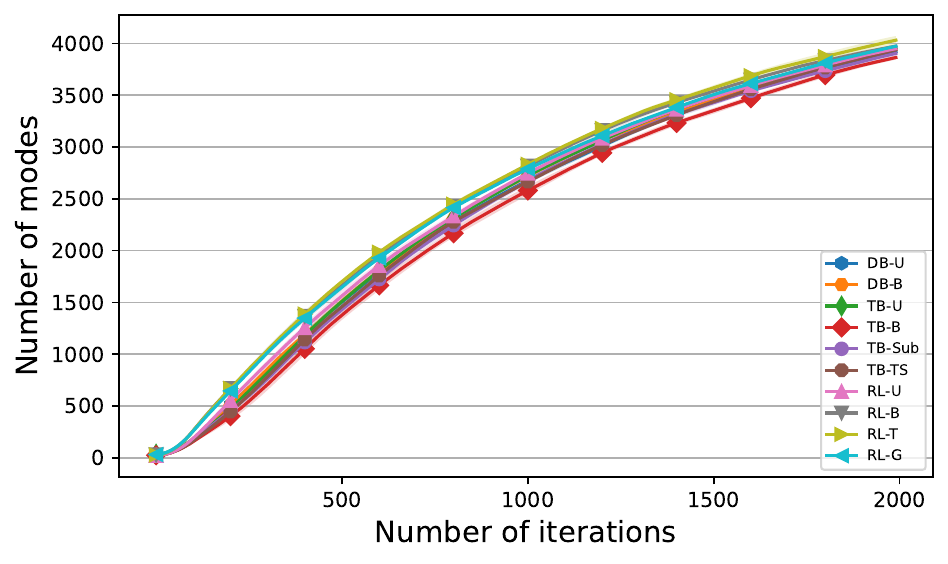}
 \end{minipage}\hspace{3mm}
   \begin{minipage}[t]{0.45\linewidth}
  \centering
\includegraphics[width=1.0\textwidth]{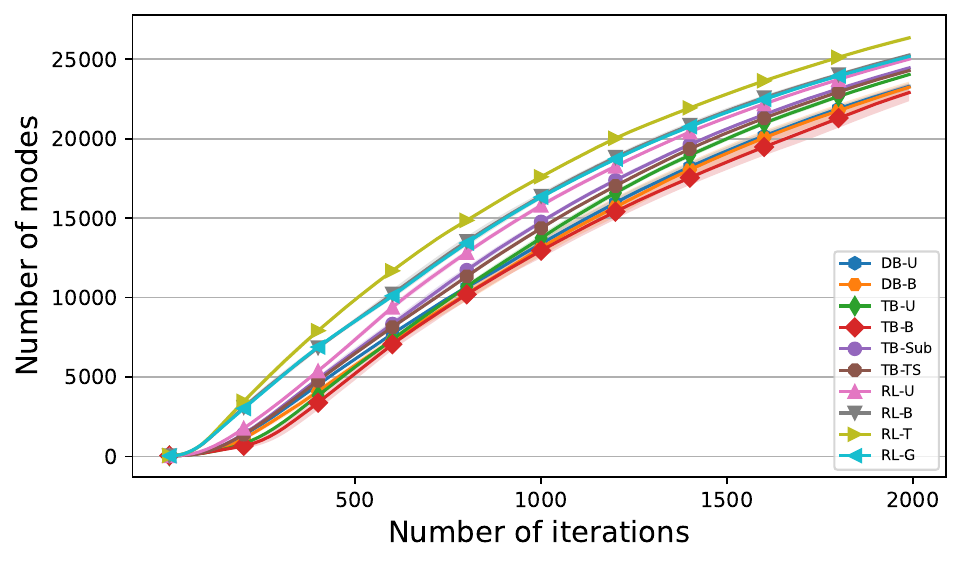}
 \end{minipage}\caption{Training curves by the number of modes discovered over training for PHO4 and sEH. The curves are plotted
based on means and standard deviations of metric values across five runs and smoothed by a sliding window of length 10. Metric values
are computed every 10 iterations. }\label{TF10-sEH-training-mode}\vspace{-2mm}
\end{figure*}

\paragraph{Model architecture} Policies are constructed in the same way as the hyper-grid modeling experiment. Integer sequences are transformed by K-hot encoding before being fed into neural networks.

\paragraph{Additional experiment results} For PHO4 and sEH dataset, the exact computation of $P_F^\top$ by dynamic programming is expensive. Thus, we only plot the training curves by $Acc$ and the number of modes as shown in Figs.~\ref{TF10-sEH-training} and~\ref{TF10-sEH-training-mode}, and the means and standard deviations of metric values at the last iteration are provided in Table~\ref{Bio-table2}. Here, the expectation $\mathbb{E}_{P_F^\top(x)}[R(x)]$ in $Acc$ is approximated by averaging over $10^5$ terminating state samples. We can observe similar performance trends as in the QM9 and SIX6 datasets. Our policy-based methods achieve faster convergence rates and better converged $Acc$ values than the TB-based and DB-based methods. While the converged $Acc$ of RL-T is slightly worse than the other policy-based methods, it achieves the fastest convergence rate. RL-G and RL-B, both employing a parametrized $\pi_B$, demonstrate similar performance and converge faster than RL-U, which utilizes a uniform $\pi_B$.

\subsection{Bayesian Network Structure Learning}\label{DAG-add}

In this experiment, we investigate GFlowNets for BN structure learning following the settings adopted in~\citet{malkin2022gflownets}. The set $\mathcal{X}$ corresponds to a set of BN structures, which are also DAGs. BN structure learning can be understood as approximating $P(x|\mathcal{D})\propto R(x)$ given a dataset $\mathcal{D}$. Given a set of nodes, the state space for GFlowNets is the set of all possible DAGs over the node set. The actions correspond to adding edges over a DAG without introducing a cycle. The generative process of a BN structure is interpreted as starting from an empty graph, an action is taken to decide to add an edge or terminate the generative process at the current graph structure.

\vspace{-2mm}
\paragraph{Environment} A Bayesian Network is a probabilistic model that represents the joint distribution of $N$ random variable and the joint distribution factorizes according to the network structure $x$:
\begin{align}
    P(y_1,\ldots,y_N)=\prod_{n=1}^N P(y_n|Pa_x(y_n))
\end{align}
where $Pa_x(y_n)$ denote the set of parent nodes of $y_n$ according to graph $x$. As the structure of any graph can be represented by its adjacency matrix, the state space can be defined as $\mathcal{S}:=\left\{s|\mathcal{C}(s)=0, s\in\{0,1\}^{N\times N}\right\}$ where $\mathcal{C}$ corresponds to the acyclic graph constraint~\citep{deleu2022bayesian}, the initial state $s^0=\mathbf{0}^{N\times N}$
and specially $s^f:=-\mathbf{1}^{N\times N}$ in our implementation. For each state $s$,  $a\in \mathcal{A}(s)$ can be any action that turns one of 0 values of $s$ to be $1$ (i.e. adding an edge) while keeping $\mathcal{C}(s')=0$ for the resulting graph $s'$, or equal to $(s\rightarrow s^f)$ that stopping the generative process and return $x=s$ as the terminating state. By definition, the corresponding $\mathcal{G}$ in this environment is not graded.
Given observation dataset $\mathcal{D}_y$ of $y_{1:N}$, the structure learning problem can be understood as approximating $P(x|\mathcal{D}_y)\propto P(x,\mathcal{D}_y)=P(\mathcal{D}_y|x)P(x)$. Without additional information about graph structure $x$, $P(x)$ is often assumed to be uniform. Thus, $P(x|\mathcal{D}_y)\propto P(\mathcal{D}_y|x)$ and the reward function is defined as $R(x)\propto P(\mathcal{D}_y|x)$. Distribution $P(\mathcal{D}_y|x)$ is also called graph score and we use $BGe$ score~\citep{kuipers2014addendum} in our experiment. Following~\citet{malkin2022gflownets}, the ground-truth graph structure and the corresponding observation dataset $\mathcal{D}_y$ are simulated from Erdős–Rényi model. The guided distribution design follows the hyper-grid experiment. A low probability value, $10^{-5}$ is assigned to the transition probability of $(s\rightarrow s^f)$ if $|\log R(s)-\log R^{\max}|\leq 5$. 

\begin{figure}[t]
  \centering
\includegraphics[width=0.45\textwidth]{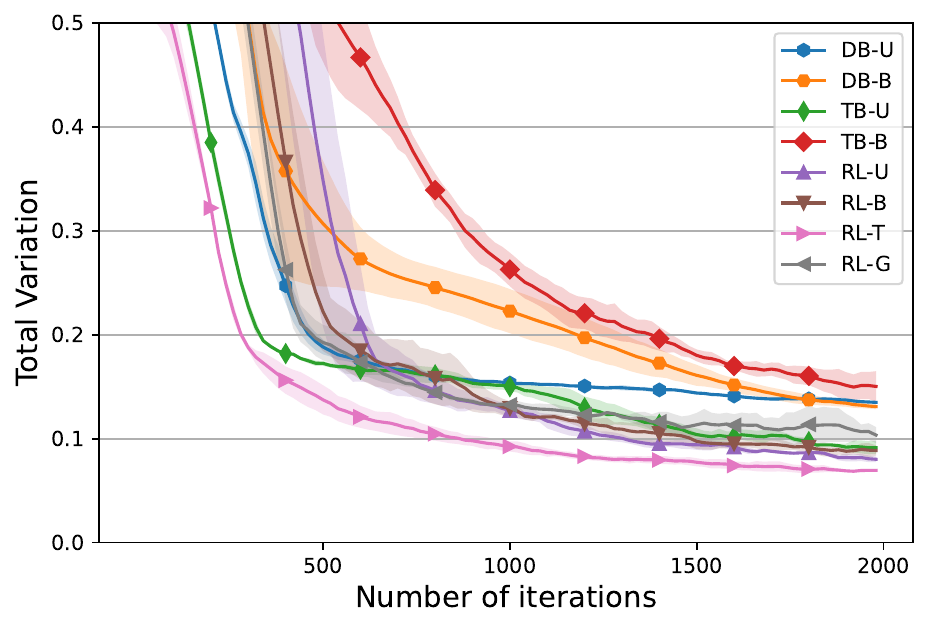} \caption{Training curves by $D_{TV}$ of $P_F^\top$ w.r.t. $P^\ast$ for the BN structure
learning experiment. The curves are plotted
based on means and standard deviations of metric values across five runs and smoothed by a sliding window of length 10. Metric values
are computed every 10 iterations.}\label{DAG-training}
\end{figure}
\paragraph{Model architecture} Policies are constructed in the same way as the hyper-grid modeling experiments, but adjacency matrices are fed into neural networks directly without encoding. As reported in~\citep{deleu2022bayesian}, the distribution can be very peaky between adjacent graph structures. The reward $R(x)$ (typically $\approx e^{80}$) in this experiment is much larger than those (typically $\approx 10$) in the previous two sets of experiments. These facts give rise to numerical issues for reliable estimation of value function. Thus, we compute the gradients w.r.t. $J_B$ empirically as the gradients w.r.t. TB-based objective, that is,  $V_B$ is not utilized during training. Besides, $\log Z$ is also very large, so we set the learning rate to $1$, which is selected from $\{0.1,0.3,0.5,0.8,1.0\}$ by TB-U.

\paragraph{Experiment results}
The number of possible DAGs grows exponentially with the number of nodes. Thus, we test the same
benchmark in~\citet{malkin2022gflownets} with the number of nodes
set to 5 and the corresponding total numbers of DAGs is
about $2.92 \times 10^4$. The number of edges in the ground-truth DAG is set to 5, and the size of the observation dataset is set to 100. The experimental results across five runs
are shown in Fig.~\ref{DAG-training} and Table~\ref{BN-table}. The graphical illustrations of $P_F^\top(x)$ are shown in Fig.~\ref{BN-plots}. As expected, performance trends similar to those in the previous two sets of experiments are observed. The converged $D_{TV}$ values of all the policy-based methods are better than those of the value-based methods, with only TB-U achieving comparable performance. Besides, RL-T achieves
the best converged $D_{TV}$ value and has the fastest convergence among all the methods. These results further demonstrate the
effectiveness of our policy-based methods for GFlowNet
training.

\clearpage
\subsection{Tables of Converged Metric Values}\label{all-table}
\begin{table}[h!]
\centering
\begin{tabular}{lllllll}
\toprule[1.1pt]
 & \multicolumn{3}{c}{$256\times256$} & \multicolumn{3}{c}{$128\times128$} \\ \hline
Method & \multicolumn{1}{c}{$D_{TV}\downarrow(\times10^{-1})$} & \multicolumn{1}{c}{$D_{JSD}\downarrow(\times10^{-2})$} & \multicolumn{1}{c}{Time$\downarrow$} & \multicolumn{1}{c}{$D_{TV}\downarrow(\times10^{-1})$} & \multicolumn{1}{c}{$D_{JSD}\downarrow(\times10^{-2})$} & \multicolumn{1}{c}{Time$\downarrow$} \\ \midrule[1.1pt]
DB-U & $6.050\pm0.129$ & $24.61\pm1.147$ & 37.2 & $3.233\pm0.138$  & $8.357\pm0.471$&  18.5\\
DB-B & $4.459\pm0.135$ & $14.42\pm0.656$ & 46.8 & $2.621\pm0.122$  & $5.621\pm0.502$  &  21.3\\
TB-U & $0.728\pm0.095$ & $0.763\pm0.131$ & \textbf{30.9} & $0.449\pm0.078$ & $0.338\pm0.062$& \textbf{14.3} \\
TB-B & $2.101\pm1.052$  & $6.612\pm5.093$ & 32.5 & $0.441\pm0.080$ & $0.355\pm0.105$ &  16.4\\
TB-Sub & $1.461\pm0.058$  & $2.915\pm0.285$ & 71.5 & $0.450\pm0.032$ & $0.367\pm0.060$ &  31.3\\
TB-TS & $1.277\pm0.268$  & $1.714\pm0.627$ & 38.5 & $0.481\pm0.050$ & $0.364\pm0.058$ &  20.4\\ \hline
RL-U & $0.621\pm0.057$ & $0.770\pm0.096$  & 44.4 & $0.440\pm0.077$ & $0.390\pm0.053$ & 18.3 \\
RL-B & $0.704\pm0.190$ & $1.064\pm0.286$ & 68.5 & $0.490\pm0.088$ &$0.467\pm0.082$  & 35.2 \\
RL-T & $0.708\pm0.058$ &  $0.774\pm0.054$ & 90.2 & $0.503\pm0.043$ & $0.426\pm0.053$ & 58.3 \\
RL-G & $\mathbf{0.439\pm0.037}$ &  $\mathbf{0.541\pm0.017}$& 69.8 & $\mathbf{0.427\pm0.158}$ & $\mathbf{0.353\pm0.180}$& 35.4 \\ \bottomrule[1.1pt]
\end{tabular}\caption{Converged metric values of different methods for the modeling of $256\times256$ and $128\times128$ grids. Training time costs are provided in minutes. }\label{hyper_grid_table_256_128}
\end{table}

\begin{table}[h!]
\centering
\begin{tabular}{lllllll}
\toprule[1.1pt]
 & \multicolumn{3}{c}{$64\times64\times 64$} & \multicolumn{3}{c}{$32\times32\times 32\times 32$ } \\ \hline
Method & \multicolumn{1}{c}{$D_{TV}\downarrow(\times 10^{-1})$} & \multicolumn{1}{c}{$D_{JSD}\downarrow(\times 10^{-2})$} & \multicolumn{1}{c}{Time$\downarrow$} & \multicolumn{1}{c}{$D_{TV}\downarrow(\times 10^{-1})$} & \multicolumn{1}{c}{$D_{JSD}\downarrow(\times 10^{-2})$} & \multicolumn{1}{c}{Time$\downarrow$} \\ \midrule[1.1pt]
DB-U & $3.687\pm0.132$  & $10.83\pm0.611$ & 19.3 & $3.254\pm0.151$  & $7.570\pm0.682$&  19.1\\
DB-B & $2.606\pm0.083$ & $6.559\pm0.486$ & 21.5 &  $1.248\pm0.041$  & $1.684\pm0.119$  &  20.4\\
TB-U & $0.870\pm0.065$& $0.737\pm0.082$& \textbf{16.1} &$1.086\pm0.078$  & $1.123\pm0.110$ & \textbf{17.3} \\
TB-B & $0.909\pm0.069$  &  $0.753\pm0.101$ & 17.4 & $1.191\pm0.112$ & $1.258\pm0.184$  &  18.6\\
TB-Sub & $1.185\pm0.106$  &  $1.397\pm0.276$ & 27.5 & $1.319\pm0.085$ & $1.961\pm0.171$  &  25.2\\
TB-TS & $1.164\pm0.222$  &  $1.255\pm0.410$ & 21.3 & $1.390\pm0.161$ & $1.781\pm0.295$  &  21.2\\\hline
RL-U & $1.082\pm0.060$  & $1.287\pm0.105$  & 17.4 & $1.180\pm0.049$ & $1.455\pm0.111$  &18.5 \\
RL-B & $\mathbf{0.647\pm0.111}$  & $\mathbf{0.440\pm0.132}$& 30.2 & $1.203\pm0.330$ & $1.304\pm0.723$  & 34.4 \\
RL-T & $0.654\pm0.069$ & $0.456\pm0.084$  & 49.0 & $\mathbf{0.838\pm0.105}$& $\mathbf{0.591\pm0.140}$ & 59.3 \\
RL-G &$0.670\pm0.095$   &  $0.527\pm0.198$& 31.8 &$0.888\pm0.068$  & $0.666\pm0.091$  & 36.4 \\ \bottomrule[1.1pt]
\end{tabular}\caption{Converged metric values of different methods for the modeling of $64\times64\times 64$ and $32\times32\times 32\times 32$  grids. Training time costs are provided in minutes. }\label{hyper_grid_table_64_32}
\end{table}

\begin{table}[h!]
\centering
\begin{tabular}{lllll}
\toprule[1.1pt]
 & \multicolumn{3}{c}{SIX6} &  \\ \hline
Method & \multicolumn{1}{c}{$Acc\uparrow(\times 10^2)$} & \multicolumn{1}{c}{$D_{TV}\downarrow(\times 10^{-1})$} & \multicolumn{1}{c}{$D_{JSD}\downarrow(\times 10^{-2})$} & \multicolumn{1}{c}{Number of modes$\uparrow$}\\ \midrule[1.1pt]
DB-U & $86.26\pm1.20$ & $1.889\pm0.021$ & $3.261\pm0.052$ & $298.24\pm6.84$ \\
DB-B & $86.26\pm1.51$ & $1.886\pm0.020$ & $3.263\pm0.064$ & $302.00\pm4.97$  \\
TB-U & $93.88\pm1.13$ & $\mathbf{1.767\pm0.015}$ & $\mathbf{2.915\pm0.025}$ & $317.92\pm2.74$ \\
TB-B & $90.94\pm1.46$ & $1.998\pm0.055$ & $3.542\pm0.146$ & $295.36\pm4.69$ \\
TB-Sub & $88.57\pm0.91$ & $1.893\pm0.021$ & $3.261\pm0.056$ & $305.26\pm3.22$  \\
TB-TS & $88.66\pm1.04$ & $1.887\pm0.029$ & $3.241\pm0.050$ & $302.58\pm1.27$ \\ \hline
RL-U & $94.30\pm1.22$ & $1.779\pm0.018$ & $2.956\pm0.025$ & $318.84\pm3.08$  \\
RL-B & $94.68\pm1.62$ & $1.786\pm0.019$ & $2.967\pm0.036$ & $319.14\pm1.96$ \\
RL-T & $93.40\pm1.01$ & $1.832\pm0.022$ & $3.174\pm0.049$ & $\mathbf{321.74\pm1.72}$  \\
RL-G & $\mathbf{94.70\pm1.65}$ & $1.782\pm0.020$ & $2.951\pm0.025$ & $318.96\pm3.03$  \\ \bottomrule[1.1pt]
\end{tabular}\caption{Converged metric values of different methods for the SIX6 datasets.}\label{Bio-table}
\end{table}

\begin{table}[h!]
\centering
\begin{tabular}{lllll}
\toprule[1.1pt]
 & \multicolumn{3}{c}{QM9} &  \\ \hline
Method & \multicolumn{1}{c}{$Acc\uparrow(\times 10^2)$} & \multicolumn{1}{c}{$D_{TV}\downarrow(\times 10^{-1})$} & \multicolumn{1}{c}{$D_{JSD}\downarrow(\times 10^{-2})$} & \multicolumn{1}{c}{Number of modes$\uparrow$}\\ \midrule[1.1pt]
DB-U & $96.30\pm1.07$ & $1.889\pm0.021$ & $3.261\pm0.052$ & $780.2\pm4.32$ \\
DB-B & $96.93\pm1.07$ & $1.886\pm0.020$ & $3.263\pm0.064$ & $780.0\pm6.04$  \\
TB-U & $98.65\pm1.10$ & $\mathbf{1.767\pm0.015}$ & $\mathbf{2.915\pm0.025}$ & $788.8\pm2.78$ \\
TB-B & $98.36\pm1.08$ & $1.998\pm0.055$ & $3.542\pm0.146$ & $781.4\pm3.72$ \\
TB-Sub & $97.00\pm1.17$ & $1.893\pm0.021$ & $3.261\pm0.056$ & $781.8\pm4.87$  \\
TB-TS & $96.83\pm0.76$ & $1.887\pm0.029$ & $3.241\pm0.050$ & $780.6\pm4.16$ \\ \hline
RL-U & $98.79\pm1.08$ & $1.779\pm0.018$ & $2.956\pm0.025$ & $787.0\pm5.29$  \\
RL-B & $\mathbf{99.10\pm0.90}$ & $1.786\pm0.019$ & $2.967\pm0.036$ & $792.8\pm2.95$ \\
RL-T & $98.74\pm1.27$ & $1.832\pm0.022$ & $3.174\pm0.049$ & $\mathbf{798.2\pm0.84}$  \\
RL-G & $90.04\pm1.20$ & $1.782\pm0.020$ & $2.951\pm0.025$ & $794.0\pm2.16$  \\ \bottomrule[1.1pt]
\end{tabular}\caption{Converged metric values of different methods for the QM9 datasets.}\label{Bio-table1}
\end{table}

\begin{table}[h!]
\centering
\begin{tabular}{lllll}
\toprule[1.1pt]
 & \multicolumn{2}{c}{PHO4} & \multicolumn{2}{c}{sEH} \\ \hline
Method & \multicolumn{1}{c}{$Acc\uparrow(\times 10^2)$} & \multicolumn{1}{c}{Number of modes$\uparrow$} & \multicolumn{1}{c}{$Acc\uparrow$} & \multicolumn{1}{c}{Number of modes$\uparrow$} \\ \midrule[1.1pt]
DB-U & $76.08\pm0.23$ & $3957.33\pm18.61$ & $88.45\pm0.61$ & $23502.97\pm260.03$ \\
DB-B & $76.61\pm0.32$ & $3975.00\pm64.00$ & $88.58\pm0.81$ & $23454.68\pm422.14$ \\
TB-U & $77.47\pm0.29$ & $3993.33\pm9.87$ & $90.49\pm0.48$ & $24250.63\pm114.85$ \\
TB-B & $77.00\pm0.14$ & $3901.33\pm26.08$ & $90.84\pm0.85$ & $23139.13\pm634.65$ \\
TB-Sub & $76.94\pm0.24$ & $3939.33\pm18.61$ & $89.81\pm0.61$ & $24652.47\pm74.43$ \\
TB-TS & $76.78\pm0.14$ & $3973.33\pm38.68$ & $89.57\pm0.24$ & $24491.93\pm90.99$ \\\hline
RL-U & $77.31\pm0.30$ & $3984.67\pm24.79$ & $90.86\pm0.42$ & $25200.67\pm24.61$ \\
RL-B & $\mathbf{77.55\pm0.25}$ & $4004.67\pm15.01$ & $91.36\pm0.85$ & $25454.03\pm142.64$ \\
RL-T & $76.81\pm0.48$ & $\mathbf{4062.67\pm47.69}$ & $\mathbf{93.98\pm0.64}$ & $\mathbf{26530.30\pm142.27}$ \\
RL-G & $77.26\pm0.55$ & $4005.67\pm12.50$ & $91.28\pm0.79$ & $25368.07\pm52.27$ \\ \bottomrule[1.1pt]
\end{tabular}\caption{Converged metric values of different methods for the PHO4 and sEH datasets}\label{Bio-table2}
\end{table}

\begin{table}[h!]
\centering
\begin{tabular}{lll|lll}
\toprule[1.1pt]
Method & \multicolumn{1}{c}{$D_{TV}\downarrow(\times10^{-1})$} & \multicolumn{1}{c|}{$D_{JSD}\downarrow(\times10^{-2})$} & Method & \multicolumn{1}{c}{$D_{TV} \downarrow(\times10^{-1})$} & \multicolumn{1}{c}{$D_{JSD}\downarrow(\times10^{-2})$} \\ \midrule[1.1pt]
DB-U & $1.346\pm0.030$ & $4.863\pm0.420$ & DB-B & $1.284\pm0.027$ & $7.533\pm1.873$ \\
TB-U & $0.898\pm0.124$ & $3.121\pm0.332$ & TB-B & $1.521\pm0.281$ & $13.301\pm2.107$ \\
RL-U & $0.831\pm0.079$ & $7.436\pm1.508$ & RL-B & $0.929\pm0.078$ & $6.114\pm1.179$ \\
RL-T & $\mathbf{0.698\pm 0.052}$ & $\mathbf{1.890\pm0.307}$ & RL-G & $0.964\pm 0.174$ & $5.712\pm0.845$ \\ \bottomrule[1.1pt]
\end{tabular}\caption{Converged metric values of different methods for the BN structure learning experiment, where the ground-truth BN has 5 nodes and 5 edges.}\label{BN-table}
\end{table}
\clearpage
\subsection{Graphical Representation of $P_F^\top$}\label{experiment_illustration}

\begin{figure}[h!]
\centering
\begin{minipage}[t]{0.20\linewidth}
  \centering
\includegraphics[width=1.0\textwidth]{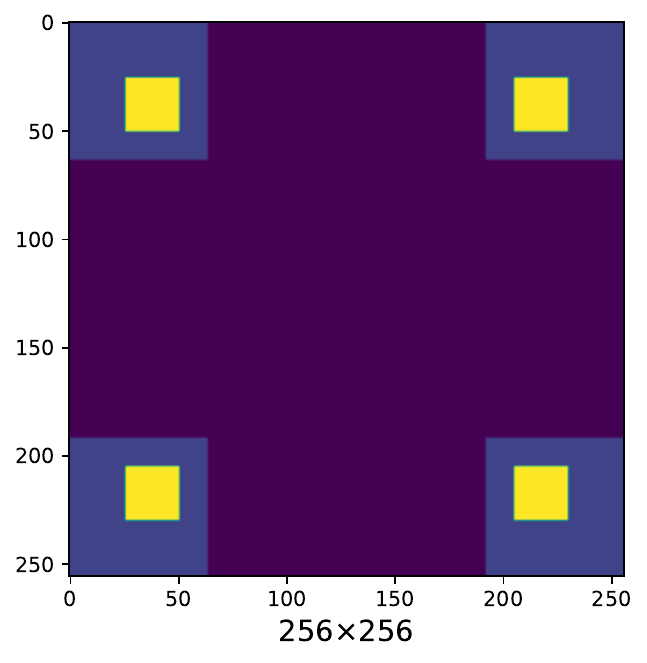}
\end{minipage}\hspace{-1.5mm}
\begin{minipage}[t]{0.20\linewidth}
  \centering
\includegraphics[width=1.0\textwidth]{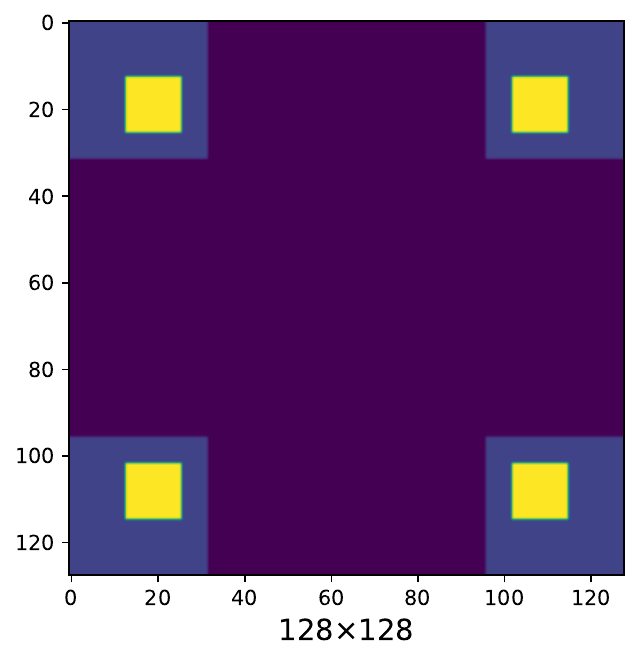}
\end{minipage}\hspace{-1.5mm}
\begin{minipage}[t]{0.20\linewidth}
  \centering
\includegraphics[width=1.0\textwidth]{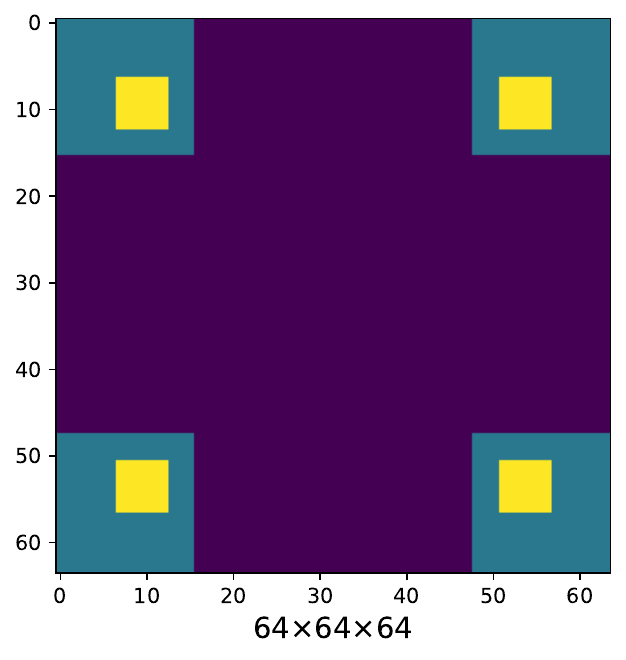}
\end{minipage}\hspace{-1.5mm}
\begin{minipage}[t]{0.20\linewidth}
  \centering
\includegraphics[width=1.0\textwidth]{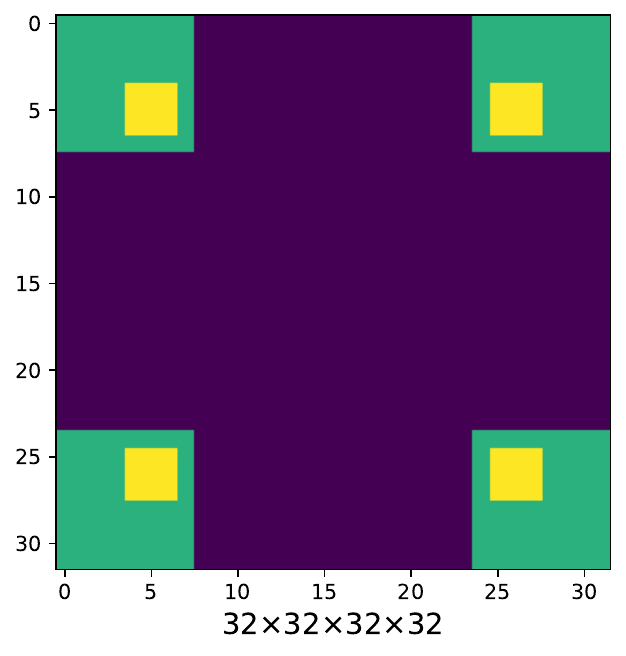}
\end{minipage}\hspace{-1.5mm}
 \caption{Graphical representation of $P^\ast$ for different hyper-grids. For visualization easiness, the ground-truth marginal distributions of two dimensions are plotted for $64\times64\times 64$ and  $32\times 32\times32\times 32$ grids.}
\end{figure}

\begin{figure}[h!]
\begin{minipage}[t]{0.20\linewidth}
  \centering
\includegraphics[width=1.0\textwidth]{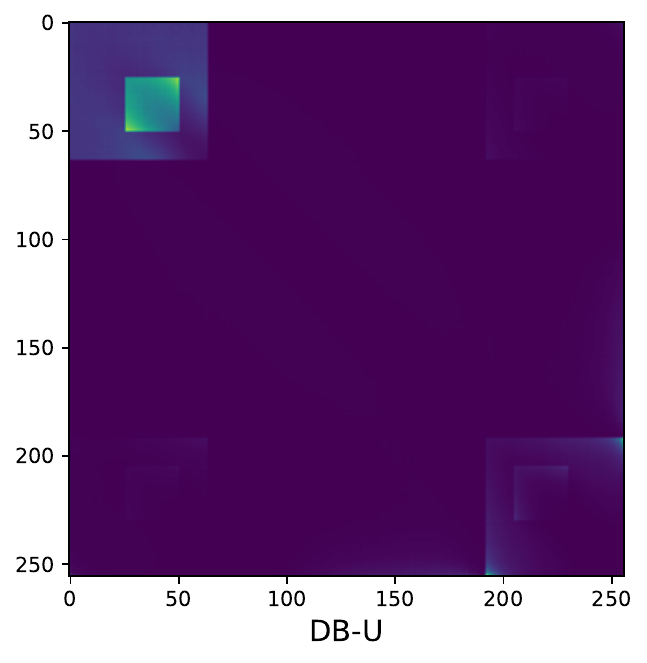}
\end{minipage}\hspace{-1.5mm}
\begin{minipage}[t]{0.20\linewidth}
  \centering
\includegraphics[width=1.0\textwidth]{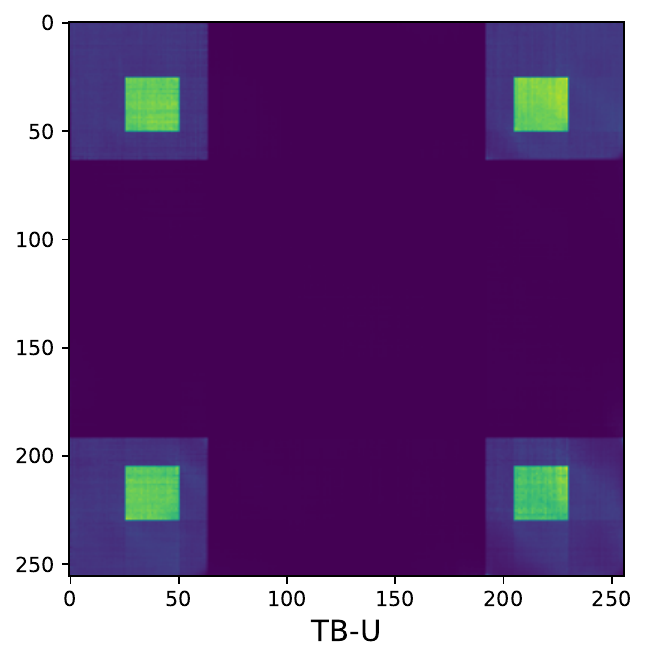}
\end{minipage}\hspace{-1.5mm}
\hspace{-1.5mm}
\begin{minipage}[t]{0.20\linewidth}
  \centering
\includegraphics[width=1.0\textwidth]{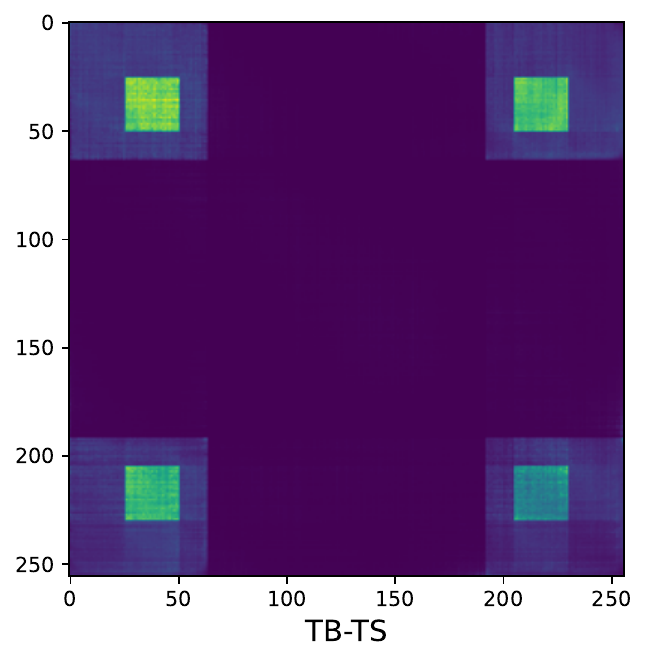}
\end{minipage}\hspace{-1.5mm}
\begin{minipage}[t]{0.20\linewidth}
  \centering
\includegraphics[width=1.0\textwidth]{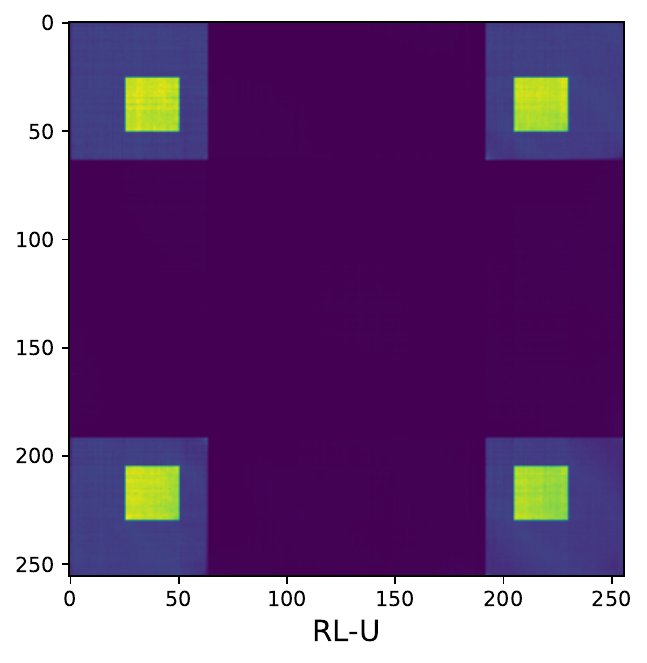}
 \end{minipage}\hspace{-1.5mm}
 \begin{minipage}[t]{0.20\linewidth}
  \centering
\includegraphics[width=1.0\textwidth]{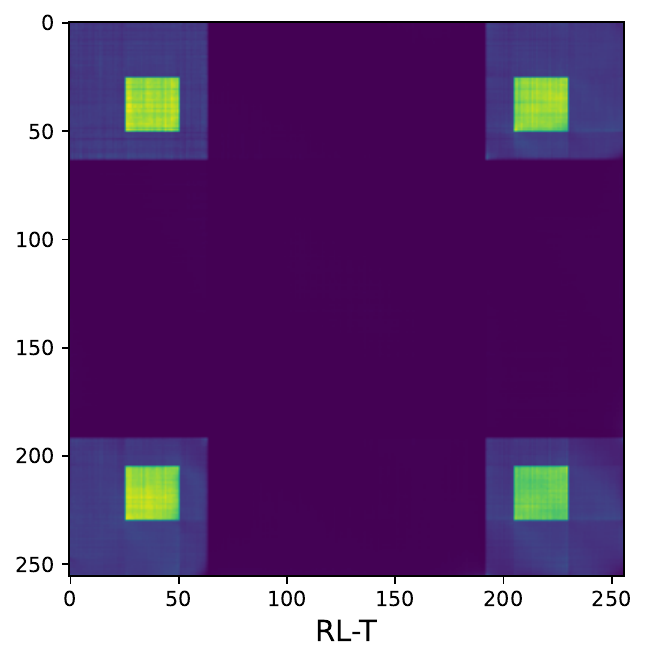}
 \end{minipage}\hspace{-1.5mm}\\
 \begin{minipage}[t]{0.20\linewidth}
  \centering
\includegraphics[width=1.0\textwidth]{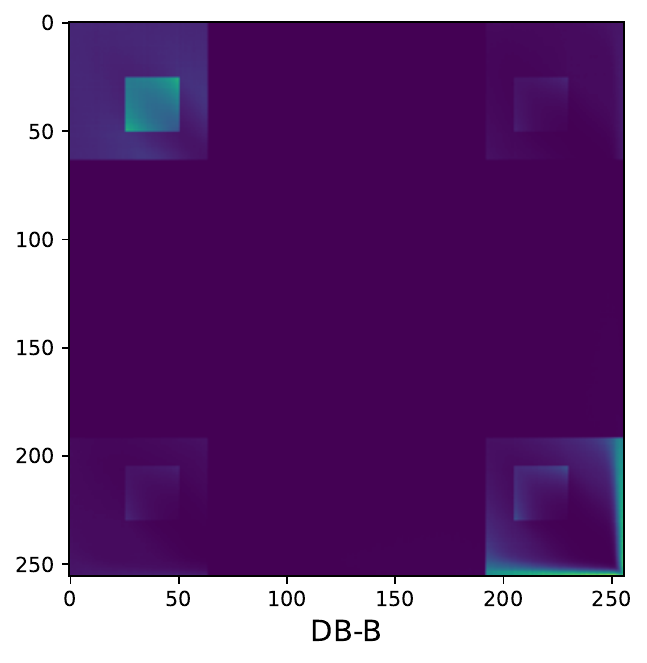}
\end{minipage}\hspace{-1.5mm}
 \begin{minipage}[t]{0.20\linewidth}
  \centering
\includegraphics[width=1.0\textwidth]{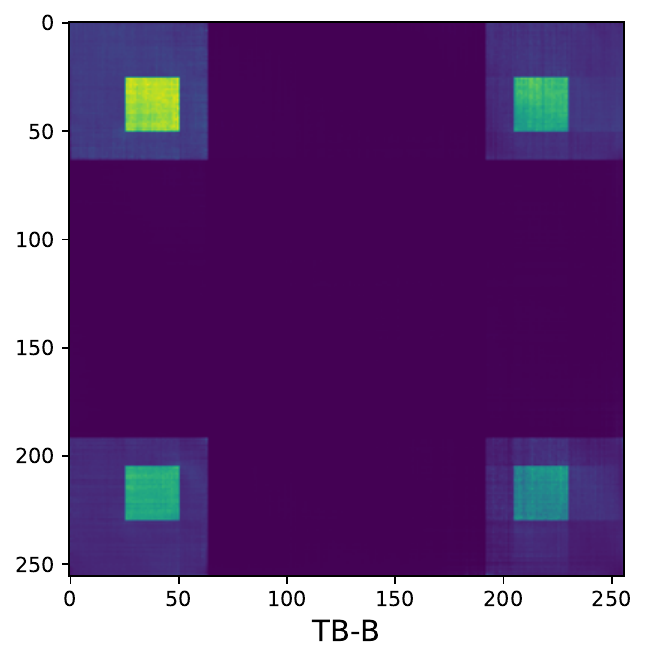}
\end{minipage}\hspace{-1.5mm}
 \begin{minipage}[t]{0.20\linewidth}
  \centering
\includegraphics[width=1.0\textwidth]{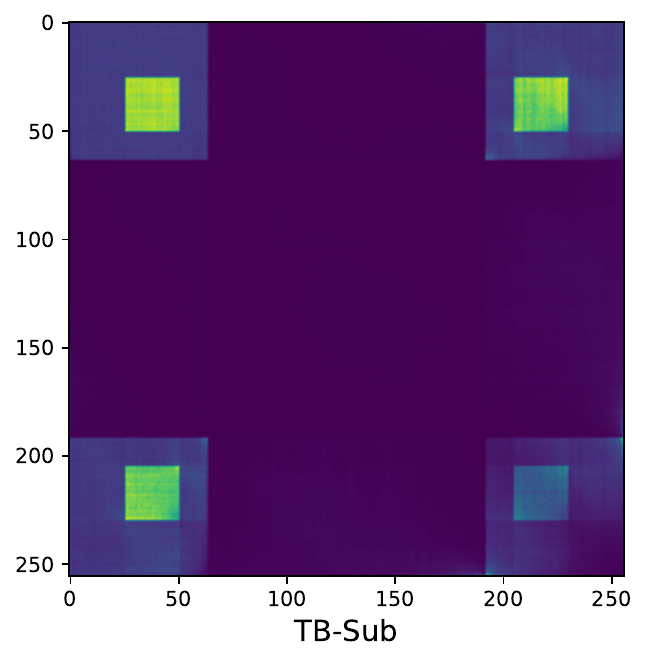}
\end{minipage}\hspace{-1.5mm}
\begin{minipage}[t]{0.20\linewidth}
  \centering
\includegraphics[width=1.0\textwidth]{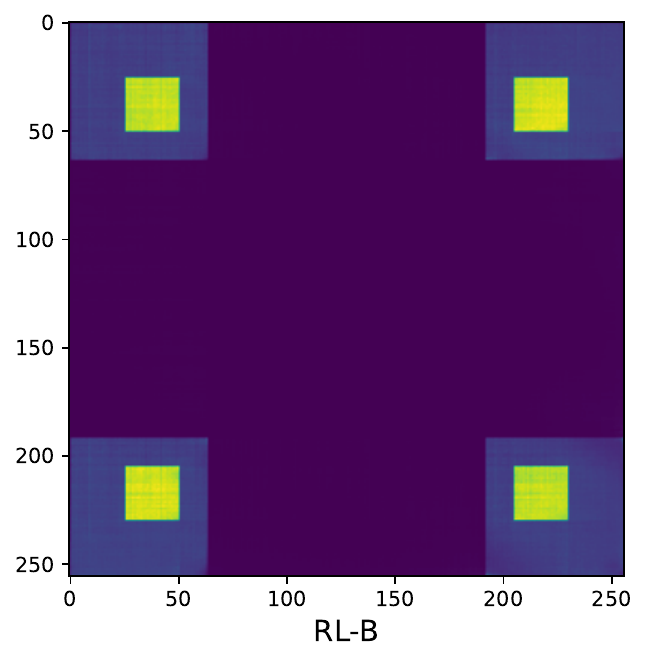}
 \end{minipage}\hspace{-1.5mm}
 \begin{minipage}[t]{0.20\linewidth}
  \centering
\includegraphics[width=1.0\textwidth]{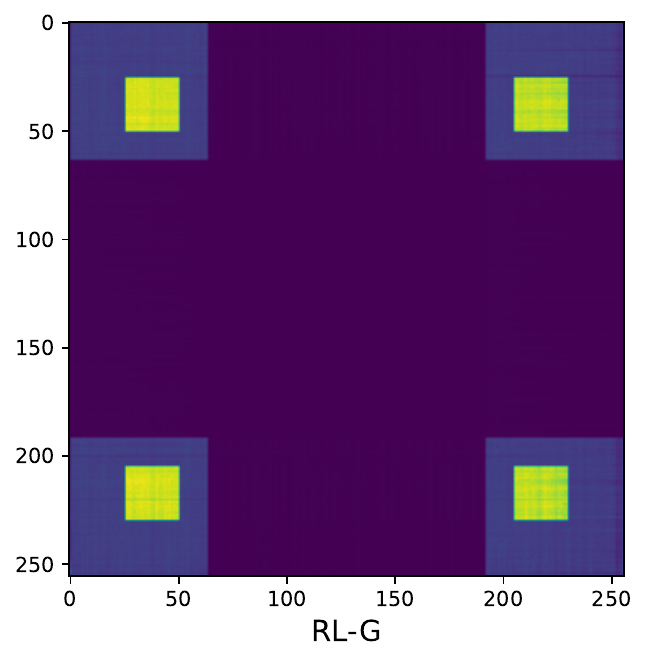}
 \end{minipage}\hspace{-1.5mm}
 \caption{Graphical illustrations of $P_F^\top(x)$ averaged across 5 runs of corresponding training strategies for a $256\times 256$ grid.}\label{256-plots}
\end{figure}

\begin{figure}[h!]
\begin{minipage}[t]{0.20\linewidth}
  \centering
\includegraphics[width=1.0\textwidth]{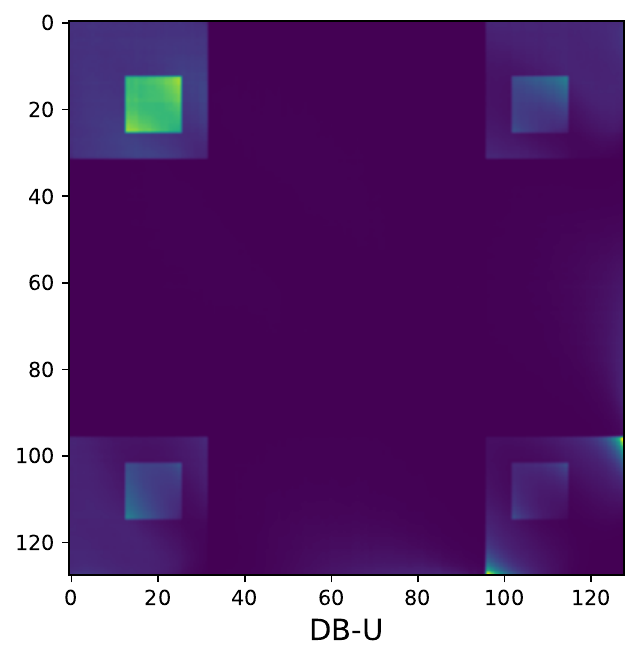}
\end{minipage}\hspace{-1.5mm}
\begin{minipage}[t]{0.20\linewidth}
  \centering
\includegraphics[width=1.0\textwidth]{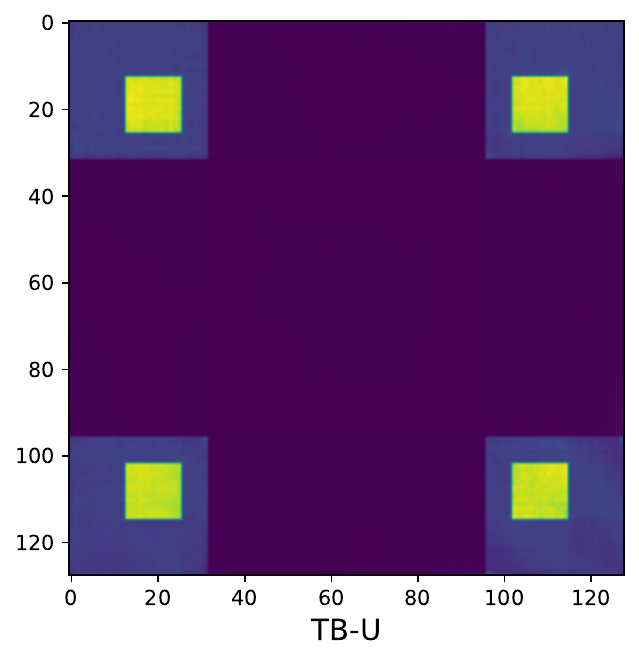}
\end{minipage}\hspace{-1.5mm}
\begin{minipage}[t]{0.20\linewidth}
  \centering
\includegraphics[width=1.0\textwidth]{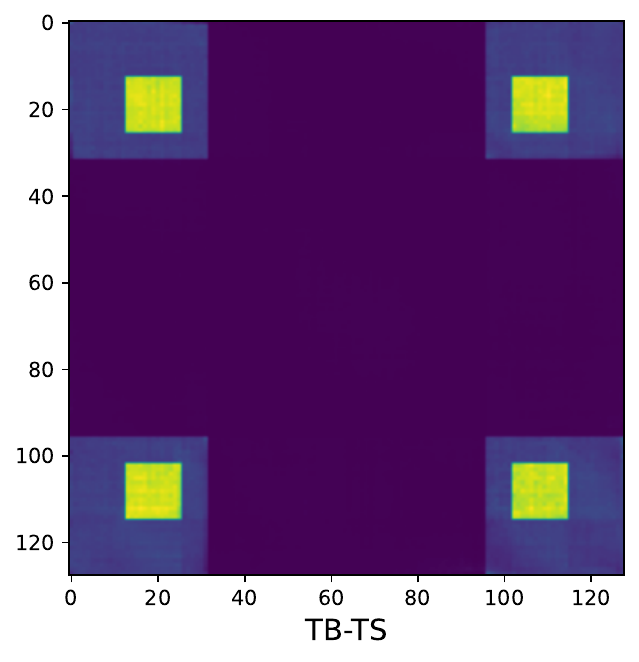}
\end{minipage}\hspace{-1.5mm}
\begin{minipage}[t]{0.20\linewidth}
  \centering
\includegraphics[width=1.0\textwidth]{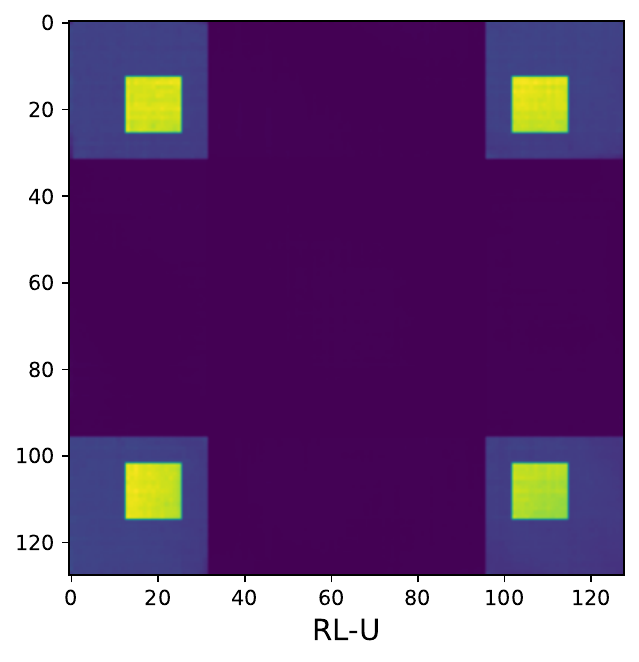}
\end{minipage}\hspace{-1.5mm}
\begin{minipage}[t]{0.20\linewidth}
  \centering
\includegraphics[width=1.0\textwidth]{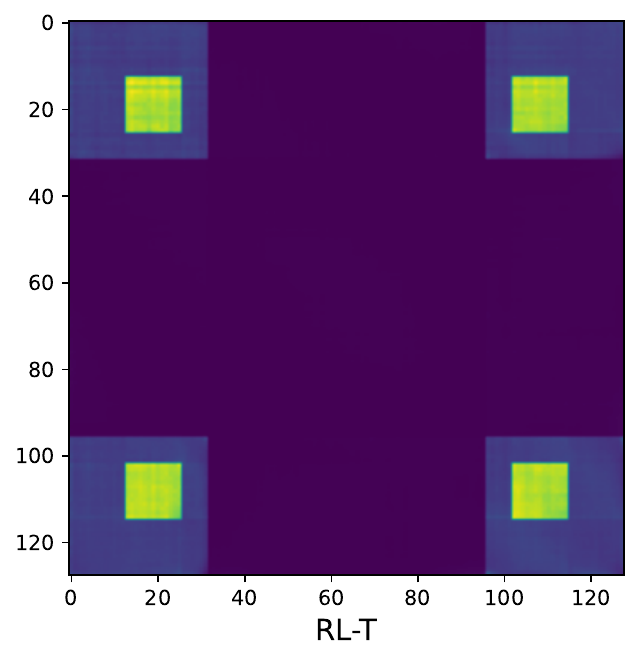}
 \end{minipage}\hspace{-1.5mm}\\
 \begin{minipage}[t]{0.20\linewidth}
  \centering
\includegraphics[width=1.0\textwidth]{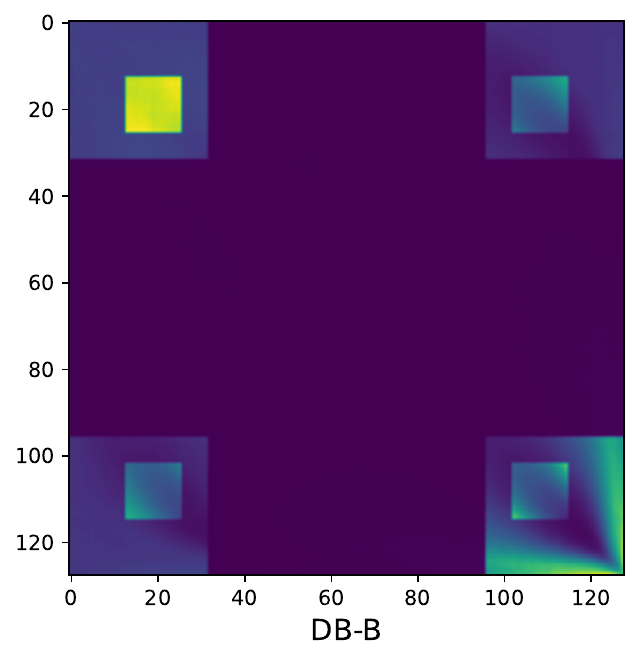}
 \end{minipage}\hspace{-1.5mm}
 \begin{minipage}[t]{0.20\linewidth}
  \centering
\includegraphics[width=1.0\textwidth]{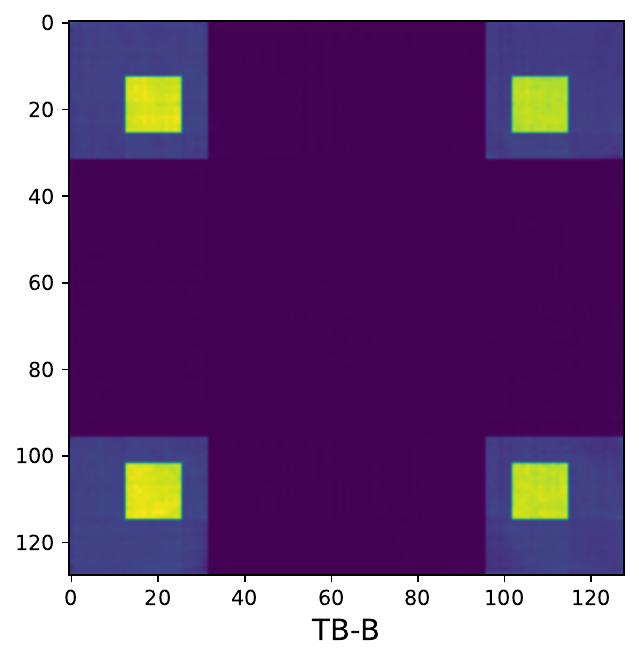}
\end{minipage}\hspace{-1.5mm}
 \begin{minipage}[t]{0.20\linewidth}
  \centering
\includegraphics[width=1.0\textwidth]{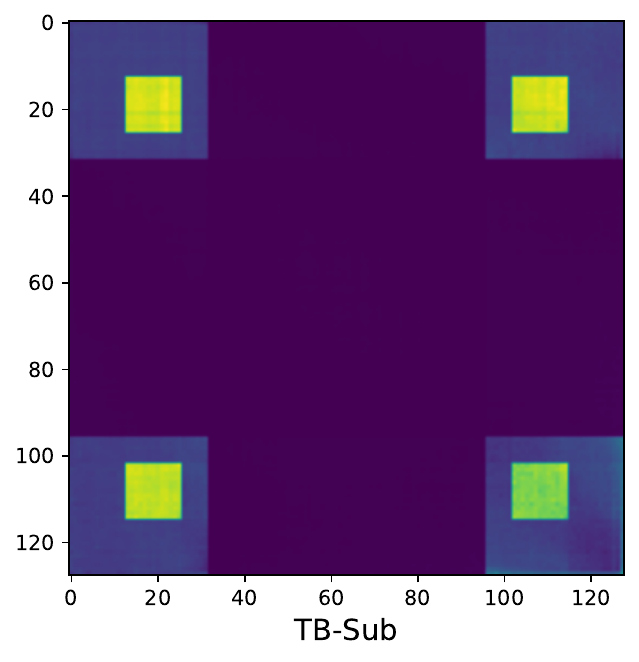}
\end{minipage}\hspace{-1.5mm}
\begin{minipage}[t]{0.20\linewidth}
  \centering
\includegraphics[width=1.0\textwidth]{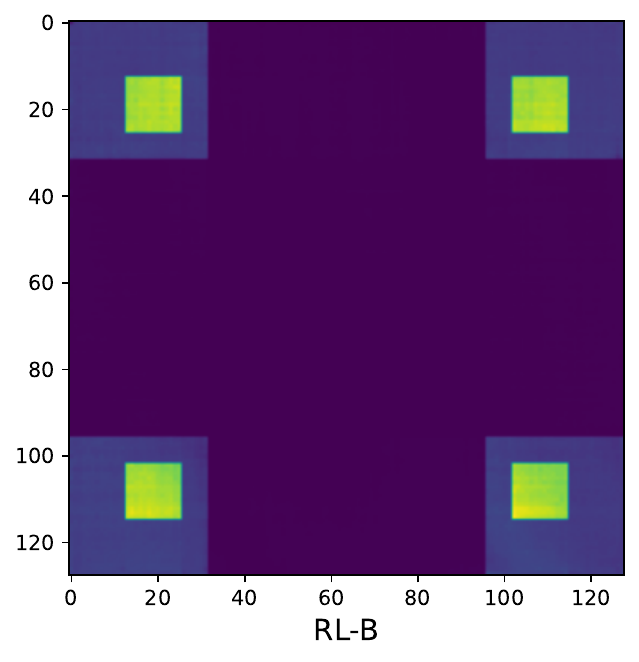}
 \end{minipage}\hspace{-1.5mm}
 \begin{minipage}[t]{0.20\linewidth}
  \centering
\includegraphics[width=1.0\textwidth]{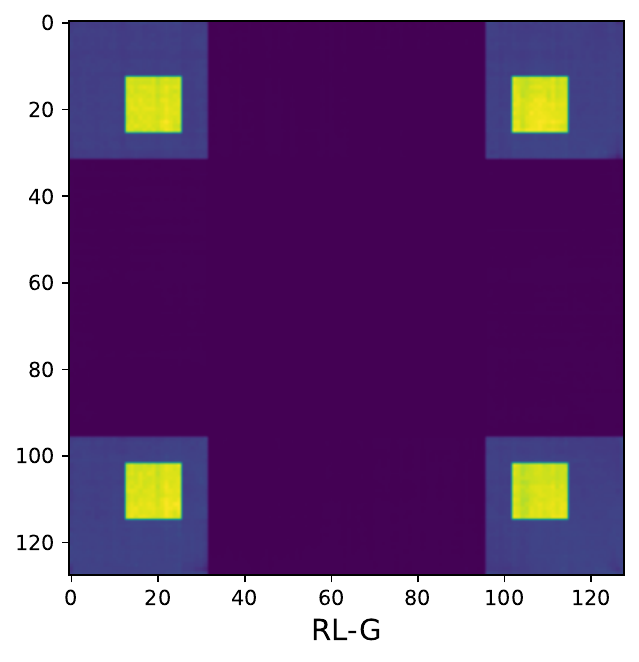}
 \end{minipage}\hspace{-1.5mm}
 \caption{Graphical illustrations of $P_F^\top(x)$ averaged across 5 runs of corresponding training strategies for a $128\times 128$ grid.}\label{128-plots}
\end{figure}

\begin{figure}[h!]
\begin{minipage}[t]{0.20\linewidth}
  \centering
\includegraphics[width=1.0\textwidth]{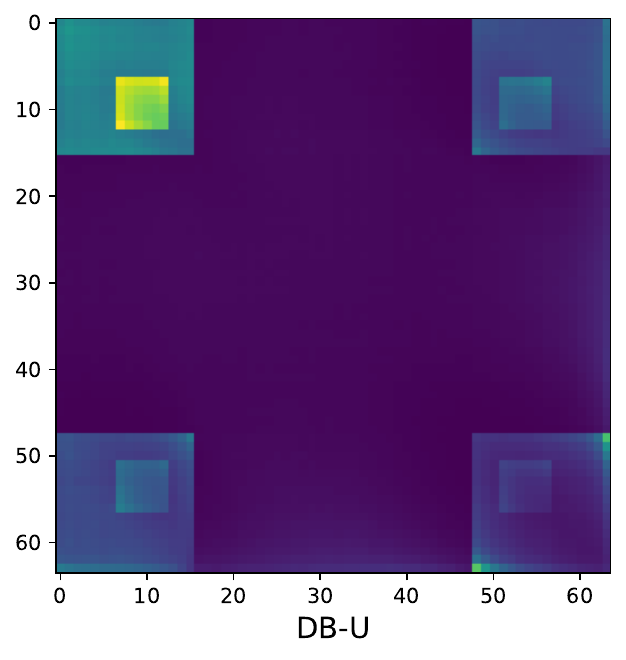}
\end{minipage}\hspace{-1.5mm}
\begin{minipage}[t]{0.20\linewidth}
  \centering
\includegraphics[width=1.0\textwidth]{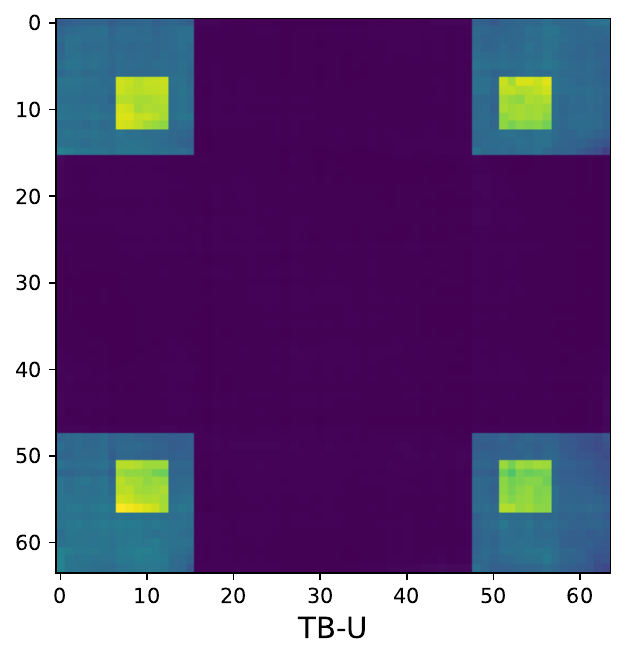}
\end{minipage}\hspace{-1.5mm}
\begin{minipage}[t]{0.20\linewidth}
  \centering
\includegraphics[width=1.0\textwidth]{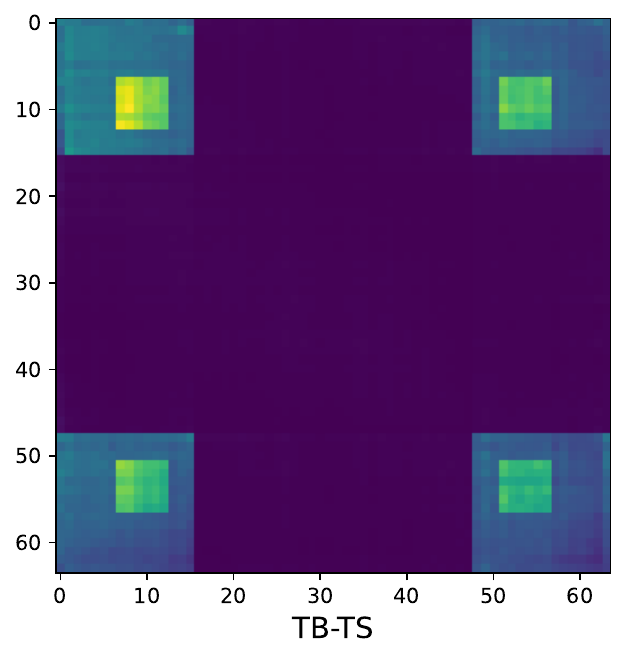}
\end{minipage}\hspace{-1.5mm}
\begin{minipage}[t]{0.20\linewidth}
  \centering
\includegraphics[width=1.0\textwidth]{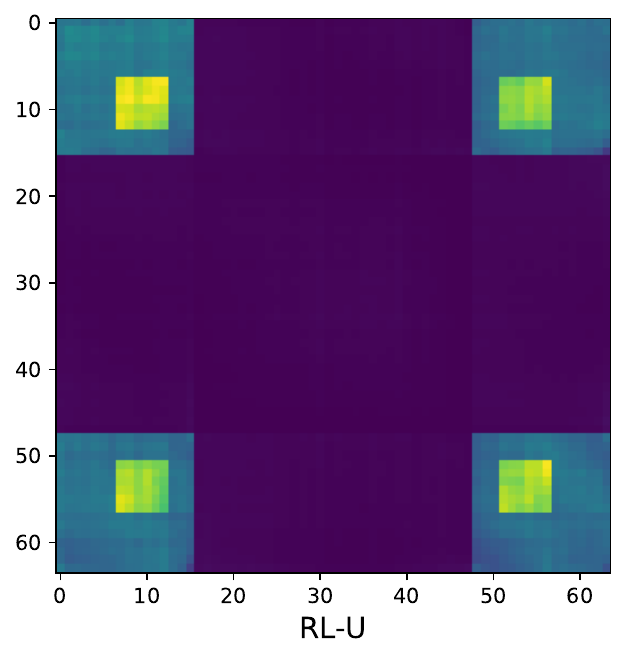}
 \end{minipage}\hspace{-1.5mm}
 \begin{minipage}[t]{0.20\linewidth}
  \centering
\includegraphics[width=1.0\textwidth]{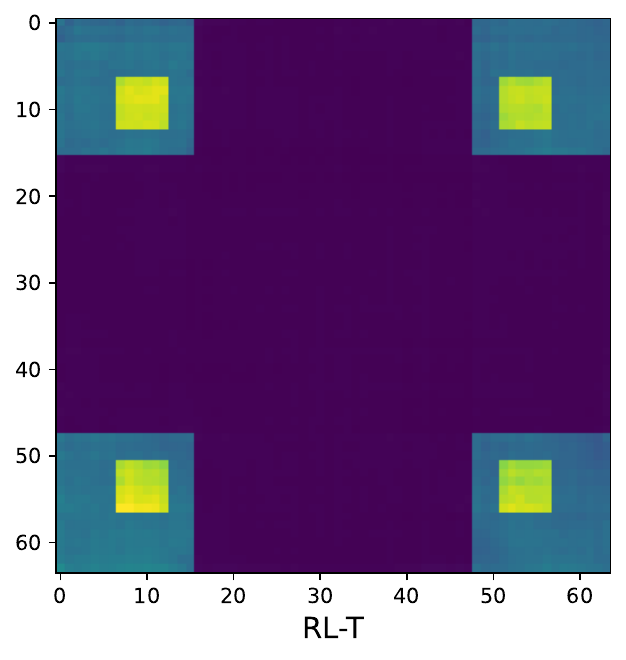}
 \end{minipage}\hspace{-1.5mm}\\
 \begin{minipage}[t]{0.20\linewidth}
  \centering
\includegraphics[width=1.0\textwidth]{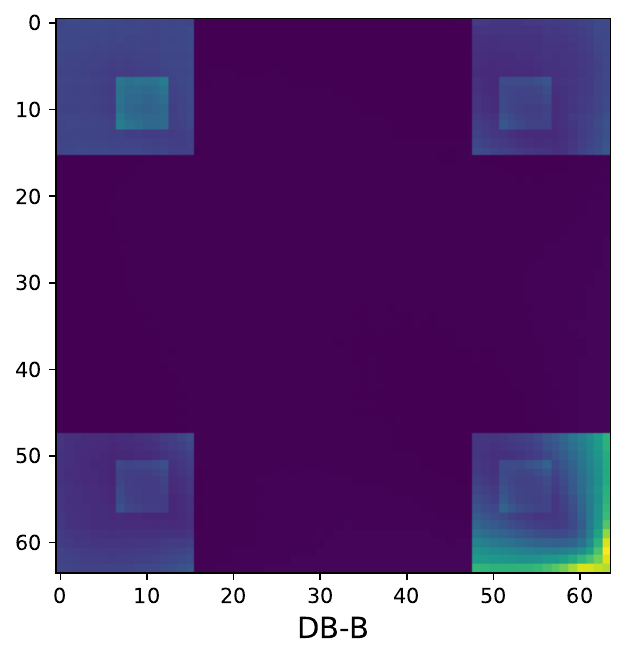}
\end{minipage}\hspace{-1.5mm}
 \begin{minipage}[t]{0.20\linewidth}
  \centering
\includegraphics[width=1.0\textwidth]{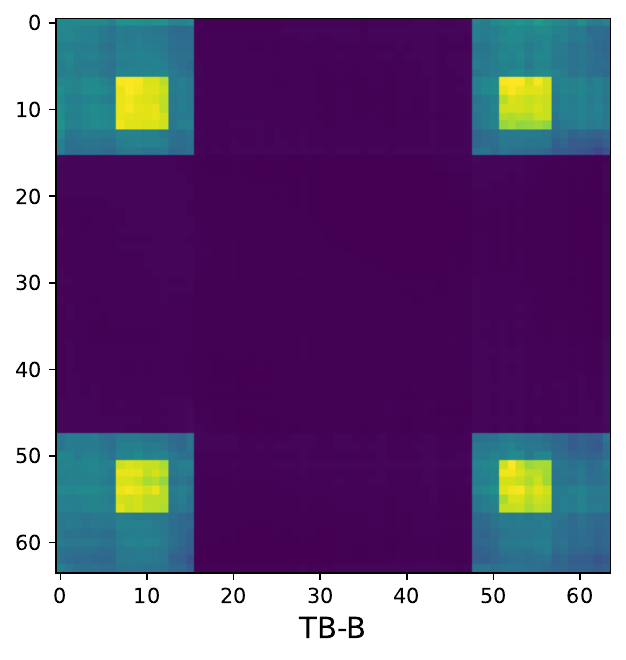}
\end{minipage}\hspace{-1.5mm}
\begin{minipage}[t]{0.20\linewidth}
  \centering
\includegraphics[width=1.0\textwidth]{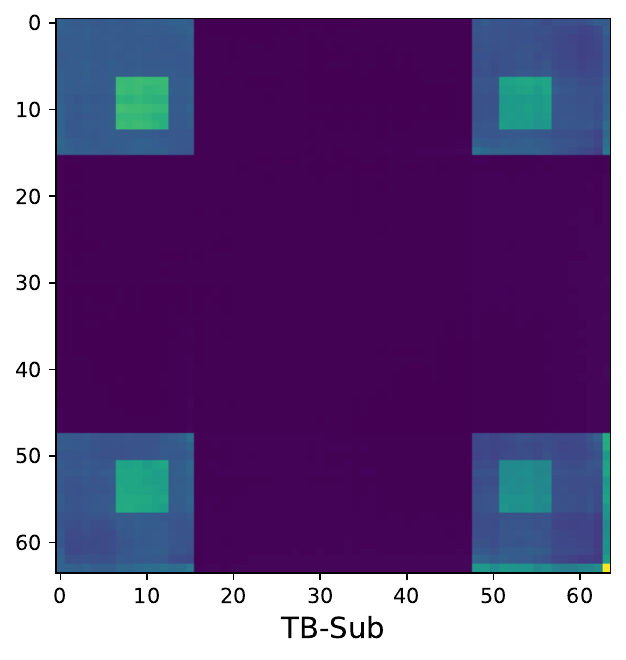}
\end{minipage}\hspace{-1.5mm}
\begin{minipage}[t]{0.20\linewidth}
  \centering
\includegraphics[width=1.0\textwidth]{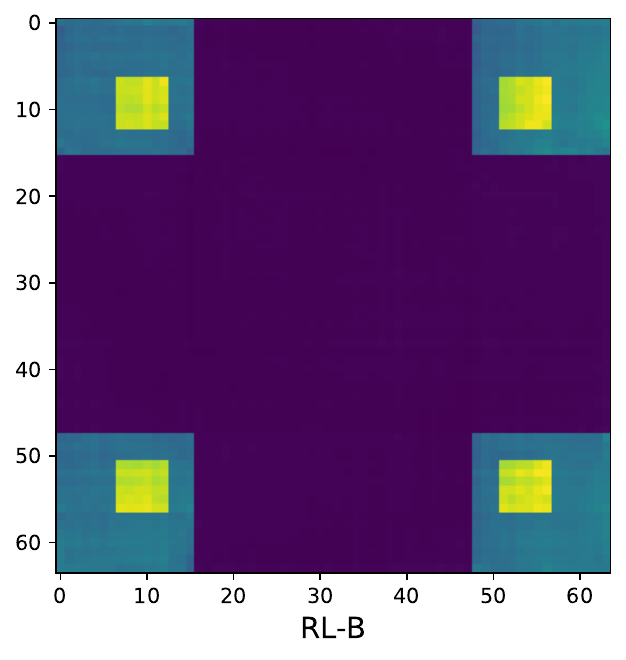}
 \end{minipage}\hspace{-1.5mm}
 \begin{minipage}[t]{0.20\linewidth}
  \centering
\includegraphics[width=1.0\textwidth]{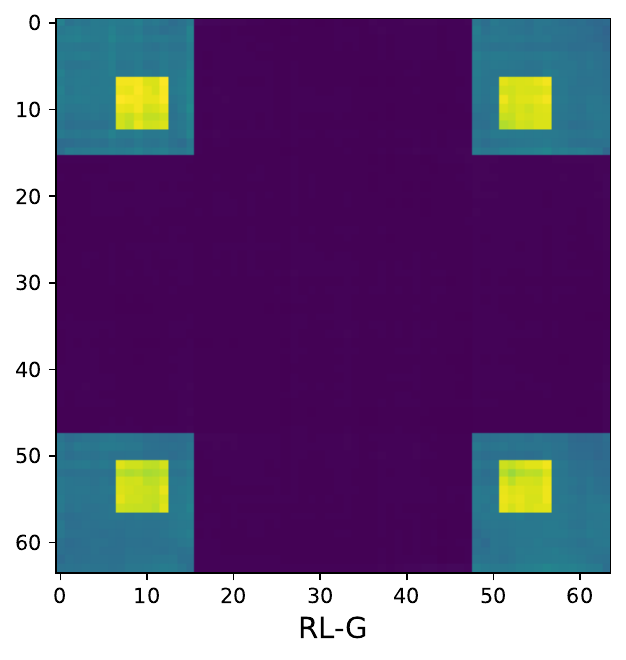}
 \end{minipage}\hspace{-1.5mm}
 \caption{Graphical illustrations of $P_F^\top(x)$ averaged across 5 runs of corresponding training strategies for a $64\times 64\times 64$ grid. For visualization easiness, only the marginals of two dimensions are plotted. }\label{64-plots}
\end{figure}

\begin{figure}[h!]
\begin{minipage}[t]{0.2\linewidth}
  \centering
\includegraphics[width=1.0\textwidth]{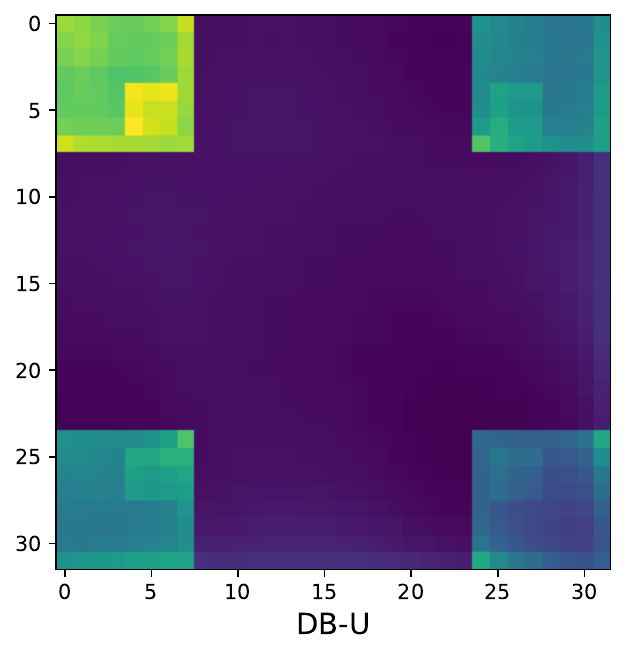}
\end{minipage}\hspace{-1.5mm}
\begin{minipage}[t]{0.2\linewidth}
  \centering
\includegraphics[width=1.0\textwidth]{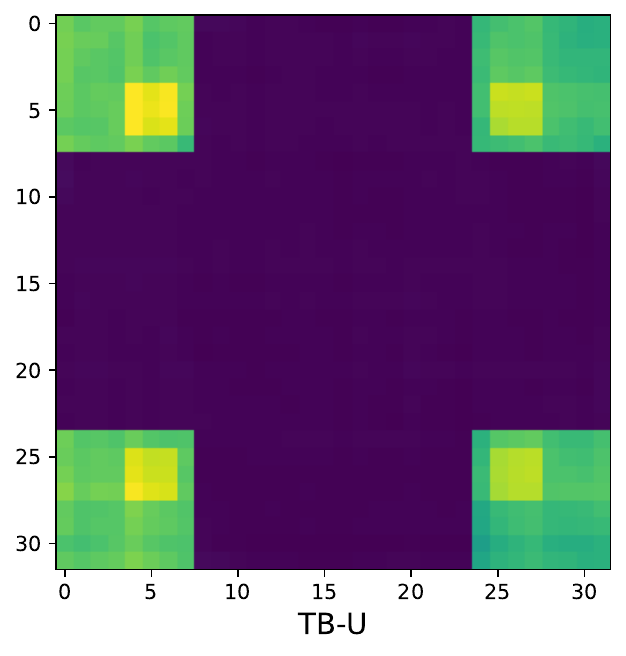}
\end{minipage}\hspace{-1.5mm}
\begin{minipage}[t]{0.2\linewidth}
  \centering
\includegraphics[width=1.0\textwidth]{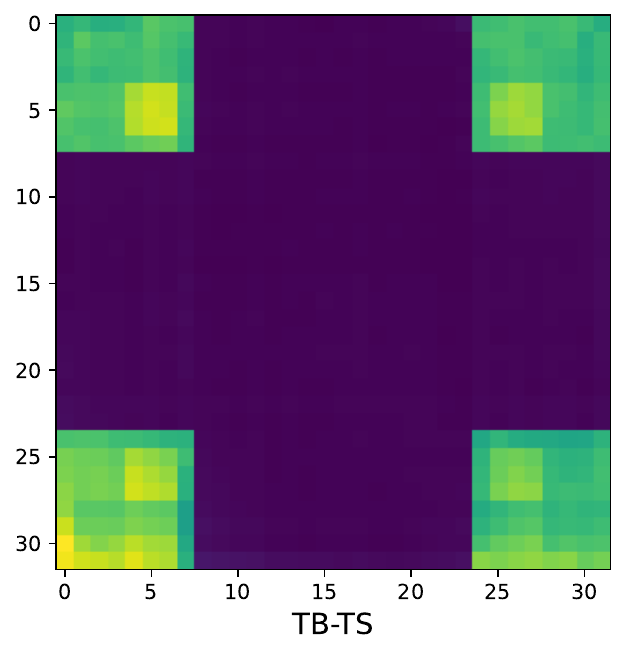}
\end{minipage}\hspace{-1.5mm}
\begin{minipage}[t]{0.2\linewidth}
  \centering
\includegraphics[width=1.0\textwidth]{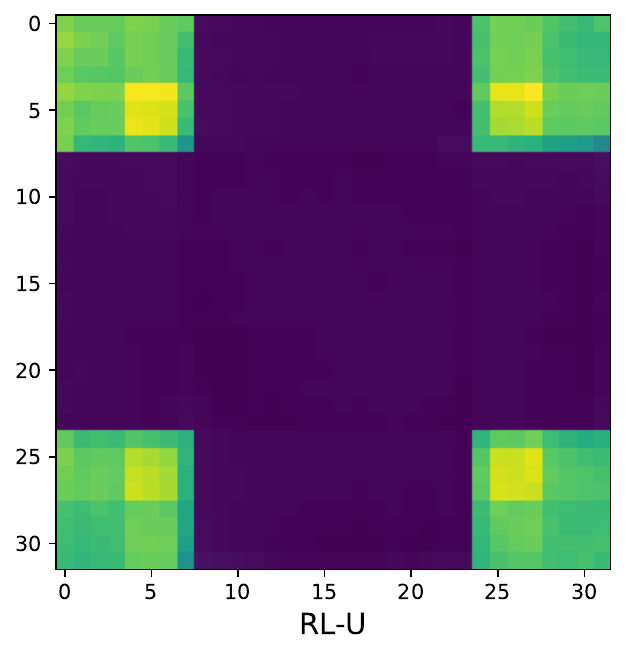}
\end{minipage}\hspace{-1.5mm}
\begin{minipage}[t]{0.2\linewidth}
  \centering
\includegraphics[width=1.0\textwidth]{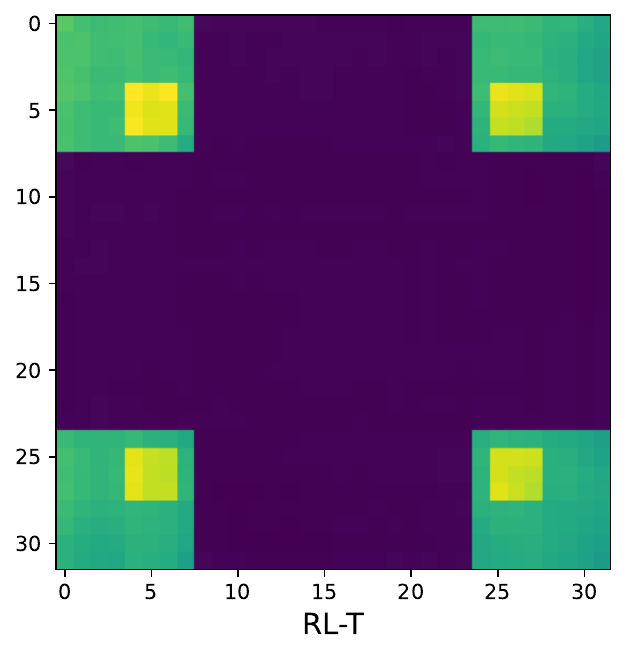}
 \end{minipage}\hspace{-1.5mm}\\
 \begin{minipage}[t]{0.2\linewidth}
  \centering
\includegraphics[width=1.0\textwidth]{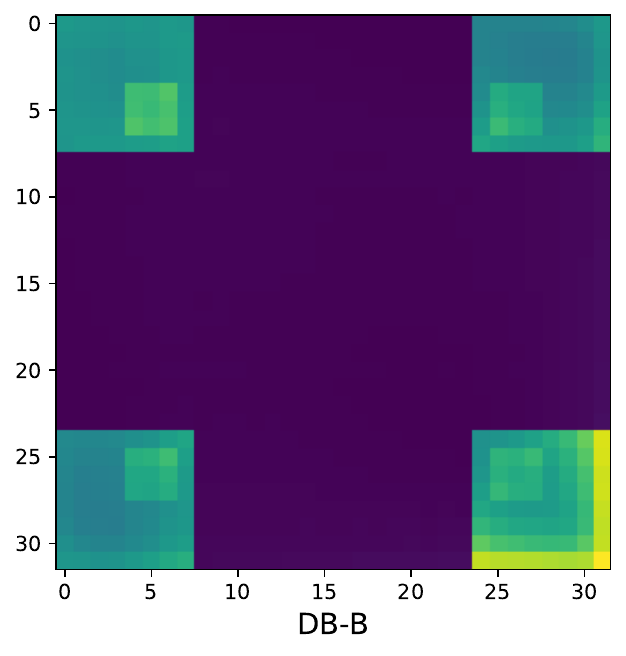}
 \end{minipage}\hspace{-1.5mm}
 \begin{minipage}[t]{0.2\linewidth}
  \centering
\includegraphics[width=1.0\textwidth]{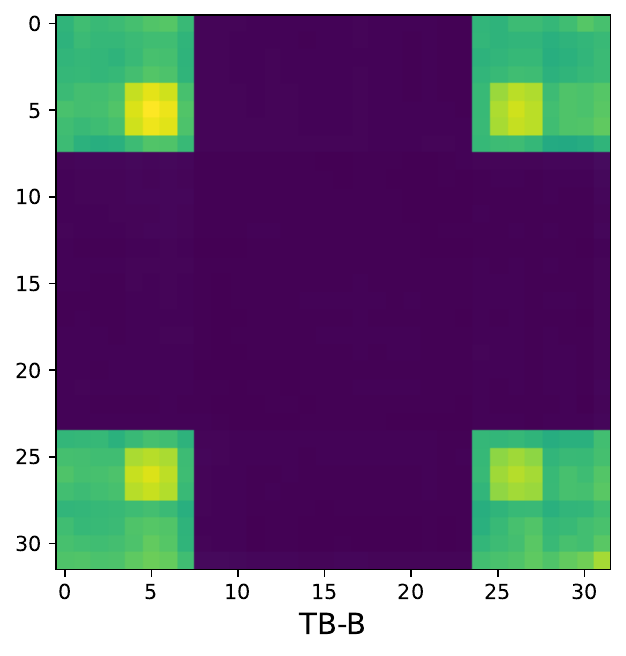}
\end{minipage}\hspace{-1.5mm}
 \begin{minipage}[t]{0.2\linewidth}
  \centering
\includegraphics[width=1.0\textwidth]{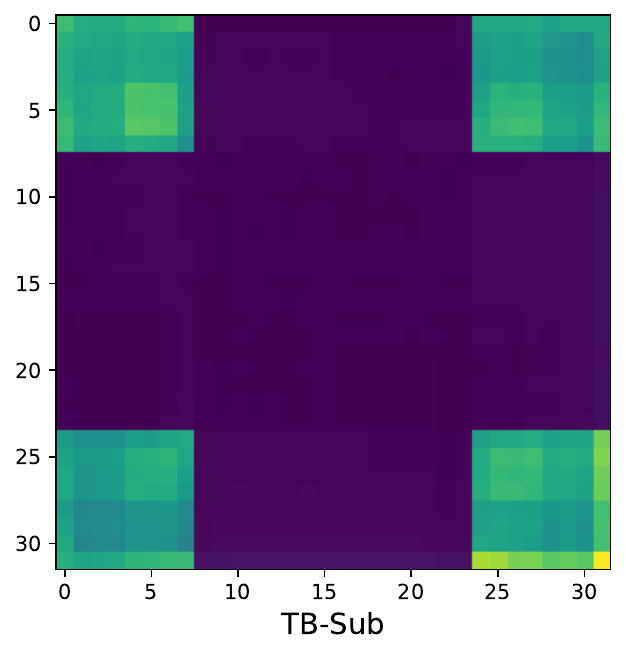}
\end{minipage}\hspace{-1.5mm}
\begin{minipage}[t]{0.2\linewidth}
  \centering
\includegraphics[width=1.0\textwidth]{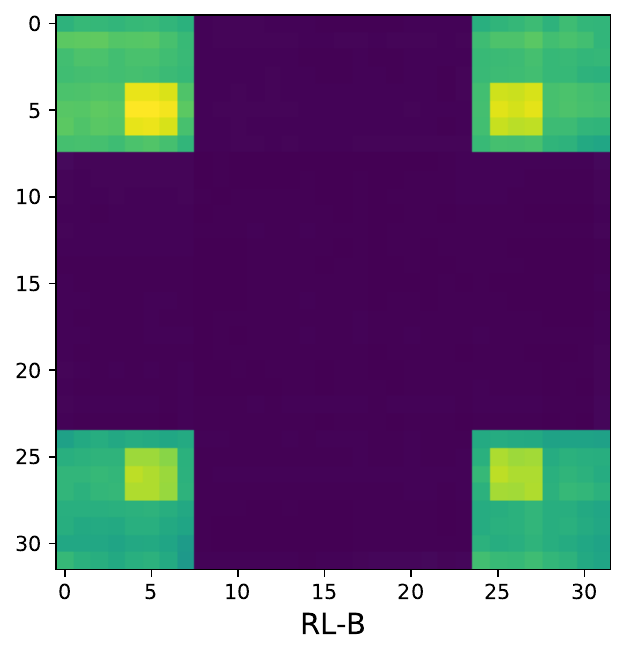}
 \end{minipage}\hspace{-1.5mm}
 \begin{minipage}[t]{0.2\linewidth}
  \centering
\includegraphics[width=1.0\textwidth]{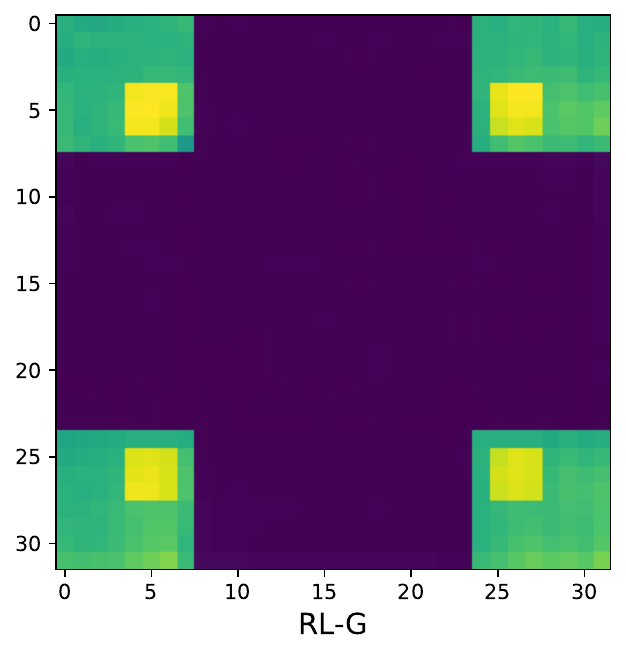}
 \end{minipage}\hspace{-1.5mm}
 \caption{Graphical illustrations of $P_F^\top(x)$ averaged across 5 runs of corresponding training strategies for a $32\times 32\times 32\times32$ grid. For visualization easiness, only the marginals of two dimensions are plotted. }\label{32-plots}
\end{figure}

\begin{figure}[h!]
\begin{minipage}[t]{0.2\linewidth}
  \centering
\includegraphics[width=1.0\textwidth]{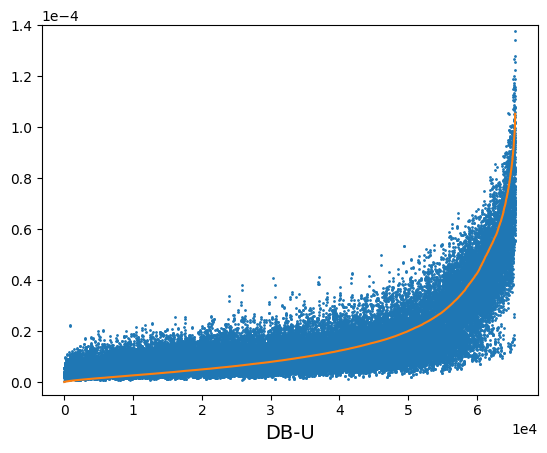}
\end{minipage}\hspace{-1.2mm}
\begin{minipage}[t]{0.2\linewidth}
  \centering
\includegraphics[width=1.0\textwidth]{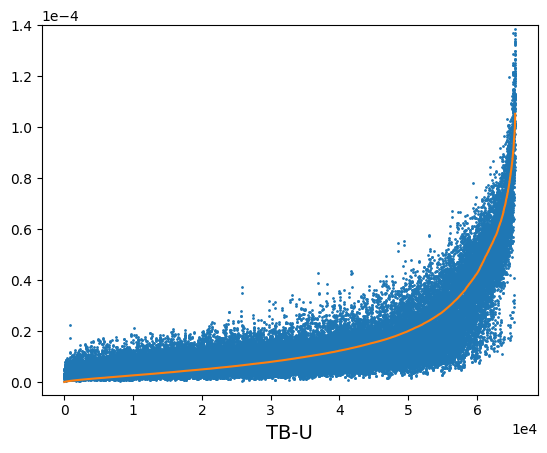}
\end{minipage}\hspace{-1.2mm}
\begin{minipage}[t]{0.2\linewidth}
  \centering
\includegraphics[width=1.0\textwidth]{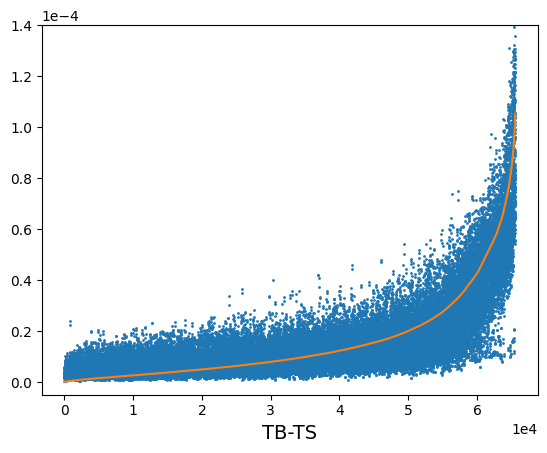}
\end{minipage}\hspace{-1.2mm}
\begin{minipage}[t]{0.2\linewidth}
  \centering
\includegraphics[width=1.0\textwidth]{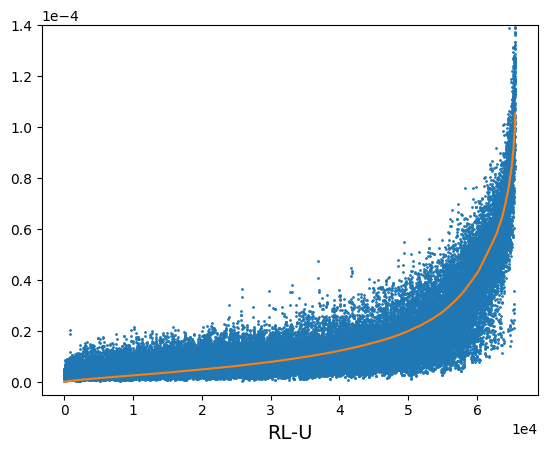}
\end{minipage}\hspace{-1.2mm}
\begin{minipage}[t]{0.2\linewidth}
  \centering
\includegraphics[width=1.0\textwidth]{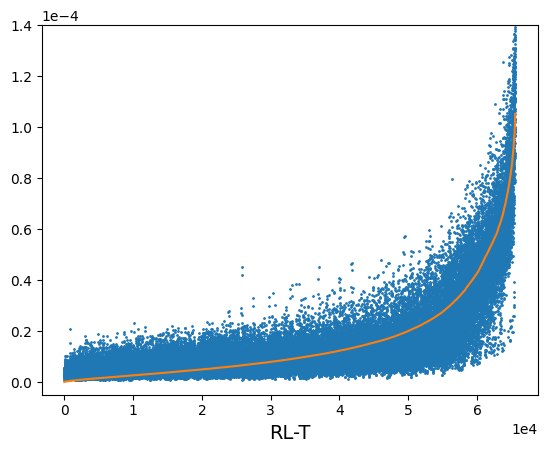}
 \end{minipage}\hspace{-1.2mm}\\
 \begin{minipage}[t]{0.2\linewidth}
  \centering
\includegraphics[width=1.0\textwidth]{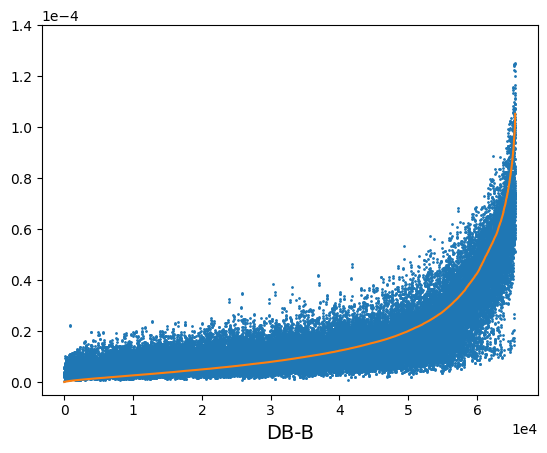}
 \end{minipage}\hspace{-1.2mm}
 \begin{minipage}[t]{0.2\linewidth}
  \centering
\includegraphics[width=1.0\textwidth]{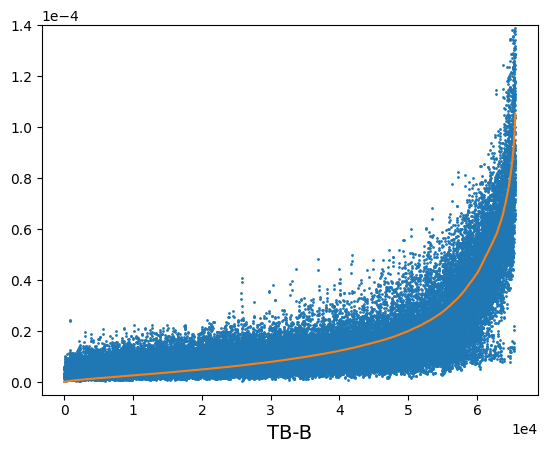}
\end{minipage}\hspace{-1.2mm}
 \begin{minipage}[t]{0.2\linewidth}
  \centering
\includegraphics[width=1.0\textwidth]{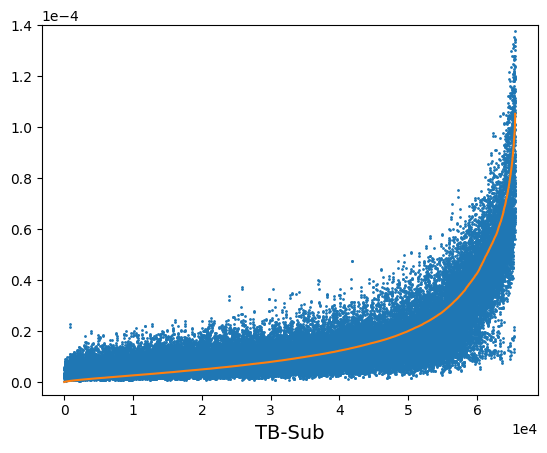}
\end{minipage}\hspace{-1.2mm}
\begin{minipage}[t]{0.2\linewidth}
  \centering
\includegraphics[width=1.0\textwidth]{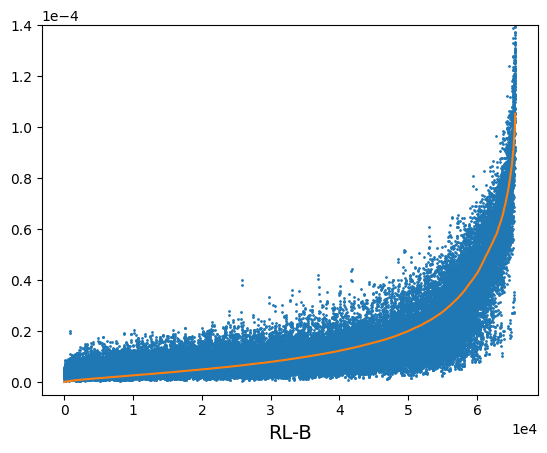}
 \end{minipage}\hspace{-1.2mm}
 \begin{minipage}[t]{0.2\linewidth}
  \centering
\includegraphics[width=1.0\textwidth]{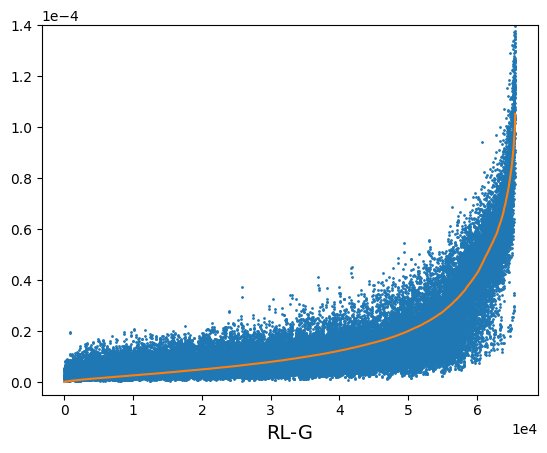}
 \end{minipage}\hspace{-1.2mm}
 \vspace{-2mm}
 \caption{In each plot, the orange line represents the $P^\ast(x)$ of all sequences in the SIX6 dataset, with its values plotted in ascending order. The blue dots represent corresponding values of $P_F^\top(x)$, averaged over five runs of the corresponding training strategy.}\label{TF8-plots}
\end{figure}

\begin{figure}[h!]
\begin{minipage}[t]{0.2\linewidth}
  \centering
\includegraphics[width=1.0\textwidth]{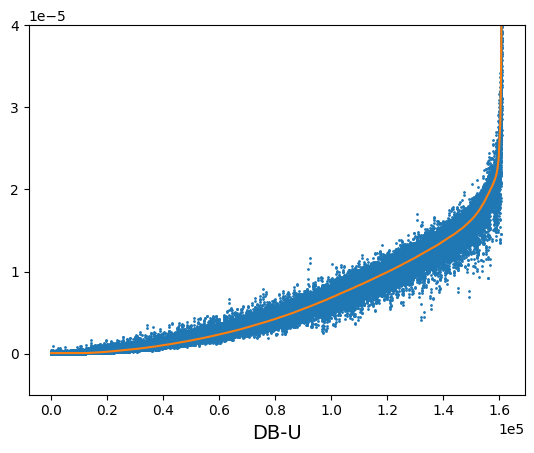}
\end{minipage}\hspace{-1.2mm}
\begin{minipage}[t]{0.2\linewidth}
  \centering
\includegraphics[width=1.0\textwidth]{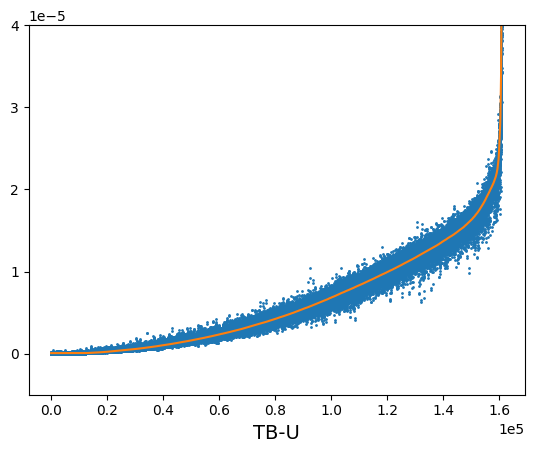}
\end{minipage}\hspace{-1.2mm}
\begin{minipage}[t]{0.2\linewidth}
  \centering
\includegraphics[width=1.0\textwidth]{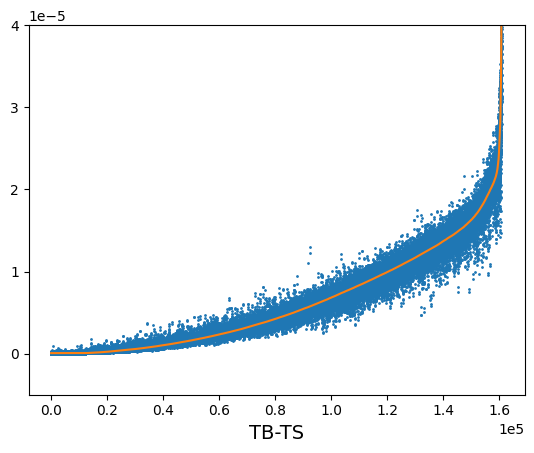}
\end{minipage}\hspace{-1.2mm}
\begin{minipage}[t]{0.2\linewidth}
  \centering
\includegraphics[width=1.0\textwidth]{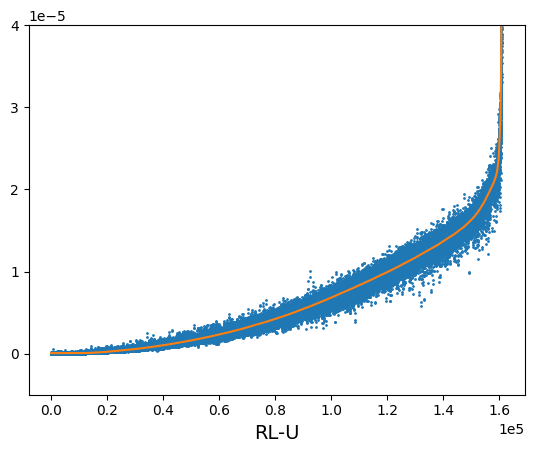}
\end{minipage}\hspace{-1.2mm}
\begin{minipage}[t]{0.2\linewidth}
  \centering
\includegraphics[width=1.0\textwidth]{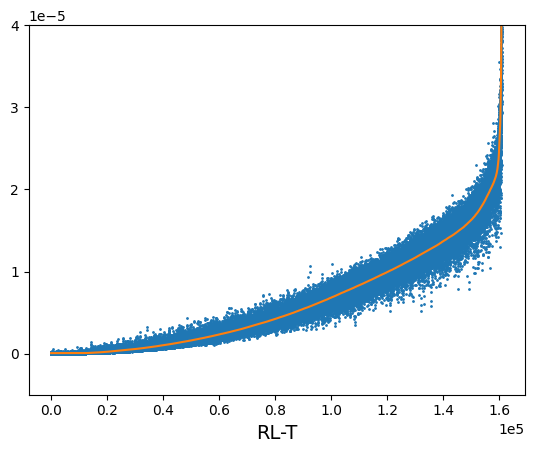}
 \end{minipage}\hspace{-1.2mm}\\
 \begin{minipage}[t]{0.2\linewidth}
  \centering
\includegraphics[width=1.0\textwidth]{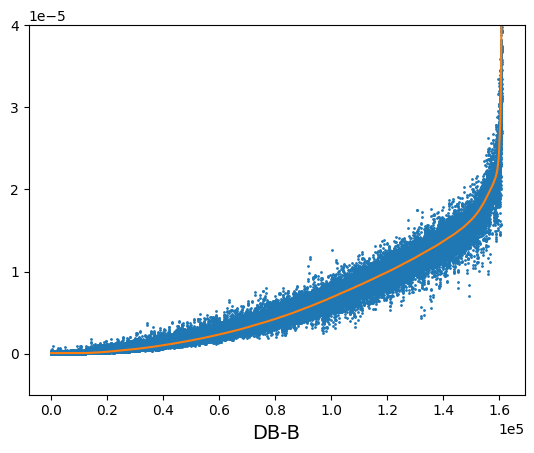}
 \end{minipage}\hspace{-1.2mm}
 \begin{minipage}[t]{0.2\linewidth}
  \centering
\includegraphics[width=1.0\textwidth]{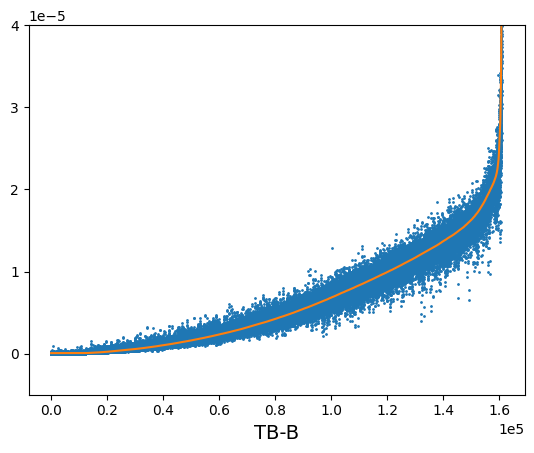}
\end{minipage}\hspace{-1.2mm}
 \begin{minipage}[t]{0.2\linewidth}
  \centering
\includegraphics[width=1.0\textwidth]{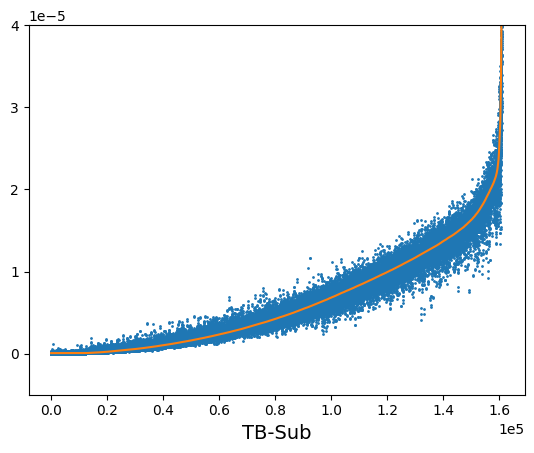}
\end{minipage}\hspace{-1.2mm}
\begin{minipage}[t]{0.2\linewidth}
  \centering
\includegraphics[width=1.0\textwidth]{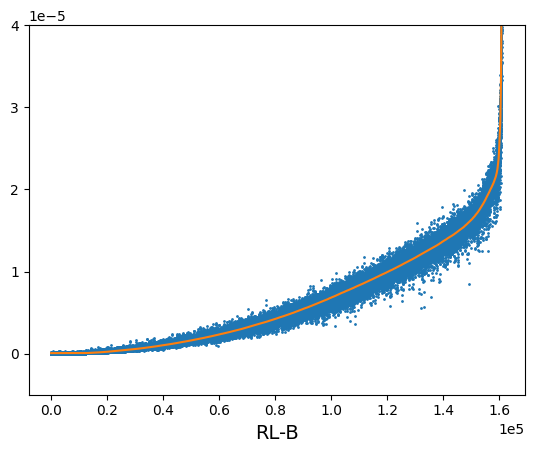}
 \end{minipage}\hspace{-1.2mm}
 \begin{minipage}[t]{0.2\linewidth}
  \centering
\includegraphics[width=1.0\textwidth]{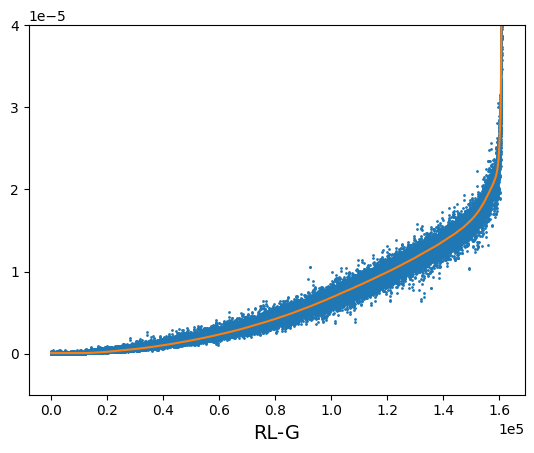}
 \end{minipage}\hspace{-1.2mm}
 \vspace{-2mm}
 \caption{In each plot, the orange line represents the $P^\ast(x)$ of all sequences in the QM9 dataset, with its values plotted in ascending order. The blue dots represent corresponding values of $P_F^\top(x)$, averaged over five runs of the corresponding training strategy.}\label{QM9-plots}
\end{figure}

\begin{figure}[h!]
\begin{minipage}[t]{0.25\linewidth}
  \centering
\includegraphics[width=1.0\textwidth]{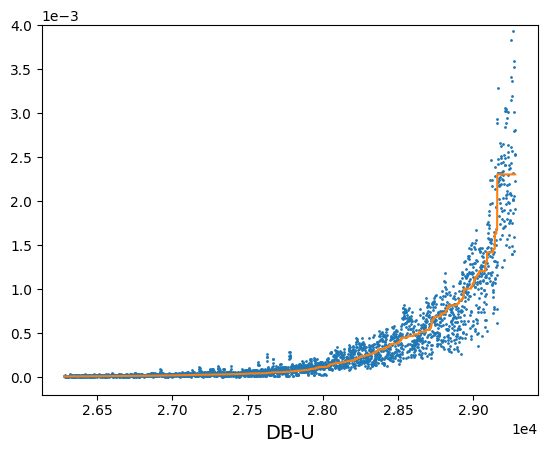}
\end{minipage}\hspace{-1.2mm}
\begin{minipage}[t]{0.25\linewidth}
  \centering
\includegraphics[width=1.0\textwidth]{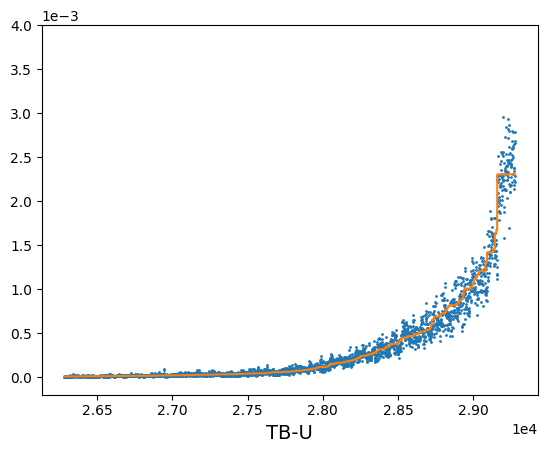}
\end{minipage}\hspace{-1.2mm}
\begin{minipage}[t]{0.25\linewidth}
  \centering
\includegraphics[width=1.0\textwidth]{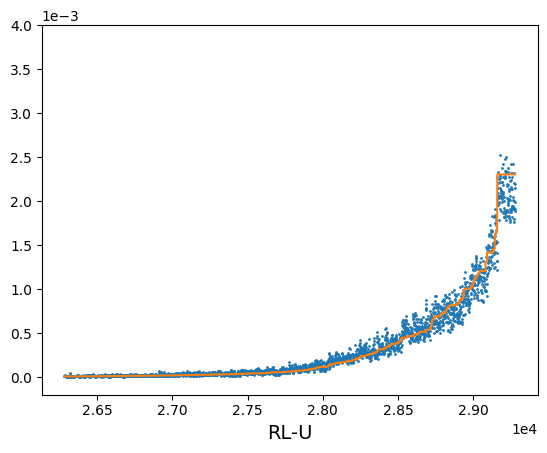}
\end{minipage}\hspace{-1.2mm}
\begin{minipage}[t]{0.25\linewidth}
  \centering
\includegraphics[width=1.0\textwidth]{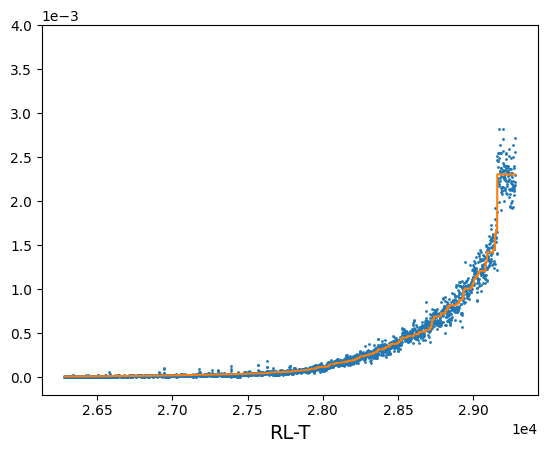}
 \end{minipage}\hspace{-1.4mm}\\
 \begin{minipage}[t]{0.25\linewidth}
  \centering
\includegraphics[width=1.0\textwidth]{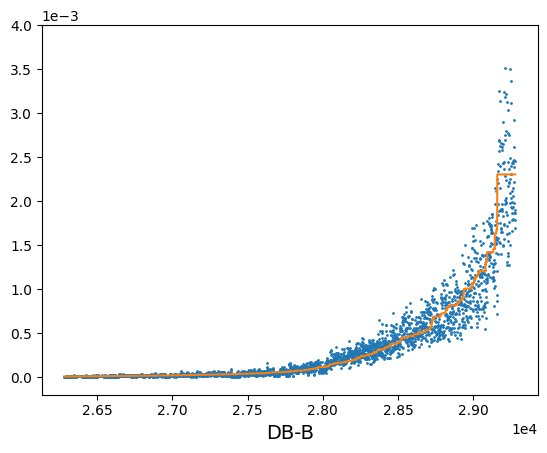}
 \end{minipage}\hspace{-1.2mm}
 \begin{minipage}[t]{0.25\linewidth}
  \centering
\includegraphics[width=1.0\textwidth]{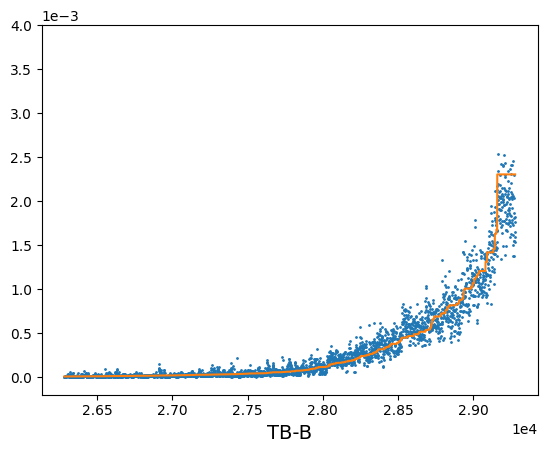}
\end{minipage}\hspace{-1.2mm}
\begin{minipage}[t]{0.25\linewidth}
  \centering
\includegraphics[width=1.0\textwidth]{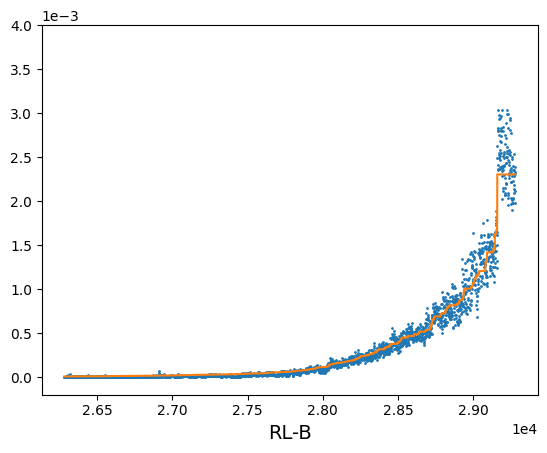}
 \end{minipage}\hspace{-1.2mm}
 \begin{minipage}[t]{0.25\linewidth}
  \centering
\includegraphics[width=1.0\textwidth]{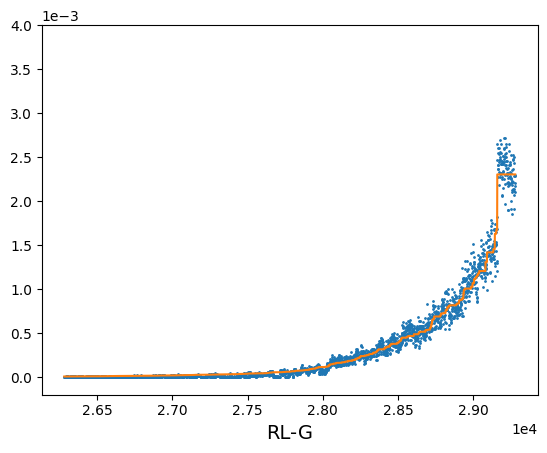}
 \end{minipage}\hspace{-1.2mm}
 \vspace{-2mm}
 \caption{In each plot, the orange line represents the $P^\ast(x)$ of all BN structures, with its values plotted in ascending order. The blue dots represent corresponding values of $P_F^\top(x)$, averaged over five runs of the corresponding training strategy. Only the ground-truth and corresponding learned values for the top 3000 structures are plotted as the remaining values are nearly zero.}\label{BN-plots}
\end{figure}

\end{document}